\documentclass[letterpaper]{article} % DO NOT CHANGE THIS
\usepackage[preprint]{aaai2027}  % DO NOT CHANGE THIS
\usepackage[hyphens]{url}  % DO NOT CHANGE THIS
\usepackage{graphicx} % DO NOT CHANGE THIS
\usepackage{natbib}  % DO NOT CHANGE THIS AND DO NOT ADD ANY OPTIONS TO IT
\usepackage{caption} % DO NOT CHANGE THIS AND DO NOT ADD ANY OPTIONS TO IT
\usepackage{algorithm}
\usepackage{algorithmic}

\usepackage{amsmath}
\usepackage{amssymb}
\usepackage{amsthm}

\newtheorem{theorem}{Theorem}

\usepackage{newfloat}
\usepackage{listings}
\DeclareCaptionStyle{ruled}{labelfont=normalfont,labelsep=colon,strut=off} % DO NOT CHANGE THIS
\floatstyle{ruled}
\newfloat{listing}{tb}{lst}{}
\floatname{listing}{Listing}

\usepackage{booktabs}

\newcommand{\INPUT}{\item[\algorithmicinput]} \newcommand{\algorithmicinput}{\textbf{Input:}}

\newcommand{\appsectionentry}[2]{%
  \noindent
  \textbf{Appendix~\ref{#1}: #2}%
  \dotfill
  \pageref{#1}\par
}

\newcommand{\appsubsectionentry}[2]{%
  \noindent
  \hspace{1.5em}%
  Appendix~\ref{#1}: #2%
  \dotfill
  \pageref{#1}\par
}

\newcommand{\apptheorementry}[2]{%
  \noindent
  \hspace{3em}%
  Theorem~\ref{#1}: #2%
  \dotfill
  \pageref{#1}\par
}

\title{ArborEnum: Decision Tree Rashomon Sets over Continuous Features}

\author{
    Zakk Heile\textsuperscript{\rm 1}\corresponding,
    Hayden McTavish\textsuperscript{\rm 1},
    Margo Seltzer\textsuperscript{\rm 2},
    Cynthia Rudin\textsuperscript{\rm 1}
}

\affiliations{
    \textsuperscript{\rm 1}Department of Computer Science, Duke University,
    Durham, USA\\
    \textsuperscript{\rm 2}Department of Computer Science, University of British Columbia,
    Vancouver, Canada\\
    zakk.heile@duke.edu,
    hayden.mctavish@duke.edu,
    mseltzer@cs.ubc.ca,
    cynthia@cs.duke.edu
}

\begin{document}

\maketitle

\begin{abstract}
The Rashomon effect describes the phenomenon that many models can achieve nearly equivalent performance on the same learning task, with significant ramifications for robustness, feature importance, and customizability. These use cases motivate the computation of Rashomon sets: the set of all models whose regularized loss is near-optimal. Decision trees are one of the few model classes for which Rashomon sets can be fully enumerated, but this computation has always been conditional on a binarization of the original data, either restricting which splits each tree is allowed to make or substantially increasing the complexity of an already difficult combinatorial problem. We introduce the first algorithm that exactly enumerates decision-tree Rashomon sets while exploiting the ordered structure of continuous features. We further develop a relaxation for approximate enumeration and an anytime algorithm that progressively refines the set of candidate thresholds, producing increasingly detailed approximations that converge to the continuous-feature Rashomon set. Experiments show that coarse binarization can miss many trees, important features, and predictive multiplicity; our algorithms achieve orders-of-magnitude speedups over existing enumeration methods, with approximations providing further speedups while maintaining near-perfect recall.
\end{abstract}

%     \link{Code for ArborEnum (our method)}{https://github.com/zakk-h/ArborEnum}
%     % \link{Datasets}{https://aaai.org/example/datasets}
%     % \link{Extended version}{https://aaai.org/example/extended-version}

\begin{links}
\link{The code for our method, ArborEnum, is available at}{https://github.com/zakk-h/ArborEnum}
\end{links}

\section{Introduction}
For many prediction problems, there is not a single uniquely best model. Instead, many distinct models often achieve nearly equivalent predictive performance. This phenomenon, known as the Rashomon effect \citep{breiman1984classification}, motivates the study of Rashomon sets \citep{fisher2019}: collections of models whose objective values are within a prescribed tolerance of optimal. Rather than returning only one model, Rashomon set methods expose the full set of high-performing alternatives. 

Rashomon sets have been investigated for several model classes, including decision trees \citep{xin2022exploring}, rule lists \citep{ciaperoni2024efficient}, generalized additive models \citep{gam_rsets}, and prototypical-part convolutional neural networks \citep{protorset}. We focus this work on finding Rashomon sets of decision trees. Existing decision-tree Rashomon set algorithms \citep{heile2026, arslan2025sorted, babbar2025near, xin2022exploring} address this problem when given a set of binary features. As a result, continuous features must usually be coarsely binarized before the Rashomon set is computed. Even then, the search space is enormous. \citet{osdt} shows that the size of the search space of decision trees of depth $4$ with only $20$ binary features is approximately $8.4 \times 10^{18}$ trees. For some continuous features, there are thousands of cut points to consider. Consequently, existing methods fail to scale to handle this important problem.

We introduce \textbf{ArborEnum} (\textbf{A}lgorithms for \textbf{R}elaying \textbf{B}ounds for \textbf{O}rdered \textbf{R}ashomon \textbf{ENUM}eration), \textbf{the first framework designed for enumerating decision-tree Rashomon sets over continuous features}. ArborEnum adapts threshold bounds developed for finding a single optimal tree \citep{brița2025optimal} to enumerate trees whose objectives fall within a prescribed bound. We use these bounds to propagate information from evaluated thresholds to nearby thresholds, allowing large ranges of candidate splits to be pruned efficiently. ArborEnum can compute the information required by these bounds either optimally or approximately. In the former case, ArborEnum \textbf{exactly enumerates the continuous-feature Rashomon set}; in the latter, it yields \textbf{approximate variants} for settings where exact enumeration remains too expensive, \textbf{achieving a median speedup of 270$\times$} (Table \ref{tab:runtime-exhaustive-8}). For datasets where even high-quality approximation remains computationally demanding, we propose an \textbf{anytime algorithm that progressively refines the binarization and converges to the exact Rashomon set when run to completion}. These methods are supported by an improved representation of the Rashomon set that reduces runtime and memory usage while potentially improving approximation quality. Across our experiments, exact enumeration is feasible on many datasets, while our approximations recover nearly all trees at a fraction of the cost on harder instances. We further show that accounting for continuous thresholds is important for downstream analyses: coarse binarization can miss predictive multiplicity, important variables, and high-quality trees in the Rashomon set, motivating the use of our anytime algorithm to refine the binarization as much as is computationally feasible.

\section{Related Work}

\paragraph{Greedy and optimal trees.}
Decision trees are among the most widely used interpretable model classes, with classical algorithms such as CART \citep{breiman1984classification} and C4.5 \citep{quinlan2014c45} providing scalable top-down procedures for fitting trees. These methods choose splits greedily, which makes them computationally efficient but without any optimality guarantees.

A large body of recent work has studied optimal decision trees, including our setting of optimizing misclassification error plus a per-leaf penalty. Although this problem is NP-hard, specialized methods based on mixed-integer programming, dynamic programming, caching, and branch-and-bound have made it practical to find optimal trees of bounded size for many datasets \citep{osdt, gosdt, dl85, murtree, streed}.
These methods admit an interpretation of searching over an AND/OR graph, where OR nodes correspond to split choices and AND nodes correspond to combining left and right solutions \citep{sullivan2024maptree, chaoukibranches}. This structure is what has allowed prior work to store Rashomon sets compactly \citep{heile2026, arslan2025sorted, xin2022exploring}.

\paragraph{Binarization methods.}

A common way to apply binary-feature tree optimization methods to continuous data is to binarize each continuous feature using a small number of thresholds, often chosen by empirical quantiles; for example, \citet{babbar2025near} uses three quantile thresholds per feature. 
Threshold guessing \citep{gosdt_guesses} instead trains a reference ensemble such as XGBoost, extracts candidate thresholds, and uses backward elimination to keep a small set of high-quality splits. While effective for finding a single optimal tree, this is less suitable for Rashomon sets, where the goal is to characterize many good models.

\paragraph{Handling continuous features.}

Recent work extends these ideas to continuous features~\citep{brița2025optimal, kiossou2026anytime}. \citet{quantbnb} handle continuous features by performing branch-and-bound over intervals of each feature and using quantile-based upper and lower bounds to prune ranges of candidate thresholds that cannot contain an optimal tree. However, they focus only on depth-2 or depth-3 trees, as the algorithm does not scale beyond that. \citet{brița2025optimal} also optimize classification trees directly over continuous 
features, using bounds that are looser but cheaper to compute. % \citet{kiossou2026anytime} build on this approach by prioritizing subproblems according to their deviation from a greedy search, making the search state dependent on the path taken to reach a subproblem. In contrast, our anytime algorithm requires no such path-dependent information and instead progressively refines the set of candidate cut points, ensuring that the resulting Rashomon set does not unintentionally favor particular regions of the search space.

\paragraph{Approximation algorithms.} 
For the approximation algorithms in our work, we build on \citet{heile2026}, which uses a fast proxy algorithm to certify feasible completions within the Rashomon budget and prunes branches whose proxy-completed objective exceeds that bound. Since the budget is set relative to the proxy objective, this pruning rule is not overly aggressive; in fact, it is never incorrect when the proxy optimality gap is maximized at the root. Our approximation algorithms extend this strategy to continuous features through new relaxations and proxy algorithms. One important proxy is LicketySPLIT~\citep{babbar2025near}. At each
subproblem, it greedily completes the children of each candidate split, selects
the split with the best completed regularized objective, and recurses on the
resulting subproblems. We relax LicketySPLIT to operate efficiently over continuous features and use the resulting algorithm, which we call ``LicketySNIP,'' as a proxy in our approximation methods.

\section{Methods}

Let $D$ denote the current subproblem, represented by a bitvector over the training samples, and let $\gamma$ be a per-leaf penalty. Let $\mathcal{T}_d$ denote the set of axis-aligned decision trees
of depth at most $d$ on the training features. Let $\mathcal{Y}$ denote the set of $\geq 2$ class labels. We score a tree by
\[
    \mathrm{Obj}(T,D,\gamma) =
    \mathrm{misclassifications}(T;D) + \gamma \cdot |\mathrm{Leaves}(T)|.
\]
We use a \emph{proxy algorithm} to obtain an objective value attained by
some tree in $\mathcal{T}_d$ for the subproblem and remaining depth $d$; this tree may be, but need not
be, optimal. Our pruning rules assume a proxy robustness condition: for instance, moving a few samples across a threshold changes the proxy-completed objective by at most the number of samples moved. Optimal proxies satisfy this condition, in which case our algorithm returns the exact Rashomon set. For non-optimal proxies, the same rule acts as a useful relaxation, enabling large speedups with little or no empirical loss in solution quality.

 Given a budget
$\varepsilon_{\mathrm{abs}}$, we find subtrees whose objective on the subproblem is at most $\varepsilon_{\mathrm{abs}}$; if we are using an approximate proxy algorithm, we recover a subset of trees that satisfy this; otherwise, it is exact. The multiplicative Rashomon set is defined using $
\varepsilon_{\mathrm{abs}} = (1+\varepsilon_{\mathrm{mult}})\textsc{Optimal}(D,d,\gamma)$ 
so that it contains all trees whose objective is within a multiplicative factor of the optimal objective. We approximate it by initially setting $\varepsilon_{\mathrm{abs}} = (1+\varepsilon_{\mathrm{mult}})\textsc{Proxy}(D,d,\gamma).$

For each subproblem--depth pair $(D,d)$, we persist at most (1) a cached proxy solution and its intermediate computations, (2) a pointer to a node whose rooted subgraph encodes a Rashomon set for that subproblem under some budget $\varepsilon_{\textrm{abs}}$, and (3) one threshold-to-proxy-completion map per continuous feature. We additionally maintain a threshold registry $\mathcal{S}$ containing the active thresholds and per-feature bounds that we lazily infer from thresholds found to be constant on $D$. We elaborate more on these in the following subsections. 

\paragraph{Storing the Rashomon Set.}

We encode the information needed to recover the Rashomon set via an (acyclic) AND/OR graph. Each OR node represents a subproblem (represented by a set of data points), remaining depth, and budget, meaning the largest subtree objective we are willing to enumerate. In the algorithms, $G$ denotes a pointer to the OR node for the current
subproblem. We refer to $G$ as a subgraph, because the solutions for the subproblem are rooted at $G$. During the execution of our algorithms, each OR node stores its current budget, the minimum objective of any tree rooted at that node, and the split/leaf choices that we know can be completed within the budget. Each of those split choices has edges to its left and right child subgraphs, so the Rashomon set can be recovered by traversing these choices in the AND/OR graph. After the search terminates, we perform a post-processing step that builds objective histograms at each OR node, which allows trees to be indexed efficiently in sorted order of objective value. Before this post-processing step, we refer to the structure as a \textit{minimum-objective AND/OR graph}. 

\paragraph{Existing Caching. }
Separate from the AND/OR graph, we adopt the proxy optimizer framework of \citet{heile2026},  which caches proxy objective values and the intermediate computations used to obtain them. When LicketySPLIT is used as the proxy, this  cache includes the objective values returned by each greedy completion and its recursive calls, as well as those returned by each LicketySPLIT call and its recursive calls. These memoized values are retained and reused throughout the Rashomon set computation.

\paragraph{Subgraph Caching. }

In addition to the cache described above, we maintain an index into the AND/OR graph that maps each subproblem-depth pair to its canonical OR node. The index is budget-independent: there is at most one OR node for each subproblem and remaining depth (see lines 1 and 3 of Algorithm \ref{alg:extend-subgraph-continuous}). When a subproblem that was previously solved under a smaller budget is later reached with a larger budget, we start from the existing subgraph and extend it in place (line 8 and after), rather than rebuilding it. Accordingly, each OR node records the largest budget for which it has been solved (because of line 8). Conversely, if the existing subgraph contains solutions for a budget larger than the requested budget, we return a pointer to it and defer the budget restriction to solution extraction. (lines 4-6, details of extraction in Appendix \ref{app:andor-caching}). During iterative budget refinement (Appendix \ref{app:iterative-budget-refinement}), our mechanism reuses and expands existing child subgraphs (called in lines 15 and 18). In contrast, the mechanism of \citet{heile2026} repeatedly recurses from scratch as larger child budgets become available. Appendix \ref{app:theorems} establishes that this yields a more compact
representation than existing methods. Moreover, with an approximate proxy, solving under a larger budget can recover new trees that satisfy a smaller budget, improving approximation quality. 

\paragraph{Leaves.} At each OR node, we add any prediction leaf whose objective is within
$\varepsilon_{\mathrm{abs}}$ (line 9). We only add leaves that are not already stored in $G$. 

\paragraph{Binary Features.}
For a binary feature, \textsc{EnumerateBinaryFeature} (see Appendix \ref{app:continuous-rset}) partitions $D$ into
$D_L$ and $D_R$, computes proxy objectives $P_L$ and $P_R$ for child
subproblems, and discards the split if $P_L+P_R>\varepsilon_{\mathrm{abs}}$.
Otherwise, it calls \textsc{AddOrExtendSplit} (Appendix \ref{app:add-or-extend-split}) to construct or extend the child
subgraphs and add the split to $G$.

\begin{algorithm}[!t]
\caption{\textsc{ArborEnum}$(G,D,d,\gamma,\varepsilon_{\textrm{abs}},\mathcal{S})$}
\label{alg:extend-subgraph-continuous}
\begin{algorithmic}[1]
\INPUT Empty or existing Rashomon subgraph $G$ for subproblem $D$ and depth $d$, per-leaf penalty $\gamma$, new budget $\varepsilon_{\textrm{abs}}$, and threshold registry $\mathcal{S}$
\STATE $k \gets \textsc{SubproblemKey}(D,d)$ \COMMENT{Identify the subproblem by its samples and remaining depth}
\IF[Access global subgraph index $\mathcal{I}_{G}$]{$k \in \mathcal{I}_{G}$}
    \STATE $G \gets \mathcal{I}_{G}[k]$ \COMMENT{Retrieve the canonical subgraph for this subproblem-depth pair}
    \IF{$\textsc{Budget}(G) \ge \varepsilon_{\textrm{abs}}$}
        \STATE \textbf{return} $G$ \COMMENT{The desired solution is already encoded in the cached subgraph}
    \ENDIF
\ENDIF
\STATE $\textsc{SetBudget}(G,\varepsilon_{\textrm{abs}})$ \COMMENT{Record the largest budget for which this node must be solved}
\STATE $\textsc{AddFeasibleLeaves}(G,D,\gamma,\varepsilon_{\textrm{abs}})$ \COMMENT{Add feasible leaves that are not already stored in $G$}
\IF{$d=0$ \textbf{ or } $\varepsilon_{\textrm{abs}}<2\gamma$}
    \STATE $\mathcal{I}_{G}[k] \gets G$
    \STATE \textbf{return} $G$ \COMMENT{No split fits in the remaining budget}
\ENDIF
\FOR{\textbf{each continuous feature} $c$}
    \STATE $\begin{aligned}
        &\textsc{EnumContFeature}(G,D,c,d,\gamma,\varepsilon_{\textrm{abs}},\mathcal{S})
    \end{aligned}$ \COMMENT{Enumerate the currently available nonconstant thresholds of feature $c$; $\mathcal{S}$ and $G$ are updated by reference}
\ENDFOR
\FOR{\textbf{each binary feature} $j$}
    \STATE $\textsc{EnumerateBinary}(G,D,j,d,\gamma,\varepsilon_{\textrm{abs}},\mathcal{S})$ \COMMENT{Construct or extend the child subgraphs if the budget permits}
\ENDFOR
\STATE $\mathcal{I}_{G}[k] \gets G$ \COMMENT{Store a pointer to the canonical root of this subgraph}
\STATE \textbf{return} $G$ \COMMENT{The Rashomon set is encoded in $G$}
\end{algorithmic}
\end{algorithm}

\paragraph{Handling Continuous Features.}

For continuous features, we assume that each continuous feature has been expanded
into a contiguous, ordered block of threshold columns. A threshold column
corresponds to the rule $x_j \le \nu$ (for some value $\nu$). These columns are nested: $x_j \leq \nu \implies x_j \leq \nu'$ for $\nu' > \nu$. 
We exploit this ordering to prune thresholds. For example, suppose we know that a certain threshold column is always zero for each row in the current subproblem. Then, we can exclude any lower-indexed threshold because it will create an empty leaf. We use $\mathcal{S}$ to track lower and upper bounds on the non-constant thresholds and eliminate thresholds outside this range in \textsc{InitAndPrune}. We update $\mathcal{S}$ lazily in \textsc{EnumContFeature} as thresholds are explored.

\paragraph{Continuous Feature Pruning.}

Within a continuous feature, \textsc{EnumContFeature} searches over threshold intervals in the spirit of ConTree \citep{brița2025optimal}. 
We modify the pruning rules to account for the Rashomon budget and per-leaf penalty and maintain a map of thresholds to proxy completions to exploit repeated visits to the same subproblem, possibly with different budgets. We allow the search to be over a subset of thresholds, included in $\mathcal{S}$, which we call active thresholds. In general, all thresholds are active, but we support activating only a subset of them for use by the anytime algorithm (Algorithm \ref{main:anytime})

\begin{figure*}[!h]
    \centering
    \includegraphics[width=1.64\columnwidth]{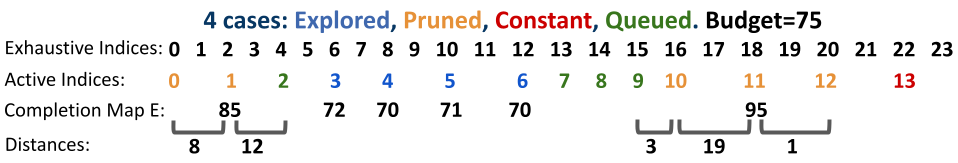}
   \caption{\textsc{InitAndPrune} example. Using proxy completions in $E$ for exhaustive thresholds 2 and 18, we prune active thresholds 0, 1, 10, 11, and 12 using active sample distances (for instance, $85 - 8 > 75$, so we prune active threshold 0, whereas $85 - 12 \leq 75$, so this test cannot prune active threshold 2. For simplicity, we show $E$ storing the sum of proxy completions (not $P_L$ and $P_R$); we omit pruning based on $\gamma$. From $\mathcal{S}$ (not shown), thresholds 22 and above are constant. We recursively construct subgraphs for active thresholds 3--6. \textsc{InitAndPrune} returns $[2,2]$ and $[7,9]$ (the remaining active threshold indices) for \textsc{EnumContFeature}.}
    \label{fig:subproblem-continuous}
\end{figure*}

\paragraph{Specifics of Pruning.}
For each continuous feature, \textsc{EnumContFeature} (Algorithm~\ref{alg:enumerate-continuous-feature}) retrieves a queue $Q$ of intervals over active threshold indices that remain to be explored. This queue is created by \textsc{InitAndPrune} (Algorithm~\ref{alg:initialize-continuous-feature-queue}). On the first visit to a subproblem, no proxy completions have been stored, and
Algorithm~\ref{alg:initialize-continuous-feature-queue} returns a range spanning all active threshold indices (outside of what $\mathcal{S}$ knew was constant). On later visits, it retrieves a map $E$, associated with the current subproblem, remaining depth, and continuous feature.  Each entry in $E$ maps a threshold index $t$ to the previously computed proxy objectives $P_L,P_R$ for its left and right subproblems. This structure is separate from the proxy algorithm's caching: $E$ provides information about the children subproblems that we have evaluated, not a solution to the subproblem. Algorithm~\ref{alg:initialize-continuous-feature-queue} uses these stored values to prune portions of the search space (Algorithm~\ref{alg:initialize-continuous-feature-queue}, lines 7-24);  Figure \ref{fig:subproblem-continuous} illustrates how different active threshold indices are ruled out to form $Q$.  Storing the proxy-completed objectives, rather than only whether each threshold was previously pruned, allows for work to be reused under different budgets. Additionally, $E$ is keyed by indices in the exhaustive threshold set, not by active-threshold positions, so it can be reused as additional thresholds become active. All subsequent logic refers to a threshold by its index in the active threshold set.

We track pruned or explored active thresholds as an ordered collection of disjoint index ranges, whose complement is the set of intervals that still need to be explored (line 25). \textsc{ExcludedRangeTracker} maintains this collection using a red-black tree of pairs $(a,b)$ ordered by $a$ and repeatedly merging a new range with any ranges that overlap or touch it.

When looping over the stored proxy completions (line 7), if we come across a threshold whose proxy-completed objective is now within budget, we go ahead and solve it for the first time or extend the subgraph; then it is no longer under consideration (lines 13-15). If the proxy-completed objective is outside the budget, we try to prune nearby thresholds (lines 16-22). We also update global bounds when a child attains the minimum possible cost $\gamma$: if the left child has cost $\gamma$, moving the threshold farther left cannot improve the optimal completions; similarly for the right child (lines 17-18).

We also prune thresholds neighboring any split whose proxy-completed objective exceeds the budget. For two threshold indices $s$ and $t$ in the same continuous-feature group, let $X_s$ and $X_t$ be the corresponding binary columns. Define their active-sample distance on $D$ as
\[
\mathrm{dist}_D(s,t)
=
\left|
\{i \in D : X_s(i)\neq X_t(i)\}
\right|
\]
This is the number of active samples (in the current subproblem) whose branch assignment changes when moving the threshold from $s$ to $t$. 
\begin{algorithm}[!t]
\caption{\textsc{InitAndPrune}$(G,D,c,d,\gamma,\varepsilon_{\textrm{abs}},\mathcal{S})$}
\label{alg:initialize-continuous-feature-queue}
\begin{algorithmic}[1]
\INPUT Current OR node $G$, subproblem $D$, continuous feature identifier $c$, depth $d$, budget $\varepsilon_{\textrm{abs}}$, and threshold registry $\mathcal{S}$ \COMMENT{Called in \textsc{EnumContFeature}}
\STATE $k \gets \textsc{SubproblemKey}(D,d)$
\STATE $E \gets \mathcal{C}_{\textrm{map}}(k,c)$ \COMMENT{Fetch the map of exhaustive indices to cached left and right proxy completions}
\STATE $(L,U) \gets \textsc{RestrictRange}(\mathcal{S},D,c)$ \COMMENT{Explore only active indices not known to be constant on $D$}
\STATE $\mathcal{R}_{\textrm{excluded}} \gets \textsc{ExcludedRangeTracker}()$ \COMMENT{Track ranges of active thresholds that are already handled or pruned}
\STATE \COMMENT{Process cached proxy completions in the desired range}
\FOR{\textbf{each} $t \in \textsc{Keys}(E)$}
    \STATE $q \gets \textsc{ActivePosition}(\mathcal{S},D,c,t)$ \COMMENT{$t$ is an exhaustive threshold-column index . Get its active index $q$.}
    \IF{$q<L$ \textbf{ or } $q>U$}
        \STATE \textbf{continue} \COMMENT{The threshold is outside search bounds}
    \ENDIF
    \STATE $(P_L,P_R) \gets E(t)$; \quad $P \gets P_L+P_R$
    \IF{$P \le \varepsilon_{\textrm{abs}}$}
        \STATE $\begin{aligned}
            &\textsc{AddOrExtendSplit}(G,D,t,d,\gamma,\varepsilon_{\textrm{abs}},
            \mathcal{S},P_L,P_R)
        \end{aligned}$ \COMMENT{Create or extend the child subgraphs for the current active thresholds}
        \STATE $\mathcal{R}_{\textrm{excluded}}.\textsc{MarkExplored}(q)$ \COMMENT{The threshold has already been evaluated; we mark it as explored}
    \ELSE
    \STATE \COMMENT{Moving the threshold farther left/right cannot improve if that side is already a pure leaf}
        \STATE \textbf{if } $P_L=\gamma$ \textbf{ then } $L \gets \max\{L,q+1\}$ 
        \STATE \textbf{if } $P_R=\gamma$ \textbf{ then } $U \gets \min\{U,q-1\}$ 
        \STATE $\Delta \gets P-\varepsilon_{\textrm{abs}}$ \COMMENT{Compute by how much the proxy completion exceeds the budget}
        \STATE \COMMENT{Closest active indices that are  $\geq \Delta$ samples away}
        \STATE $\ell \gets \textsc{LeftBoundary}(D,\mathcal{S},c,t,q-1,L,\Delta)$
        \STATE $r \gets \textsc{RightBoundary}(D,\mathcal{S},c,t,q+1,U,\Delta)$
       \STATE $\mathcal{R}_{\textrm{excluded}}.\textsc{PruneInterval}(\ell+1,r-1)$ \COMMENT{Prune the neighborhood that can't improve within budget}
    \ENDIF
\ENDFOR
\STATE $Q \gets \mathcal{R}_{\textrm{excluded}}.\textsc{Complement}(L,U)$ \COMMENT{Return intervals of active thresholds that still require exploration.}
\STATE \textbf{return} $(E,L,U,Q)$
\end{algorithmic}
\end{algorithm}

\begin{algorithm}[!t]
\caption{\textsc{EnumContFeature}
($G$,$D$,$c$,$d$,$\gamma$,$\varepsilon_{\textrm{abs}}$,$\mathcal{S})$}
\label{alg:enumerate-continuous-feature}
\label{alg}
\begin{algorithmic}[1]
\INPUT Current OR node $G$, subproblem $D$, continuous feature identifier $c$,
depth $d$, budget $\varepsilon_{\textrm{abs}}$, and threshold registry
$\mathcal{S}$

\STATE $\begin{aligned}
(E,L,U,Q) \gets{}&
\textsc{InitAndPrune}(G,D,c,d,\gamma,
\varepsilon_{\textrm{abs}},\mathcal{S})
\end{aligned}$

\STATE \COMMENT{Explore the remaining intervals}
\WHILE{$Q \neq \emptyset$}
\STATE $[i,j] \gets Q.\textsc{Pop}()$
\COMMENT{The endpoints are active threshold positions}

\STATE $i \gets \max\{i,L\}$; \quad
$j \gets \min\{j,U\}$

\IF{$i>j$}
    \STATE \textbf{continue}
\ENDIF

\STATE $m \gets \lfloor(i+j)/2\rfloor$

\COMMENT{See Appendix \ref{app:enumerate-continuous}. We evaluate the midpoint, update $E$ and $\mathcal{S}$ by reference, tighten $L$ and $U$, and prune and/or enqueue parts of the interval (into $Q$)}
\STATE $\begin{aligned}
(L,U) \gets{}&
\textsc{ProcessInterval}
(G,D,c,d,\gamma,\\
&\qquad\varepsilon_{\textrm{abs}},\mathcal{S},E,Q,L,U,i,j,m)
\end{aligned}$

\ENDWHILE
\end{algorithmic}
\end{algorithm}
Given a failed threshold, \textsc{RightBoundary} searches through the active thresholds to its right and returns the first whose distance from the failed threshold is at least $\Delta$. If no such threshold exists within the search interval, it returns the position immediately beyond the interval. \textsc{LeftBoundary} searches to the left symmetrically.

These routines implement the pruning implied by this assumption: if changing the branch assignment of $k$ samples can change the proxy-completed objective by at most $k$, then any threshold within distance less than $\Delta$ of a failed threshold must also exceed the budget. Accordingly, the routines move the interval endpoints just past all such thresholds. This pruning is exact for any proxy satisfying the assumption, including an optimal proxy over all thresholds or any subset of thresholds, as well as a proxy that simply predicts the majority-class leaf. These routines can be implemented by binary search, because $\mathrm{dist}_D$ is monotone when one threshold is fixed, and each distance computation can stop as soon as $\Delta$ differing active samples have been found.

After line 23, we have now exhausted all of the information we have cached about the subproblem. We return a queue of the remaining intervals to explore, formed by taking the complement of the intervals that have already been pruned or re-evaluated (within the bounds for thresholds). Algorithm~\ref{alg:enumerate-continuous-feature} processes this queue of intervals, evaluating the midpoint threshold of each interval as it is popped. \textsc{ProcessInterval} evaluates the midpoint threshold and its proxy completions. If the proxy-completed objective is within budget, it adds the split and requeues the neighboring left and right subintervals; otherwise, it applies the interval-pruning rules described above.

\paragraph{Proxy Algorithms with Continuous Features}

In Appendix \ref{app:licketysnip}, we describe our modifications to LicketySPLIT \citep{babbar2025near} for continuous features, creating LicketySNIP. We use interval-pruning techniques similar to those in  Algorithms \ref{alg:initialize-continuous-feature-queue} and \ref{alg:enumerate-continuous-feature} that are valid for optimal completions, but we apply them to greedy completions. We also discuss modifications to the greedy subroutine so that consecutive thresholds can be evaluated efficiently \citep{quinlan2014c45}.

\paragraph{Improvements without Caching.} In Appendix \ref{app:walk}, we describe algorithmic improvements when not caching proxy solutions. For example, once the Rashomon sets for the two children of a split have been computed, their minimum objectives can seed the iterative budget refinement for neighboring splits, allowing us to bypass the proxy-pruning test with provably no loss in quality (and a possible gain in quality).

\paragraph{Further Approximation with Proxy Algorithms.}

Although the proxy could also operate over continuous thresholds, restricting it to a fixed binarization still provides theoretical guarantees relative to existing  Rashomon set algorithms that require binarization. In particular, if the proxy
is optimal over the fixed binarization, then
Algorithm~\ref{alg:extend-subgraph-continuous} recovers a superset of the trees
returned by any method that enumerates the Rashomon set over that binarization (Appendix \ref{app:theorems}).

\paragraph{Anytime Algorithm.}

\begin{algorithm}[!t]
\caption{\textsc{AnytimeArborEnum}
\label{main:anytime}
$(d,\gamma,\varepsilon_{\textrm{mult}},\mathcal{B}_{\textrm{proxy}})$}
\label{alg:anytime-continuous-rset}
\begin{algorithmic}[1]
\INPUT Depth budget $d$, per-leaf penalty $\gamma$, Rashomon multiplier
$\varepsilon_{\textrm{mult}}$, and sorted list of proxy thresholds
$\mathcal{B}_{\textrm{proxy}}$

\STATE $\mathcal{B}_{\textrm{bin}} \gets$
sorted list of indices of ordinary binary features

\STATE $\mathcal{B}_{\textrm{initial}} \gets
\textsc{SortUnique}
(\mathcal{B}_{\textrm{bin}} \cup \mathcal{B}_{\textrm{proxy}})$

\STATE $\mathcal{S}_{\textrm{root}} \gets
\textsc{InitializeThresholdRegistry}(\mathcal{B}_{\textrm{initial}})$
\COMMENT{Initialize the active thresholds and continuous-feature bounds (no restrictions yet)}

\STATE $D_{\textrm{root}} \gets$
the bitvector containing all training samples

\STATE Restrict proxy algorithms to $\mathcal{B}_{\textrm{proxy}}$

\STATE $\varepsilon_{\textrm{abs}} \gets
(1+\varepsilon_{\textrm{mult}})
\textsc{Proxy}
(D_{\textrm{root}},d,\gamma,\mathcal{S}_{\textrm{root}})$
\COMMENT{Initialize the root budget}

\STATE \COMMENT{Initially solve using a small set of active thresholds}
\STATE $\begin{aligned}
G \gets{}&
\textsc{ArborEnum}
(G,D_{\textrm{root}},d,\gamma, \varepsilon_{\textrm{abs}},
\mathcal{S}_{\textrm{root}})
\end{aligned}$
\WHILE{\textbf{not}
$\textsc{AllThresholdsActive}(\mathcal{S}_{\textrm{root}})$}

    \STATE $\mathcal{B}_{\textrm{new}} \gets
    \textsc{SelectNewThresholds}(\mathcal{S}_{\textrm{root}})$
    \COMMENT{By default, select one threshold from each gap between adjacent active thresholds}

    \STATE $\textsc{ActivateThresholds}
    (\mathcal{S}_{\textrm{root}},\mathcal{B}_{\textrm{new}})$
    \COMMENT{Update the threshold registry by reference}

    \STATE \text{Clear }$\mathcal{I}_{G}$
    \COMMENT{Cached subgraphs may be incomplete for the expanded active threshold set}

    \STATE $\mathcal{V} \gets \emptyset$
    \COMMENT{Track graph nodes visited during this refinement pass}

    \STATE $\textsc{RefineGraph}
    (G,D_{\textrm{root}},d,
    \mathcal{S}_{\textrm{root}},\mathcal{V})$
    \COMMENT{Extend the graph to include newly active thresholds when feasible}
\ENDWHILE

\WHILE{\textbf{not}
$\textsc{IsProxyOptimal}(\textsc{Proxy})$}

    \STATE Increase proxy strength by $1$ and update caches
    \COMMENT{See Appendix~\ref{app:licketysnip}}

    \STATE $\mathcal{V} \gets \emptyset$

    \STATE $\textsc{RefineGraph}
    (G,D_{\textrm{root}},d,
    \mathcal{S}_{\textrm{root}},\mathcal{V})$
    \COMMENT{Revisit the graph using the stronger proxy to recover falsely pruned splits}
\ENDWHILE

\STATE \textbf{return} $G$
\end{algorithmic}
\end{algorithm}

\begin{table*}[!t]
\centering
\small
\setlength{\tabcolsep}{4pt}
\caption{
 Runtime on datasets using fully-continuous features (i.e., for existing methods, binarizing between every pair of unique values for each feature). $d=5;\lambda=0.02; \varepsilon=0.03$.
Time is reported by rounding to the nearest second. The approximate algorithms are shown in \textbf{bold}, along with the best runtime among them for each dataset; likewise, the optimal algorithms and the best runtime among them are \underline{underlined}.
 Mean and standard deviation are shown across 3 bootstraps. The number after each dataset name is $\sum_j (u_j - 1)$, where $u_j$ is the number of unique values of feature $j$.
Results for all datasets are shown in Appendix \ref{app:timing-memory-recall}. -- runs exceed 100-hours or 128GB memory on at least one bootstrap. ).}
\label{tab:exhaustive-runtime}
\resizebox{\textwidth}{!}{%
\begin{tabular}{lrrrrrrrr}
\toprule
Dataset
& \textbf{ArborEnum+LSR}
& \textbf{ArborEnum+SNIP}
& \textbf{ArborEnum+SNIP+GR}
& \textbf{PRAXIS}
& \textbf{RESPLIT}
& \underline{ArborEnum+OPT}
& \underline{SORTD}
& \underline{TreeFARMS} \\
\midrule
Abalone-6039 & 264 $\pm$ 401 & 2811 $\pm$ 4187 & \textbf{205 $\pm$ 312} & -- & -- & \underline{55174 $\pm$ 47867} & -- & -- \\
Adult-27237 & \textbf{240 $\pm$ 26} & 12132 $\pm$ 372 & 292 $\pm$ 16 & -- & -- & \underline{141689 $\pm$ 13771} & -- & -- \\
Bank-9530 & 15 $\pm$ 1 & 213 $\pm$ 5 & \textbf{6 $\pm$ 0} & -- & -- & \underline{20978 $\pm$ 2133} & -- & -- \\
Bike-279 & \textbf{8 $\pm$ 2} & 27 $\pm$ 5 & \textbf{8 $\pm$ 2} & 117 $\pm$ 25 & 3813 $\pm$ 306 & \underline{1380 $\pm$ 57} & 86351 $\pm$ 15523 & -- \\
Churn-16400 & \textbf{7 $\pm$ 6} & 644 $\pm$ 556 & 9 $\pm$ 9 & -- & -- & \underline{257505 $\pm$ 25437} & -- & -- \\
Credit-174581 & 74 $\pm$ 43 & 24850 $\pm$ 619 & \textbf{66 $\pm$ 22} & -- & -- & -- & -- & -- \\
Diamonds-2074 & 46 $\pm$ 3 & 313 $\pm$ 24 & \textbf{39 $\pm$ 3} & 77628 $\pm$ 6710 & -- & \underline{16616 $\pm$ 1024} & -- & -- \\
Helena-1540012 & 213 $\pm$ 17 & 286988 $\pm$ 12347 & \textbf{186 $\pm$ 11} & -- & -- & -- & -- & -- \\
\bottomrule
\end{tabular}%
}
{\footnotesize
LSR = LicketySPLIT-Restricted; SNIP = LicketySNIP; SNIP+GR = LicketySNIP with greedy restricted to a small binarization; OPT = optimal proxy.
}
\end{table*}

\textsc{AnytimeArborEnum} (Algorithm~\ref{alg:anytime-continuous-rset}) starts from a coarse threshold set and progressively activates additional thresholds, thereby reducing the covering radius of the active thresholds relative to all thresholds and, consequently, the discretization-induced optimality gap. We present the simpler case in which the proxy remains restricted to the initial binarization ($\mathcal{B}_{\textrm{proxy}}$), preserving the aforementioned guarantee relative to existing methods on the binarization. More generally, the proxy may access all thresholds. In this setting, whenever the proxy selects a split at a subproblem, we add that split to the subproblem’s active threshold set, ensuring that every proxy-certified tree remains recoverable (additional details in Appendix \ref{app:anytime}). This variant is truly anytime: it can begin with no active continuous thresholds and, as the proxy is strengthened to optimality, converges to the complete continuous-feature Rashomon set. We discuss proxy strength further in Appendix \ref{app:licketysnip}. In practice, this final step is optional, since our proxy algorithms attain near-perfect recall (Table \ref{tab:exhaustive-recall-summary}).

The initial pass runs Algorithm~\ref{alg:extend-subgraph-continuous} over an initial set of splits. This call returns a minimum-objective AND/OR graph; objective histograms are not populated until the algorithm terminates, since they will become stale. Subsequent rounds activate new thresholds and call \textsc{RefineGraph} (Appendix \ref{app:refine-graph}) to update the graph. For each existing split, \textsc{RefineGraph} recursively refines its children, then performs iterative budget refinement, as new splits may improve a child’s minimum objective.

\section{Evaluation}

We organize our evaluation around three research questions: \textbf{1) Exact: } when can we compute exact Rashomon sets on continuous features, \textbf{2) Approximate: } can our approximations accurately recover the Rashomon set while substantially expanding the range of problems for which Rashomon sets can be computed efficiently, and \textbf{3) Anytime: } can our anytime algorithm produce useful intermediate Rashomon sets while incurring minimal overhead when run to completion. We use 20 datasets that have at least one continuous feature; additional experimental details are available in Appendix \ref{app:computational-resources}. All methods are given a 100-hour timeout and 128GB of memory to compute a Rashomon set for one bootstrap of a dataset, except for our anytime algorithms, which we gave a 24-hour timeout and 32GB of memory. We parameterize the leaf penalty as $\gamma=\mathrm{round}(\lambda |D|)$, and choose $\lambda$.

\paragraph{Runtime.} Table~\ref{tab:exhaustive-runtime} compares our methods (with different proxy algorithms) to existing methods in terms of runtime across datasets. Memory usage is shown in Appendix \ref{app:timing-memory-recall}; we are more memory-efficient than existing methods on almost all datasets. We report additional values of $\lambda$ in Appendix \ref{app:timing-memory-recall}. We consider four proxies: an optimal decision tree algorithm obtained by extending LicketySNIP to higher lookaheads, LicketySNIP (SNIP), LicketySNIP with the greedy subroutine restricted to a small binarization (SNIP+GR), and LicketySPLIT run directly on the small binarization (LSR). The binarization for LSR and SNIP+GR is obtained from thresholds selected by a gradient-boosted tree ensemble. The four existing methods are given fully exhaustive binarizations, with thresholds placed between every pair of unique values for each continuous feature, so that they can enumerate the continuous-feature Rashomon set exactly. On Bike (which had the fewest binary features by far), our optimal method finishes $63 \times$ faster than SORTeD, the only optimal method to finish; TreeFARMS did not finish.

\begin{figure*}[!t]
    \centering
    \includegraphics[width=1\linewidth]{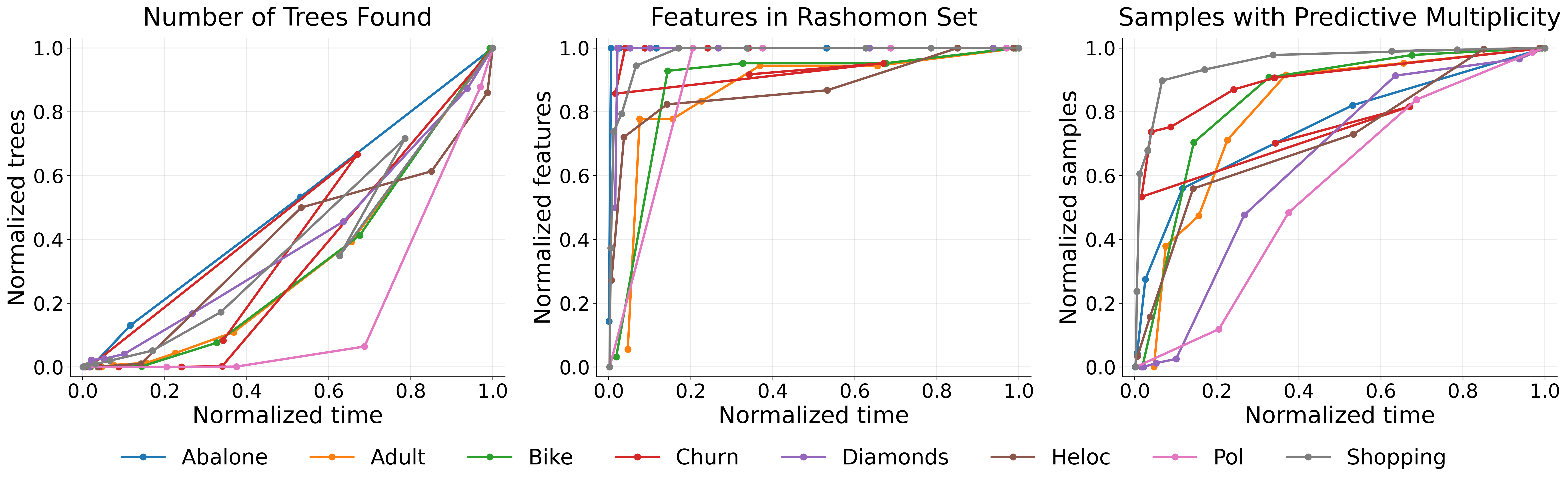}
    \caption{Anytime algorithm on 8 real-world datasets. We display three quantities about the Rashomon set at each stopping point: number of trees, the number of features present in the Rashomon set (any of the splits from a continuous feature counts as one feature), and the number of samples that receive conflicting predictions from Rashomon set members. The curve for a given dataset is normalized by its maximum value. $d=5;\lambda=0.005; \varepsilon=0.015$. }
    \label{fig:heloc-anytime}
\end{figure*}

\paragraph{Approximation Quality.} 

\begin{table}[!b]
\centering
\scriptsize
\setlength{\tabcolsep}{2.5pt}
\renewcommand{\arraystretch}{0.92}
\caption{
Recall summary across 20 datasets with fully continuous features, relative to the best method that finished. Entries are mean $\pm$ standard deviation across 3 bootstraps. Individual results appear in Appendix~\ref{app:timing-memory-recall}. $d=5$ and $\varepsilon=0.03$. \textbf{Bold}: $>0.99$; \underline{underline}: $>0.98$.
}
\label{tab:exhaustive-recall-summary}
\resizebox{\columnwidth}{!}{%
\begin{tabular}{llrrr}
\toprule
$\lambda$ & Statistic
& ArborEnum+LSR
& ArborEnum+SNIP
& ArborEnum+SNIP+GR \\
\midrule
0.02
& Min
& 0.667 $\pm$ 0.577
& \textbf{1.000 $\pm$ 0.000}
& \textbf{1.000 $\pm$ 0.000} \\
& Second
& 0.931 $\pm$ 0.120
& \textbf{1.000 $\pm$ 0.000}
& \textbf{1.000 $\pm$ 0.000} \\
& Median
& \textbf{1.000 $\pm$ 0.000}
& \textbf{1.000 $\pm$ 0.000}
& \textbf{1.000 $\pm$ 0.000} \\
\midrule
0.01
& Min
& 0.898 $\pm$ 0.176
& \textbf{0.994 $\pm$ 0.010}
& \underline{0.981 $\pm$ 0.032} \\
& Second
& 0.911 $\pm$ 0.114
& \textbf{1.000 $\pm$ 0.000}
& \underline{0.988 $\pm$ 0.019} \\
& Median
& \textbf{1.000 $\pm$ 0.000}
& \textbf{1.000 $\pm$ 0.000}
& \textbf{1.000 $\pm$ 0.000} \\
\midrule
0.005
& Min
& 0.880 $\pm$ 0.206
& 0.945 $\pm$ 0.096
& 0.945 $\pm$ 0.096 \\
& Second
& 0.949 $\pm$ 0.088
& \textbf{0.995 $\pm$ 0.009}
& \textbf{0.995 $\pm$ 0.009} \\
& Median
& \textbf{1.000 $\pm$ 0.000}
& \textbf{1.000 $\pm$ 0.000}
& \textbf{1.000 $\pm$ 0.000} \\
\bottomrule
\end{tabular}%
}
\end{table}

Table~\ref{tab:exhaustive-recall-summary} shows that using LicketySNIP as the proxy (with or without restricting the greedy subroutine to a smaller binarization, middle and right columns) achieves nearly-perfect recall across all datasets, while remaining substantially faster than the optimal proxy. The worst-case across all 20 datasets and three $\lambda$ values is still $\geq 94.5\%$. Restricting further to LicketySPLIT over a binarization (left column) is sometimes faster (and sometimes slower) but can incur moderate recall loss on several datasets. Overall, LicketySNIP+GR offers the best tradeoff for users willing to sacrifice some recall; when perfect recall and a certificate of optimality are required, the optimal proxy should be used.

\paragraph{Extending the Root Budget.}

\begin{table}[!b]
\centering
\scriptsize
\setlength{\tabcolsep}{2.5pt}
\caption{For every dataset on which ArborEnum+SNIP+GR does not achieve perfect recall (averaged across bootstraps), extending the root budget recovers the remaining trees with little additional runtime while remaining substantially faster than our optimal method. The timings are obtained by first solving with $\varepsilon_{\textrm{mult}}=0.03$ and then extending the subgraph. $\lambda=0.005$; $\lambda=0.01$ is shown in Appendix \ref{app:bigger-budget}.}
\label{tab:budget-extension-main}
\begin{tabular}{lrrrr}
\toprule
Dataset
& Recall
& ArborEnum+SNIP+GR\
& Expanded
& ArborEnum+OPT \\
\midrule
Bike
& $\textbf{.999} \!\to\! \textbf{1.000}$
& $191s$
& $227s$ ($\varepsilon=.0325$)
& $5158s$ \\
Pol
& $.945 \!\to\! \textbf{1.000}$
& $219s$
& $331s$ ($\varepsilon=.0375$)
& $98967s$ \\
Student
& $\textbf{.995} \!\to\! \textbf{1.000}$
& $16s$
& $21s$ ($\varepsilon=.0375$)
& $79s$ \\
\bottomrule
\end{tabular}
\end{table}

Because our subgraphs can be extended to a larger budget without repeating the subgraph computation, we can improve an approximate Rashomon set by repeatedly increasing the root budget without restarting the computation. SORTeD is an anytime algorithm in the root budget $\varepsilon_{\textrm{abs}}$, and such anytime algorithms have not been created previously for approximation algorithms. Table \ref{tab:budget-extension-main} shows that loosening the root budget can recover the remaining trees with limited additional runtime. 

This capability is useful in workflows where many candidate Rashomon sets are computed during development: after comparing them, a user can select one and invest additional computation to improve its approximation quality. Likewise, this capability can facilitate the selection of $\varepsilon_{\textrm{mult}}$: the graph can be extended incrementally across increasingly large values until a desired stopping criterion is reached.

\paragraph{Adding Thresholds Anytime.} Figure~\ref{fig:heloc-anytime} shows one variant of the anytime algorithm (stopping when all thresholds are added; not increasing proxy strength). We use LicketySNIP as a fully continuous proxy (for simplicity, we don't restrict greedy to a binarization). Enumeration begins without considering any continuous-feature thresholds and adds them over time. Each point enumerates the Rashomon set over a binarization of increasing complexity (possibly with some proxy thresholds included).  The results show that a coarse binarization provides an incomplete view of the Rashomon set over all thresholds. As the binarization is refined, downstream properties converge more quickly than does the number of trees: important features are typically identified early, although some datasets continue to gain features late in the run, while predictive multiplicity converges more slowly. Ideally, enumeration would always run to completion, but this is often infeasible. This flexibility of early-stopping comes at little additional cost: relative to running the corresponding non-anytime algorithm on the final set of thresholds, the anytime procedure incurs a median runtime overhead of only 2.7\%. Thus, the anytime variant can replace our non-anytime algorithms, running to completion when feasible and stopping early when necessary.

\section{Conclusion}

We introduced ArborEnum, a framework for enumerating decision-tree Rashomon sets using continuous features. ArborEnum provides exact, approximate, and anytime algorithms, allowing practitioners to choose an approach suited to their computational constraints. By combining continuous-threshold pruning with budget-independent subgraphs, our methods improve scalability without coarse binarization, while the anytime algorithm progressively considers more splits as time permits. Future work could extend our algorithms to Rashomon sets of piecewise-constant or piecewise-linear regression trees or modify the pruning to ensure recovery of rule lists. Another promising direction is to adapt the pruning conditions to obtain theoretical guarantees with respect to Rashomon sets of rule lists.

\bibliography{references}

\newpage

\appendix

\section*{Appendix Contents}

\appsectionentry{app:full-algorithms}{Full Algorithms}
\appsubsectionentry{app:implementation}{Implementation}
\appsubsectionentry{app:continuous-rset}{ArborEnum}
\appsubsectionentry{app:init-and-prune}{InitAndPrune}
\appsubsectionentry{app:enumerate-continuous}{EnumContFeature}
\appsubsectionentry{app:iterative-budget-refinement}
  {Iterative Budget Refinement}
\appsubsectionentry{app:refine-graph}{RefineGraph}
\appsubsectionentry{app:add-or-extend-split}{AddOrExtendSplit}
\appsubsectionentry{app:anytime}{Anytime Algorithm}
\appsubsectionentry{app:additional-methods}{Additional Methods}
\appsubsectionentry{app:andor-caching}{AND/OR Graph Caching}

\appsectionentry{app:theorems}{Theoretical Results}

\apptheorementry{thm:budget-independent}
  {Distinct Budgets for One Subproblem}

\apptheorementry{thm:model-set-duplication}
  {Distinct Objectives for One Subproblem}

\apptheorementry{thm:path-duplication}
  {Factorial Path Duplication}

\apptheorementry{thm:snapping}
  {Binarization Optimality Gap}

\apptheorementry{thm:rank-net}
  {Selecting Thresholds with Covering Radius Guarantee}

\apptheorementry{cor:rank-net-gap}
  {Optimality Gap induced by Threshold-Selection}

\apptheorementry{thm:rank-net-size}
  {Number of Thresholds for some Covering Radius}

\apptheorementry{thm:quantile-snapping}
  {Quantile Snapping Bound}

\apptheorementry{thm:rank-refinement-halves}
  {Midpoint Refinement Halves Covering Radius}

\apptheorementry{thm:fixed-binarization-superset}
  {Fixed-Binarization Superset Guarantee}

\appsectionentry{app:proxy-algorithms}{Proxy Algorithms}

\appsubsectionentry{app:licketysnip}
  {LicketySNIP and Proxy Strength}

\appsubsectionentry{app:walk}
  {Improvements without Proxy Caching}

\apptheorementry{thm:neighboring-threshold-walk}
  {Neighboring-Threshold Pruning-Test}

\appsectionentry{app:datasets}{Datasets}

\appsectionentry{app:experiments}{Experiments}

\appsubsectionentry{app:computational-resources}
  {Computational Resources}

\appsubsectionentry{app:timing-memory-recall}
  {Timing, Memory, and Recall}

\appsubsectionentry{app:anytime-overhead}
  {Anytime Overhead}

\appsubsectionentry{app:bigger-budget}
  {Improving Recall with a Bigger Budget}

\appsubsectionentry{app:budget-independent-results}
  {Budget-Independent Subgraphs}

\section{Full Algorithms}\label{app:full-algorithms}

This appendix gives the full pseudocode for the algorithms described in the main text. 

\subsection{Implementation}\label{app:implementation}

\paragraph{Subproblem representation.}
We represent each subproblem $D$ as a bitvector over the training samples, with one bit indicating whether each sample is active in the subproblem. Bitvectors are stored as arrays of 64-bit words, so intersections, differences, and sample counts can be computed efficiently using bitwise operations and popcount. Feature columns and class-label indicators are represented in the same way. To reduce memory usage, we identify each cached subproblem by a 64-bit fingerprint of its bitvector together with its remaining depth; this mapping is not strictly injective, but the probability of a collision over the bitvectors encountered (with the same remaining depth) in a run is astronomically small \citep{heile2026}.

\subsection{ArborEnum}\label{app:continuous-rset}

Algorithm~\ref{alg:extend-subgraph-continuous} gives the complete
\textsc{ArborEnum} procedure. Here, we expand the two helper routines that
handle feasible prediction leaves and ordinary binary features. Continuous
features are handled by \textsc{EnumContFeature}
(Algorithm~\ref{alg:enumerate-continuous-feature}), and feasible splits are
materialized or extended by \textsc{AddOrExtendSplit}
(Algorithm~\ref{alg:resolve-threshold}).

\paragraph{Feasible leaves.}
For each class label, \textsc{AddFeasibleLeaves} computes the objective of the
corresponding prediction leaf and adds it when it is within the current budget.
Because a cached subgraph may already contain leaves found under a smaller
budget, the routine adds only leaves that are not already stored in \(G\).

\begin{algorithm}[!h]
\caption{\textsc{AddFeasibleLeaves}%
\((G,D,\gamma,\varepsilon_{\textrm{abs}})\)}
\label{alg:add-feasible-leaves}
\begin{algorithmic}[1]
\INPUT Current OR node \(G\), subproblem \(D\), per-leaf penalty \(\gamma\),
and budget \(\varepsilon_{\textrm{abs}}\)
\FOR{\textbf{each} class label \(b\in\mathcal{Y}\)}
    \STATE \(\ell_b \gets
    \gamma + \left|\{(x_i,y_i)\in D:y_i\neq b\}\right|\)
    \IF{\(\ell_b\leq\varepsilon_{\textrm{abs}}\)
        \textbf{ and } \(G\) does not contain a leaf predicting \(b\)}
        \STATE \(\textsc{AddLeaf}(G,b,\ell_b)\)
        \COMMENT{Add the newly feasible prediction leaf}
    \ENDIF
\ENDFOR
\end{algorithmic}
\end{algorithm}

\paragraph{Faster leaf implementation.}
Algorithm~\ref{alg:add-feasible-leaves} is written for the general multiclass
setting, but its leaf objectives can often be computed more efficiently. For
example, in binary classification, let \(n=|D|\) and let \(n_1\) be the number
of positive examples in \(D\). Then \(n_0=n-n_1\), so the two possible leaf
objectives are obtained from a single class-count computation:
a leaf predicting \(0\) has objective \(\gamma+n_1\), whereas a leaf predicting
\(1\) has objective \(\gamma+n_0\). When \(D\) and the positive-label set are
stored as bitvectors, \(n_1\) can be computed with one bitwise intersection and
population count. Algorithm~\ref{alg:add-feasible-binary-leaves} gives this
specialized implementation.

\begin{algorithm}[!h]
\caption{\textsc{AddFeasibleBinaryLeaves}%
\((G,D,\gamma,\varepsilon_{\textrm{abs}})\)}
\label{alg:add-feasible-binary-leaves}
\begin{algorithmic}[1]
\INPUT Current OR node \(G\), subproblem \(D\), per-leaf penalty \(\gamma\),
and budget \(\varepsilon_{\textrm{abs}}\)
\STATE \(n \gets |D|\)
\STATE \(n_1 \gets |D \cap Y_1|\)
\COMMENT{One bitwise intersection and popcount}
\STATE \(n_0 \gets n-n_1\)
\STATE \(\ell_0 \gets \gamma+n_1\)
\COMMENT{Misclassified positives when predicting \(0\)}
\STATE \(\ell_1 \gets \gamma+n_0\)
\COMMENT{Misclassified negatives when predicting \(1\)}
\IF{\(\ell_0\leq\varepsilon_{\textrm{abs}}\)
    \textbf{ and } \(G\) does not contain a leaf predicting \(0\)}
    \STATE \(\textsc{AddLeaf}(G,0,\ell_0)\)
\ENDIF
\IF{\(\ell_1\leq\varepsilon_{\textrm{abs}}\)
    \textbf{ and } \(G\) does not contain a leaf predicting \(1\)}
    \STATE \(\textsc{AddLeaf}(G,1,\ell_1)\)
\ENDIF
\end{algorithmic}
\end{algorithm}

\paragraph{Binary features.}
For an ordinary binary feature \(j\), \textsc{EnumerateBinary} partitions the
current subproblem and ignores the feature if either child is empty. Otherwise,
it computes proxy objectives for the two children. If their sum is within the
parent budget, \textsc{AddOrExtendSplit} constructs or extends the corresponding
child subgraphs using the budget-refinement procedure described in
Appendix~\ref{app:add-or-extend-split}.

\begin{algorithm}[!h]
\caption{\textsc{EnumerateBinary}%
\((G,D,j,d,\gamma,\varepsilon_{\textrm{abs}},\mathcal{S})\)}
\label{alg:enumerate-binary-feature}
\begin{algorithmic}[1]
\INPUT Current OR node \(G\), subproblem \(D\), binary feature \(j\),
remaining depth \(d\), per-leaf penalty \(\gamma\), budget
\(\varepsilon_{\textrm{abs}}\), and threshold registry \(\mathcal{S}\)
\STATE \((D_L,D_R)\gets\textsc{Partition}(D,j)\)
\IF{\(D_L=\emptyset\) \textbf{ or } \(D_R=\emptyset\)}
    \STATE \textbf{return}
    \COMMENT{The feature is constant on \(D\)}
\ENDIF
\STATE \(P_L\gets\textsc{Proxy}(D_L,d-1,\gamma,\mathcal{S})\)
\STATE \(P_R\gets\textsc{Proxy}(D_R,d-1,\gamma,\mathcal{S})\)
\IF{\(P_L+P_R>\varepsilon_{\textrm{abs}}\)}
    \STATE \textbf{return}
    \COMMENT{The proxy completion exceeds the parent budget}
\ENDIF
\STATE \(\begin{aligned}
&\textsc{AddOrExtendSplit}
(G,D,j,d,\gamma,\varepsilon_{\textrm{abs}},\\
&\qquad\mathcal{S},P_L,P_R)
\end{aligned}\)
\COMMENT{Construct or extend the child subgraphs}
\end{algorithmic}
\end{algorithm}

\FloatBarrier

\subsection{InitAndPrune}\label{app:init-and-prune}

Algorithm~\ref{alg:initialize-continuous-feature-queue} gives the complete
\textsc{InitAndPrune} procedure. Here, we describe the implementation of
\textsc{ExcludedRangeTracker}, which records active threshold positions that
have already been explored or pruned. These positions are indices into the
ordered active thresholds of the current continuous feature: \(0,1,\ldots\).
Thus, the tracker operates directly on contiguous active-position indices and
does not require access to the threshold registry \(\mathcal{S}\).

The tracker stores disjoint, nonadjacent closed intervals \([a,b]\) in an
ordered map implemented as a red-black tree. Each tree entry is a pair
\((a,b)\), keyed and ordered by its left endpoint \(a\). This makes it so that
predecessor and successor intervals can be found in logarithmic time. When a new
excluded interval is inserted, the tracker merges it with every stored interval
that overlaps or touches it.

\textsc{MarkExplored} excludes one active threshold position and is implemented
as a call to \textsc{PruneInterval}. \textsc{PruneInterval} inserts an arbitrary
closed interval and restores the invariant that all stored intervals are
disjoint and nonadjacent. Here, \textsc{Left}(z) and \textsc{Right}(z) denote the left and right endpoints, respectively, of the interval stored at red-black-tree entry $z$. 

There are two cases to be aware of. 

\begin{itemize}
    \item There can be at most one interval to handle whose interval overlaps or touches the new interval on the left (by the invariant). That is, we merge the contiguous interval that overlaps or touches $\ell$, if it exists. There cannot be a chain of contiguous intervals by the invariant.
    \item There can be multiple intervals to the right of $\ell$ that become mergeable as $r$ expands. In particular, there may be single indices that all get absorbed by $[\ell, r]$ (i.e., they are inside the new interval but not contiguous: such as $[2,8]$ having $[4,4]$ and  $[6,6]$ inside)
    
\end{itemize}

\textsc{PruneInterval} first checks the interval immediately preceding $\ell$, since that is the only earlier interval that could overlap or touch the new interval (by the invariant). It then scans forward through successor intervals, repeatedly merging any interval whose left endpoint is at most $r+1$; each merge may increase $r$, which can cause the enlarged interval to reach additional successors. After all overlapping or adjacent intervals have been removed, the routine inserts the single merged interval $[\ell,r]$ into the red-black tree.

\textsc{Complement} scans the intervals from
left to right and returns the maximal unexplored intervals within the current
search bounds \([L,U]\). It scans the stored intervals from left to right while maintaining $q$, the first position that has not yet been covered or returned. That is, we have the invariant that (at every point in the scan) $q$ is the smallest position in \([L,U]\) that has not yet been covered by an excluded interval or added to $Q$. After adding a gap $[q,a-1]$, the update $q\gets\max\{q,b+1\}$ restores the invariant for the next iteration.

The implementations of \textsc{LeftBoundary} and
\textsc{RightBoundary}, which determine the intervals passed to
\textsc{PruneInterval}, are given in Appendix~\ref{app:additional-methods}.

\begin{algorithm}[!h]
\caption{\textsc{$\mathcal{R}_{\textrm{excluded}}$.MarkExplored}%
\((q)\)}
\label{alg:excluded-range-mark-explored}
\begin{algorithmic}[1]
\INPUT Excluded-range tracker \(\mathcal{R}_{\textrm{excluded}}\) and active
threshold position \(q\)
\STATE
\(\mathcal{R}_{\textrm{excluded}}.\textsc{PruneInterval}(q,q)\)
\COMMENT{Exclude the single explored position}
\end{algorithmic}
\end{algorithm}

\begin{algorithm}[!h]
\caption{$\mathcal{R}_{\textrm{excluded}}$.\textsc{PruneInterval}%
\((\ell,r)\)}
\label{alg:excluded-range-prune-interval}
\begin{algorithmic}[1]
\INPUT Excluded-range tracker \(\mathcal{R}_{\textrm{excluded}}\) and closed
interval \([\ell,r]\) of active threshold positions
\IF{\(\ell>r\)}
    \STATE \textbf{return}
\ENDIF
\STATE \(z\gets
\mathcal{R}_{\textrm{excluded}}.\textsc{UpperBound}(\ell)\)
\COMMENT{First stored interval whose left endpoint exceeds \(\ell\)}
\IF{\(z\) has a predecessor \(p\) \textbf{ and }
\(\textsc{Right}(p)+1\geq\ell\)}
    \STATE \(\ell\gets\textsc{Left}(p)\)
    \STATE \(r\gets\max\{r,\textsc{Right}(p)\}\)
    \STATE Delete \(p\) from \(\mathcal{R}_{\textrm{excluded}}\)
\ENDIF
\WHILE{\(z\neq\emptyset\) \textbf{ and }
\(\textsc{Left}(z)\leq r+1\)}
    \STATE \(r\gets\max\{r,\textsc{Right}(z)\}\)
    \STATE \(z\gets\) the successor of \(z\) after deleting \(z\)
\ENDWHILE
\STATE Insert \((\ell,r)\) into
\(\mathcal{R}_{\textrm{excluded}}\)
\COMMENT{Intervals remain disjoint and nonadjacent}
\end{algorithmic}
\end{algorithm}

\begin{algorithm}[!h]
\caption{\textsc{Complement}%
\((\mathcal{R}_{\textrm{excluded}},L,U)\)}
\label{alg:excluded-range-complement}
\begin{algorithmic}[1]
\INPUT Excluded-range tracker \(\mathcal{R}_{\textrm{excluded}}\) and current
search bounds \([L,U]\) over active threshold positions
\STATE \(Q\gets\emptyset\); \quad \(q\gets L\)
\FOR{\textbf{each} stored interval \((a,b)\) in increasing order of \(a\)}
    \IF{\(b<L\)}
        \STATE \textbf{continue}
    \ENDIF
    \IF{\(a>U\)}
        \STATE \textbf{break}
    \ENDIF
    \IF{\(q<a\)}
        \STATE \(Q.\textsc{Push}([q,\min\{a-1,U\}])\)
        \COMMENT{Add the unexplored gap preceding \([a,b]\)}
    \ENDIF
    \STATE \(q\gets\max\{q,b+1\}\)
    \IF{\(q>U\)}
        \STATE \textbf{break}
    \ENDIF
\ENDFOR
\IF{\(q\leq U\)}
    \STATE \(Q.\textsc{Push}([q,U])\)
\ENDIF
\STATE \textbf{return} \(Q\)
\end{algorithmic}
\end{algorithm}

\FloatBarrier

\subsection{\textsc{EnumContFeature}}
\label{app:enumerate-continuous}

Algorithm~\ref{alg:enumerate-continuous-feature} gives the complete
\textsc{EnumContFeature} procedure. It searches the intervals of active
threshold positions returned by \textsc{InitAndPrune}
(Algorithm~\ref{alg:initialize-continuous-feature-queue}). The queue \(Q\)
contains intervals that have not already been evaluated or pruned using cached
entries in the map \(E\). For each interval, the algorithm evaluates its
midpoint so that a failed proxy completion may prune thresholds on both sides.
The helper routine \textsc{ProcessInterval}, given in
Algorithm~\ref{alg:process-continuous-interval}, performs this evaluation and
updates \(E\), \(Q\), and the threshold registry \(\mathcal{S}\).

\paragraph{Threshold-support bounds.}
For each continuous feature \(c\), the threshold registry \(\mathcal{S}\)
stores the active thresholds and lower and upper bounds on the exhaustive
threshold indices that may still induce a nonconstant split. These bounds
summarize information that would otherwise be carried separately along every
path. For a threshold rule \(x_c\leq \nu_t\), an empty left child implies that
every smaller threshold is also constant on the current subproblem; an empty
right child implies the analogous fact for every larger threshold. Moreover, a
threshold that is constant on \(D\) remains constant on every descendant
subproblem \(D'\subseteq D\). Therefore, whenever
\textsc{ProcessInterval} discovers an empty child, it lazily records the
corresponding bound in \(\mathcal{S}\), rather than eagerly updating every
descendant or testing all thresholds in that direction.

The bounds stored in \(\mathcal{S}\) are expressed using exhaustive threshold
indices, so they remain valid if additional thresholds become active during the
anytime algorithm. \textsc{RestrictRange}, used by
Algorithm~\ref{alg:initialize-continuous-feature-queue}, converts these bounds
to the corresponding active-position interval \([L,U]\) for the current call.
Within \textsc{EnumContFeature}, \(L\) and \(U\) are tightened immediately when
new constant regions are discovered (we note that they can also be tightened for other reasons). Intervals already present in \(Q\) need
not be removed eagerly: when they are later popped,
Algorithm~\ref{alg:enumerate-continuous-feature} intersects them with the
current \([L,U]\) and discards them if the intersection is empty.

\paragraph{Cached proxy completions.}
The structure \(E\) is a map from exhaustive threshold indices to pairs of
cached child proxy objectives \((P_L,P_R)\). We use exhaustive indices as keys,
rather than active positions, because active positions may change when the
anytime algorithm activates new thresholds. \textsc{ProcessInterval} receives
the midpoint \(m\) as an active position and first maps it to its exhaustive
threshold index \(t\). If both children are nonempty, it computes and stores
\(E(t)=(P_L,P_R)\). Thus, later calls under a different budget or a refined
active threshold set can reuse the same proxy completions.

\makeatletter
\newcommand{\savealgline}[1]{\xdef#1{\arabic{ALC@line}}}
\newcommand{\restorealgline}[1]{\setcounter{ALC@line}{#1}}
\makeatother

\begin{algorithm}[!t]
\caption{\textsc{ProcessInterval}%
\((G,D,c,d,\gamma,\varepsilon_{\textrm{abs}},
\mathcal{S},E,Q,L,U,i,j,m)\)}
\label{alg:process-continuous-interval}
\begin{algorithmic}[1]
\INPUT Current OR node \(G\), subproblem \(D\), continuous feature \(c\),
remaining depth \(d\), per-leaf penalty \(\gamma\), budget
\(\varepsilon_{\textrm{abs}}\), threshold registry \(\mathcal{S}\), proxy
completion map \(E\), interval queue \(Q\), current bounds \([L,U]\), popped
interval \([i,j]\), and midpoint active position \(m\)

\STATE \(t\gets\textsc{ExhaustiveIndex}(\mathcal{S},c,m)\)
\COMMENT{Map the active position to its threshold-column index}
\STATE \((D_L,D_R)\gets\textsc{Partition}(D,t)\)

\IF{\(D_L=\emptyset\)}
    \STATE \(\textsc{RaiseLowerBound}(\mathcal{S},c,t+1)\)
    \COMMENT{Thresholds at or below \(t\) are constant on \(D\)}
    \STATE \(L\gets\max\{L,m+1\}\)
    \IF{\(m+1\leq j\)}
        \STATE \(Q.\textsc{Push}([m+1,j])\)
    \ENDIF
    \STATE \textbf{return} \((L,U)\)
\ENDIF

\IF{\(D_R=\emptyset\)}
    \STATE \(\textsc{LowerUpperBound}(\mathcal{S},c,t-1)\)
    \COMMENT{Thresholds at or above \(t\) are constant on \(D\)}
    \STATE \(U\gets\min\{U,m-1\}\)
    \IF{\(i\leq m-1\)}
        \STATE \(Q.\textsc{Push}([i,m-1])\)
    \ENDIF
    \STATE \textbf{return} \((L,U)\)
\ENDIF

\STATE \(\mathcal{S}_L\gets\mathcal{S}\); \quad
\(\mathcal{S}_R\gets\mathcal{S}\)
\STATE \(\textsc{LowerUpperBound}(\mathcal{S}_L,c,t-1)\)
\STATE \(\textsc{RaiseLowerBound}(\mathcal{S}_R,c,t+1)\)
\STATE \(P_L\gets\textsc{Proxy}(D_L,d-1,\gamma,\mathcal{S}_L)\)
\COMMENT{Using \(\mathcal{S}_L\) only exploits known constant thresholds}
\STATE \(P_R\gets\textsc{Proxy}(D_R,d-1,\gamma,\mathcal{S}_R)\)
\STATE \(E(t)\gets(P_L,P_R)\)
\COMMENT{Update the map by reference}
\STATE \(P\gets P_L+P_R\)

\makeatletter
\savealgline{\ProcessIntervalSavedLine}
\makeatother
\end{algorithmic}
\end{algorithm}

\begin{algorithm}[!t]
\ContinuedFloat
\caption{\textsc{ProcessInterval} (continued)}
\begin{algorithmic}[1]
\makeatletter
\restorealgline{\ProcessIntervalSavedLine}
\makeatother

\IF{\(P\leq\varepsilon_{\textrm{abs}}\)}
    \STATE \(\begin{aligned}
    &\textsc{AddOrExtendSplit}
    (G,D,t,d,\gamma,\varepsilon_{\textrm{abs}},\\
    &\qquad\mathcal{S},P_L,P_R)
    \end{aligned}\)
    \IF{\(i\leq m-1\)}
        \STATE \(Q.\textsc{Push}([i,m-1])\)
    \ENDIF
    \IF{\(m+1\leq j\)}
        \STATE \(Q.\textsc{Push}([m+1,j])\)
    \ENDIF
    \STATE \textbf{return} \((L,U)\)
\ENDIF

\STATE \COMMENT{The proxy completion is outside the budget; tighten bounds when one child is already a pure leaf}
\IF{\(P_L=\gamma\)}
    \STATE \(L\gets\max\{L,m+1\}\)
\ENDIF
\IF{\(P_R=\gamma\)}
    \STATE \(U\gets\min\{U,m-1\}\)
\ENDIF

\STATE \(\Delta\gets P-\varepsilon_{\textrm{abs}}\)
\STATE \(\ell\gets
\textsc{LeftBoundary}(D,\mathcal{S},c,t,m-1,
\max\{i,L\},\Delta)\)
\STATE \(r\gets
\textsc{RightBoundary}(D,\mathcal{S},c,t,m+1,
\min\{j,U\},\Delta)\)
\COMMENT{Positions in \([\ell+1,r-1]\) are pruned}

\IF{\(i\leq\min\{\ell,U\}\)}
    \STATE \(Q.\textsc{Push}([i,\min\{\ell,U\}])\)
\ENDIF
\IF{\(\max\{r,L\}\leq j\)}
    \STATE \(Q.\textsc{Push}([\max\{r,L\},j])\)
\ENDIF
\STATE \textbf{return} \((L,U)\)
\end{algorithmic}
\end{algorithm}

If the midpoint split is feasible (within budget), both neighboring subintervals remain
potentially feasible and are placed back into \(Q\). If it is infeasible,
\textsc{LeftBoundary} and \textsc{RightBoundary} identify the nearest positions
on each side whose active-sample distance from \(t\) is at least
\(\Delta=P-\varepsilon_{\textrm{abs}}\); the positions strictly between those
boundaries cannot improve enough to enter the budget. Their implementations are
given in Appendix~\ref{app:additional-methods}.

\FloatBarrier

\subsection{Iterative Budget Refinement}
\label{app:iterative-budget-refinement}

Algorithms~\ref{alg:graph-solve-ordered} and~\ref{alg:graph-solve-siblings} describe how we solve the two child subgraphs of a feasible split under a shared parent budget. The key issue is that the budget available to one child depends on the best objective attainable by the other child. That is, we want all solutions to the left subproblem that can be paired with the best solution in the right subproblem and remain within the parent budget (and vice versa). If we only know that the right child can be completed with objective $P_2$, then we can only prove that the left child needs to be solved up to budget $\varepsilon_{\mathrm{abs}}-P_2$. Once the left child is solved, its true minimum objective may be smaller than its initial proxy cost, which increases the budget available to the right child. This can in turn increase the budget available to the left child, so the process alternates until neither child receives a larger budget (this occurs when a side does not improve on its best solution). \textsc{GraphSolveOrdered} is very similar to iterative budget refinement in \citet{heile2026}.

\begin{algorithm}[!h]
\caption{\small$
\begin{aligned}
\textsc{GraphSolveOrdered}(&G_1,G_2,D_1,D_2,d,\gamma,
\varepsilon_{\textrm{abs}},\\
&P_1,P_2,\mathcal{S}_1,\mathcal{S}_2)
\end{aligned}
$}
\label{alg:graph-solve-ordered}
\begin{algorithmic}[1]
\REQUIRE Left/right subgraphs after relabeling, left/right datasets after relabeling,
remaining depth $d$, regularization $\gamma$, parent budget $\varepsilon_{\textrm{abs}}$,
initial proxy/minimum costs $P_1,P_2$, and threshold registries
$\mathcal{S}_1,\mathcal{S}_2$, $\mathrm{Budget}(\emptyset)=0$

\STATE $\varepsilon_1^{(\textrm{new})} \gets \varepsilon_{\textrm{abs}} - P_2$

\WHILE{$\varepsilon_1^{(\textrm{new})} > \textsc{Budget}(G_1)$}
  \STATE $\textsc{ArborEnum}
  (G_1,D_1,d,\gamma,\varepsilon_1^{(\textrm{new})},\mathcal{S}_1)$
  \COMMENT{Solve or extend the first subgraph}

  \STATE $\varepsilon_2^{(\textrm{new})} \gets
  \varepsilon_{\textrm{abs}} - G_1.\textit{min\_objective}$

  \IF{$\varepsilon_2^{(\textrm{new})} > \textsc{Budget}(G_2)$}
    \STATE $\textsc{ArborEnum}
    (G_2,D_2,d,\gamma,\varepsilon_2^{(\textrm{new})},\mathcal{S}_2)$
    \COMMENT{Solve or extend the second subgraph}

    \STATE $\varepsilon_1^{(\textrm{new})} \gets
    \varepsilon_{\textrm{abs}} - G_2.\textit{min\_objective}$
  \ENDIF
\ENDWHILE
\end{algorithmic}
\end{algorithm}

\FloatBarrier

The continuous-feature pruning rules and the anytime algorithm require a wrapper around the procedure in \citet{heile2026}. In particular, because the anytime algorithm revisits subproblems after activating additional thresholds, one child subgraph may improve while the other remains unchanged. Therefore, the order in which the two children are extended can matter.

\textsc{GraphSolveSiblings} (Algorithm~\ref{alg:graph-solve-siblings}) wraps the left and right children of an actual split. If both child subgraphs are empty, it simply calls \textsc{GraphSolveOrdered}, and both sides will be solve at least once. If the split has already been encountered, the routine uses the existing child budgets, stored minimum objectives, and new proxy costs to determine whether either child can be extended under the parent budget. When extension is needed, it starts with the child whose newly available budget has increased.

This ordering is important during anytime refinement. After new thresholds are activated, \textsc{RefineGraph} may improve the minimum objective of one child but not the other. If one child's minimum objective improves, then the budget available to its sibling increases. We therefore first solve the sibling whose budget increased. This gives that sibling an opportunity to improve its own minimum objective, which may in turn increase the budget available to the first child. The ordered solve then alternates until neither child can be extended further.

\begin{algorithm}[!h]
\caption{\small$
\begin{aligned}
\textsc{GraphSolveSiblings}(&G_L,G_R,D_L,D_R,d,\gamma,\varepsilon_{\textrm{abs}},\\
&P_L,P_R,\mathcal{S}_L,\mathcal{S}_R)
\end{aligned}
$}
\label{alg:graph-solve-siblings}
\begin{algorithmic}[1]
\REQUIRE Existing left/right subgraphs $G_L,G_R$ which are either both empty or both non-empty,
left/right datasets $D_L,D_R$, remaining depth $d$, regularization $\gamma$, parent budget
$\varepsilon_{\textrm{abs}}$, proxy costs $P_L,P_R$, and threshold registries
$\mathcal{S}_L,\mathcal{S}_R$

\IF{$G_L=\emptyset$ and $G_R=\emptyset$}
  \STATE $\begin{aligned}
  \textsc{GraphSolveOrdered}(&G_L,G_R,D_L,D_R,d,\gamma,
  \varepsilon_{\textrm{abs}},\\
  &P_L,P_R,\mathcal{S}_L,\mathcal{S}_R)
  \end{aligned}$

\ELSE
 \STATE $\varepsilon_L \gets \textsc{Budget}(G_L)$
  \COMMENT{$G_L$ has already been solved up to this budget}
  \STATE $\varepsilon_R \gets \textsc{Budget}(G_R)$
  \COMMENT{$G_R$ has already been solved up to this budget}

  \STATE $\widehat{P}_L \gets
  \min\{P_L,G_L.\textit{min\_objective}\}$
  \STATE $\widehat{P}_R \gets
  \min\{P_R,G_R.\textit{min\_objective}\}$

  \STATE $\varepsilon_L^{(\textrm{new})} \gets
  \varepsilon_{\textrm{abs}} - \widehat{P}_R$
  \STATE $\varepsilon_R^{(\textrm{new})} \gets
  \varepsilon_{\textrm{abs}} - \widehat{P}_L$

  \IF{$\varepsilon_L^{(\textrm{new})} \le \varepsilon_L$ and
      $\varepsilon_R^{(\textrm{new})} \le \varepsilon_R$}
    \STATE \textbf{return}
  \ENDIF

  \IF{$\varepsilon_L^{(\textrm{new})} > \varepsilon_L$}
    \STATE $\begin{aligned}
    \textsc{GraphSolveOrdered}(&G_L,G_R,D_L,D_R,d,\gamma,\\
    &\varepsilon_{\textrm{abs}},\widehat{P}_L,\widehat{P}_R,
    \mathcal{S}_L,\mathcal{S}_R)
    \end{aligned}$
  \ELSE
    \STATE $\begin{aligned}
    \textsc{GraphSolveOrdered}(&G_R,G_L,D_R,D_L,d,\gamma,\\
    &\varepsilon_{\textrm{abs}},\widehat{P}_R,\widehat{P}_L,
    \mathcal{S}_R,\mathcal{S}_L)
    \end{aligned}$
  \ENDIF
\ENDIF

\end{algorithmic}
\end{algorithm}

\FloatBarrier

\subsection{RefineGraph}
\label{app:refine-graph}

\textsc{RefineGraph} (Algorithm~\ref{alg:refine-graph-dfs}) updates an existing minimum-objective AND/OR graph after the active threshold set has been enlarged by the anytime algorithm. The input graph already encodes feasible trees for an earlier, smaller set of active thresholds. The goal of \textsc{RefineGraph} is not to rebuild this graph from scratch, but to recursively extend it so that it is valid for the new active threshold set.

The routine proceeds in post-order. For each split already stored at \(G\), it
copies the current threshold registry into child-specific registries
\(\mathcal{S}_L\) and \(\mathcal{S}_R\). If the split is a threshold \(t\) of
continuous feature \(c\), the left child satisfies \(x_c\leq\nu_t\), so
\textsc{LowerUpperBound} restricts the upper bound of \(\mathcal{S}_L\) to
\(t-1\). Similarly, the right child satisfies \(x_c>\nu_t\), so
\textsc{RaiseLowerBound} restricts the lower bound of \(\mathcal{S}_R\) to
\(t+1\). The routine then recursively refines both children using their
respective registries.

After the children have been refined, an existing binary split is passed to
\textsc{AddOrExtendSplit}, since improved child minimum objectives may permit
additional iterative budget refinement. Existing continuous splits need not be
handled separately here: the subsequent call to \textsc{EnumContFeature}
processes all active thresholds, including thresholds already stored in \(G\),
using their cached proxy completions. The threshold bounds are already stored
in \(\mathcal{S}\) and will be used later.

\begin{algorithm}[h]
\caption{\textsc{RefineGraph}%
\((G,D,d,\mathcal{S},\mathcal{V})\)}
\label{alg:refine-graph-dfs}
\begin{algorithmic}[1]
\INPUT Existing OR node \(G\), subproblem bitvector \(D\), remaining depth
\(d\), threshold registry \(\mathcal{S}\), and set \(\mathcal{V}\) of OR nodes
already refined during this pass

\IF{\(G=\emptyset\)}
    \STATE \textbf{return}
\ENDIF

\STATE \(\varepsilon_{\textrm{abs}}\gets\textsc{Budget}(G)\)

\IF{\(G\in\mathcal{V}\)}
    \STATE \textbf{return}
    \COMMENT{\(G\) was already refined during this pass}
\ENDIF

\STATE \(\mathcal{V}\gets\mathcal{V}\cup\{G\}\)

\IF{\(d\leq 0\) \textbf{ or }
\(\varepsilon_{\textrm{abs}}<2\gamma\)}
    \STATE \textbf{return}
    \COMMENT{No split can fit within the remaining budget}
\ENDIF

\STATE \(\mathcal{A}\gets\textsc{Splits}(G)\)
\COMMENT{Snapshot the splits present when this call begins}

\FOR{\textbf{each split node} \(s\in\mathcal{A}\)}
    \STATE \(t\gets\textsc{Feature}(s)\)
    \STATE \((D_L,D_R)\gets\textsc{Partition}(D,t)\)
    \STATE \(\mathcal{S}_L\gets\mathcal{S}\); \quad
    \(\mathcal{S}_R\gets\mathcal{S}\)

    \IF{\(t\) is a threshold of continuous feature \(c\)}
        \STATE
        \(\textsc{LowerUpperBound}(\mathcal{S}_L,c,t-1)\)
        \COMMENT{Thresholds at or above \(t\) are constant on \(D_L\)}
        \STATE
        \(\textsc{RaiseLowerBound}(\mathcal{S}_R,c,t+1)\)
        \COMMENT{Thresholds at or below \(t\) are constant on \(D_R\)}
    \ENDIF

    \STATE \(\textsc{RefineGraph}
    (s.\textit{left},D_L,d-1,\mathcal{S}_L,\mathcal{V})\)

    \STATE \(\textsc{RefineGraph}
    (s.\textit{right},D_R,d-1,\mathcal{S}_R,\mathcal{V})\)

    \IF{\(t\) is an ordinary binary feature}
        \STATE
        \(P_L\gets\textsc{Proxy}(D_L,d-1,\gamma,\mathcal{S}_L)\)
        \STATE
        \(P_R\gets\textsc{Proxy}(D_R,d-1,\gamma,\mathcal{S}_R)\)

        \STATE \(\begin{aligned}
        &\textsc{AddOrExtendSplit}
        (G,D,t,d,\gamma,\varepsilon_{\textrm{abs}},\\
        &\qquad\mathcal{S},P_L,P_R)
        \end{aligned}\)
    \ENDIF
\ENDFOR

\FOR{\textbf{each continuous feature} \(c\)}
    \STATE \(\begin{aligned}
    &\textsc{EnumContFeature}
    (G,D,c,d,\gamma,\varepsilon_{\textrm{abs}},\mathcal{S})
    \end{aligned}\)
    \COMMENT{Add newly active feasible thresholds and extend existing ones}
\ENDFOR

\end{algorithmic}
\end{algorithm}

\FloatBarrier

\subsection{AddOrExtendSplit}
\label{app:add-or-extend-split}

\textsc{AddOrExtendSplit} (Algorithm~\ref{alg:resolve-threshold}) is the routine that materializes a creates or extends a given split at an OR node. It partitions the current subproblem,
constructs the threshold registries for the two children, and retrieves the
existing child subgraphs when the split is already stored in \(G\).

For a continuous threshold \(t\) of feature \(c\), the left child satisfies
\(x_c\leq \nu_t\), so every threshold of feature \(c\) with exhaustive index at
least \(t\) is constant on \(D_L\). We therefore tighten the upper bound in
\(\mathcal{S}_L\) to \(t-1\). Symmetrically, the right child satisfies
\(x_c>\nu_t\), so every threshold with index at most \(t\) is constant on
\(D_R\), and we tighten the lower bound in \(\mathcal{S}_R\) to \(t+1\).
These are the child-specific bounds later used by
\textsc{RestrictRange}. For an ordinary binary feature, both child registries are unchanged
copies of \(\mathcal{S}\).

The two children are then passed to \textsc{GraphSolveSiblings}, which solves or extends them under the shared parent budget. If both child subgraphs are nonempty after this refinement, the split is valid in the Rashomon graph. If the split is new, it is added to $G$ with its left and right child subgraphs. If the split already exists, the routine updates the minimum objective stored at $G$ using the improved child objectives. \textsc{AddSPLIT} \citep{heile2026} already handles this minimum objective propagation, so we only explicitly handle it in the case that we are extending a split.

\begin{algorithm}[!h]
\caption{\textsc{AddOrExtendSplit}%
\((G,D,t,d,\gamma,\varepsilon_{\textrm{abs}},\mathcal{S},P_L,P_R)\)}
\label{alg:resolve-threshold}
\begin{algorithmic}[1]
\INPUT Current OR node \(G\), subproblem \(D\), split feature or threshold
\(t\), remaining depth \(d\), per-leaf penalty \(\gamma\), parent budget
\(\varepsilon_{\textrm{abs}}\), threshold registry \(\mathcal{S}\), and child
proxy objectives \(P_L,P_R\)

\STATE \((D_L,D_R)\gets\textsc{Partition}(D,t)\)
\STATE \(\mathcal{S}_L\gets\mathcal{S}\); \quad
\(\mathcal{S}_R\gets\mathcal{S}\)

\IF{\(t\) is a threshold of continuous feature \(c\)}
    \STATE \(\textsc{LowerUpperBound}(\mathcal{S}_L,c,t-1)\)
    \COMMENT{Thresholds at or above \(t\) are constant on \(D_L\)}
    \STATE \(\textsc{RaiseLowerBound}(\mathcal{S}_R,c,t+1)\)
    \COMMENT{Thresholds at or below \(t\) are constant on \(D_R\)}
\ENDIF

\IF{\(G\) already contains split \(t\)}
    \STATE Let \((G_L,G_R)\) be the existing children of split \(t\)
    \STATE \(\textit{is\_new}\gets\textbf{false}\)
\ELSE
    \STATE \(G_L\gets\emptyset\); \quad \(G_R\gets\emptyset\)
    \STATE \(\textit{is\_new}\gets\textbf{true}\)
\ENDIF

\STATE \(\begin{aligned}
\textsc{GraphSolveSiblings}(&G_L,G_R,D_L,D_R,d-1,\gamma,
\varepsilon_{\textrm{abs}},\\
&P_L,P_R,\mathcal{S}_L,\mathcal{S}_R)
\end{aligned}\)

\IF{\(G_L=\emptyset\) \textbf{ or } \(G_R=\emptyset\)}
    \STATE \textbf{return}
    \COMMENT{The split has no feasible pair of child subtrees}
\ENDIF

\IF{\(\textit{is\_new}\)}
    \STATE \(\textsc{AddSplit}(G,t,G_L,G_R)\)
\ELSE
    \STATE \(G.\textit{min\_objective}\gets
    \min\{G.\textit{min\_objective},
    G_L.\textit{min\_objective}+G_R.\textit{min\_objective}\}\)
\ENDIF
\end{algorithmic}
\end{algorithm}

\FloatBarrier

\subsection{Anytime Algorithm}
\label{app:anytime}

\textsc{AnytimeArborEnum}
(Algorithm~\ref{alg:anytime-continuous-rset}) progressively constructs a
Rashomon graph over increasingly large sets of active thresholds. The goal of this section is to provide an overview of the method and additional details including refining the proxy. The algorithm begins
with the ordinary binary features and an initial proxy threshold set
\(\mathcal{B}_{\textrm{proxy}}\). These thresholds are used to initialize the
root threshold registry \(\mathcal{S}_{\textrm{root}}\), which stores both the
currently active thresholds and the continuous-feature bounds. The proxy is
restricted to \(\mathcal{B}_{\textrm{proxy}}\), and the first call to
\textsc{ArborEnum} constructs a minimum-objective AND/OR graph over the
initial active threshold set.

The algorithm then repeatedly activates additional continuous thresholds.
After each refinement,
\(\textsc{ActivateThresholds}(\mathcal{S}_{\textrm{root}},
\mathcal{B}_{\textrm{new}})\) updates the threshold registry by reference (to add the new thresholds), and
\textsc{RefineGraph} revisits the existing graph. This recursively refines
child subgraphs, reruns iterative budget refinement where necessary, and searches
for continuous splits that were unavailable in earlier rounds but are now
active . We clear the canonical subgraph index \(\mathcal{I}_G\) between
threshold-refinement rounds because a cached subgraph that was complete for the
previous active threshold set is not necessarily complete for the enlarged
set. The graph itself is preserved and extended in place. The threshold-to-proxy
completion maps and the proxy algorithm's own cache are also retained, so
previously computed proxy completions remain available. As subgraphs are
revisited, \(\mathcal{I}_G\) is repopulated and again provides reuse within the
new active threshold set.

Once all thresholds are active, the algorithm may optionally strengthen the
proxy until it becomes optimal. This phase is optional because increasing proxy
strength may be substantially more expensive than threshold refinement. After
each increase in proxy strength, proxy caches are updated as described in
Appendix~\ref{app:licketysnip}, the threshold-to-proxy completion maps whose
values depend on the previous proxy are cleared or recomputed, and
\(\mathcal{I}_G\) is cleared because cached completeness guarantees were
established using the previous proxy. The existing Rashomon graph is not
discarded. Instead, \textsc{RefineGraph} revisits and extends it using the
stronger proxy. Thus, early stopping returns the best graph constructed so far,
whereas running both phases to completion recovers the full continuous-feature
Rashomon set when the proxy becomes optimal.

Algorithm~\ref{alg:anytime-continuous-rset} presents the simpler variant in
which the proxy is restricted to \(\mathcal{B}_{\textrm{proxy}}\), while the
active enumeration thresholds stored in \(\mathcal{S}_{\textrm{root}}\) form a
superset of \(\mathcal{B}_{\textrm{proxy}}\). We can relax this restriction by
allowing the proxy to use a larger threshold set and return both the objective
of its certified tree and that tree's root split for the current subproblem and
depth. We then activate this root split locally at the corresponding OR node
and evaluate it immediately. This adds at most one proxy-selected threshold per
subproblem and ensures that the proxy-certified tree remains recoverable even
when its root split is not yet globally active. Consequently, the anytime
algorithm may begin with no active continuous thresholds, add proxy-selected
thresholds locally as needed, and still progressively activate the remaining
thresholds globally until all thresholds are available.

There is one important thing to note that may not be immediately clear. During refinement, a binary feature may not have been used previously. But, now with an expanded active set, it may need to be used. This can only happen in our framework if a minimum objective improves in the sibling subproblem. Because the proxy is fixed at some granularity, whether fully continuous features or some binarization, it is not able to give a different value. If one child subgraph improves its minimum objective, the budget available to the sibling increases. \textsc{RefineGraph} therefore calls \textsc{AddOrExtendSplit} on each existing binary split after recursively refining its children. This invokes iterative budget refinement, which may extend either child under the newly available budget. Because extending a child calls \textsc{ArborEnum}, all of that child’s candidate splits are considered again, including binary splits that were previously outside the child’s budget but are now feasible. \textsc{ArborEnum} will extend the subgraph to the bigger budget while reconsidering all binary features that are not yet in the graph, reconsidering all old active continuous thresholds that are not yet in the graph, and considering all continuous thresholds that just became active.

\begin{algorithm}[!h]
\caption{\textsc{AnytimeArborEnum}%
\((d,\gamma,\varepsilon_{\textrm{mult}},
\mathcal{B}_{\textrm{proxy}})\)}
\label{alg:anytime-continuous-rset-full}
\begin{algorithmic}[1]
\INPUT Depth budget \(d\), per-leaf penalty \(\gamma\), Rashomon multiplier
\(\varepsilon_{\textrm{mult}}\), and sorted list of proxy thresholds
\(\mathcal{B}_{\textrm{proxy}}\)

\STATE \(\mathcal{B}_{\textrm{bin}}\gets\)
sorted list of indices of ordinary binary features

\STATE \(\mathcal{B}_{\textrm{initial}}\gets
\textsc{SortUnique}
(\mathcal{B}_{\textrm{bin}}\cup\mathcal{B}_{\textrm{proxy}})\)

\STATE \(\mathcal{S}_{\textrm{root}}\gets
\textsc{InitializeThresholdRegistry}
(\mathcal{B}_{\textrm{initial}})\)
\COMMENT{Initialize active thresholds and unrestricted feature bounds}

\STATE \(D_{\textrm{root}}\gets\)
the bitvector containing all training samples

\STATE Restrict proxy algorithms to
\(\mathcal{B}_{\textrm{proxy}}\)

\STATE \(\varepsilon_{\textrm{abs}}\gets
(1+\varepsilon_{\textrm{mult}})
\textsc{Proxy}
(D_{\textrm{root}},d,\gamma,\mathcal{S}_{\textrm{root}})\)
\COMMENT{Initialize the root budget}

\STATE \(\begin{aligned}
G\gets{}&
\textsc{ArborEnum}
(G,D_{\textrm{root}},d,\gamma,\varepsilon_{\textrm{abs}},\\
&\qquad\mathcal{S}_{\textrm{root}})
\end{aligned}\)

\WHILE{\textbf{not }
\(\textsc{AllThresholdsActive}(\mathcal{S}_{\textrm{root}})\)}

    \STATE \(\mathcal{B}_{\textrm{new}}\gets
    \textsc{SelectNewThresholds}(\mathcal{S}_{\textrm{root}})\)
    \COMMENT{By default, select one threshold in the gap between two active indices (for each pair of consecutive active indices)}

    \STATE \(\textsc{ActivateThresholds}
    (\mathcal{S}_{\textrm{root}},\mathcal{B}_{\textrm{new}})\)
    \COMMENT{Update the root registry by reference}

    \STATE Clear \(\mathcal{I}_G\)
    \COMMENT{Cached subgraphs may be incomplete for the enlarged active set}

    \STATE \(\mathcal{V}\gets\emptyset\)

    \STATE \(\textsc{RefineGraph}
    (G,D_{\textrm{root}},d,
    \mathcal{S}_{\textrm{root}},\mathcal{V})\)
    \COMMENT{Extend the existing graph using the newly active thresholds}
\ENDWHILE

\WHILE{\textbf{not }
\(\textsc{IsProxyOptimal}(\textsc{Proxy})\)}

    \STATE Increase proxy strength by \(1\)

    \STATE Clear each threshold-to-proxy completion map $\mathcal{C}_\textrm{map}$ 
    \COMMENT{Their stored objectives were computed using the previous proxy}

    \STATE Clear \(\mathcal{I}_G\)
    \COMMENT{Cached completeness guarantees used the previous proxy}

    \STATE \(\mathcal{V}\gets\emptyset\)

    \STATE \(\textsc{RefineGraph}
    (G,D_{\textrm{root}},d,
    \mathcal{S}_{\textrm{root}},\mathcal{V})\)
    \COMMENT{Revisit the existing graph using the stronger proxy}
\ENDWHILE

\STATE \textbf{return} \(G\)
\end{algorithmic}
\end{algorithm}

\FloatBarrier

\subsection{Additional Methods}
\label{app:additional-methods}

The boundary routines are implemented by binary search (with a small amount of initial probing) and use the active thresholds stored in \(\mathcal{S}\).
After a failed threshold, they find the closest active position that can
possibly recover from the gap \(\Delta\), using
\(O(\log |\mathcal{A}_c|)\) calls to \textsc{FarEnough}, where
\(\mathcal{A}_c\) is the ordered list of active exhaustive threshold indices for
feature \(c\). Depending on the active-threshold spacing, however, the nearest
active threshold may already be far enough away. Therefore, each routine first
tests the nearest two active positions in its search direction and returns
immediately if either has active-sample distance at least \(\Delta\); otherwise,
it applies binary search to the remaining interval. Each call to
\textsc{FarEnough} stops as soon as \(\Delta\) differing active samples have
been found.

We store each subproblem \(D\) and exhaustive threshold column \(X_t\) as packed
bitvectors. The \(i\)th bit indicates whether sample \(i\) is active in \(D\)
or satisfies threshold \(t\), respectively. These bitvectors are stored as arrays of 64-bit machine words, allowing each bitwise operation to process 64 samples at once; \textsc{Popcount} returns the number of set bits in each word.

\begin{algorithm}[!h]
\caption{\textsc{RightBoundary}%
\((D,\mathcal{S},c,t,i,b,\Delta)\)}
\label{alg:right-boundary}
\begin{algorithmic}[1]
\INPUT Subproblem bitvector \(D\), threshold registry \(\mathcal{S}\),
continuous feature \(c\), failed exhaustive threshold index \(t\), search
interval \([i,b]\) over active positions, and gap \(\Delta\)

\IF{\(i>b\)}
    \STATE \textbf{return} \(b+1\)
\ENDIF

\FOR{\textbf{each} \(p\in\{i,\min\{i+1,b\}\}\)}
    \STATE \(s\gets\textsc{ExhaustiveIndex}(\mathcal{S},c,p)\)
    \IF{\(\textsc{FarEnough}(D,t,s,\Delta)\)}
        \STATE \textbf{return} \(p\)
    \ENDIF
\ENDFOR

\STATE \(\ell\gets\min\{i+2,b+1\}\); \quad
\(r\gets b\); \quad \(q\gets b+1\)

\WHILE{\(\ell\leq r\)}
    \STATE \(m\gets\lfloor(\ell+r)/2\rfloor\)
    \STATE \(s\gets\textsc{ExhaustiveIndex}(\mathcal{S},c,m)\)
    \IF{\(\textsc{FarEnough}(D,t,s,\Delta)\)}
        \STATE \(q\gets m\); \quad \(r\gets m-1\)
    \ELSE
        \STATE \(\ell\gets m+1\)
    \ENDIF
\ENDWHILE

\STATE \textbf{return} \(q\)
\end{algorithmic}
\end{algorithm}

\begin{algorithm}[!h]
\caption{\textsc{LeftBoundary}%
\((D,\mathcal{S},c,t,j,a,\Delta)\)}
\label{alg:left-boundary}
\begin{algorithmic}[1]
\INPUT Subproblem bitvector \(D\), threshold registry \(\mathcal{S}\),
continuous feature \(c\), failed exhaustive threshold index \(t\), search
interval \([a,j]\) over active positions, and gap \(\Delta\)

\IF{\(a>j\)}
    \STATE \textbf{return} \(a-1\)
\ENDIF

\FOR{\textbf{each} \(p\in\{j,\max\{j-1,a\}\}\)}
    \STATE \(s\gets\textsc{ExhaustiveIndex}(\mathcal{S},c,p)\)
    \IF{\(\textsc{FarEnough}(D,s,t,\Delta)\)}
        \STATE \textbf{return} \(p\)
    \ENDIF
\ENDFOR

\STATE \(\ell\gets a\); \quad
\(r\gets\max\{j-2,a-1\}\); \quad \(q\gets a-1\)

\WHILE{\(\ell\leq r\)}
    \STATE \(m\gets\lfloor(\ell+r)/2\rfloor\)
    \STATE \(s\gets\textsc{ExhaustiveIndex}(\mathcal{S},c,m)\)
    \IF{\(\textsc{FarEnough}(D,s,t,\Delta)\)}
        \STATE \(q\gets m\); \quad \(\ell\gets m+1\)
    \ELSE
        \STATE \(r\gets m-1\)
    \ENDIF
\ENDWHILE

\STATE \textbf{return} \(q\)
\end{algorithmic}
\end{algorithm}

\begin{algorithm}[!h]
\caption{\textsc{FarEnough}%
\((D,s,t,\Delta)\)}
\label{alg:far-enough}
\begin{algorithmic}[1]
\INPUT Subproblem bitvector \(D\), exhaustive threshold indices \(s,t\), and
gap \(\Delta\)

\IF{\(\Delta\leq 0\)}
    \STATE \textbf{return} \(\textbf{true}\)
\ENDIF

\STATE \(z\gets 0\)

\FOR{\textbf{each machine word} \(w\)}
    \STATE \(z\gets z+
    \textsc{Popcount}\!\left(
    D_w\wedge(X_{s,w}\oplus X_{t,w})
    \right)\)
    \IF{\(z\geq\Delta\)}
        \STATE \textbf{return} \(\textbf{true}\)
    \ENDIF
\ENDFOR

\STATE \textbf{return} \(\textbf{false}\)
\end{algorithmic}
\end{algorithm}

We also provide pseudocode for calculating the number of samples that exhibit predictive multiplicity and the number of features used in the Rashomon graph. See Algorithms \ref{alg:count-predictive-multiplicity}, \ref{alg:reachable-predictions}, \ref{alg:count-graph-features}, and \ref{alg:collect-graph-features}.

\begin{algorithm}[!t]
\caption{\textsc{CountSamplesWithMultiplePredictions}$(G,X)$}
\label{alg:count-predictive-multiplicity}
\begin{algorithmic}[1]
\INPUT Root OR node $G$ of the Rashomon graph and training samples $X$
\STATE $m \gets 0$
\FOR{\textbf{each sample} $x_i \in X$}
    \STATE $\mathcal{Y}_i \gets \textsc{ReachablePredictions}(G,x_i)$
    \IF{$|\mathcal{Y}_i| \ge 2$}
        \STATE $m \gets m+1$
    \ENDIF
\ENDFOR
\STATE \textbf{return} $m$
\end{algorithmic}
\end{algorithm}

\begin{algorithm}[!t]
\caption{\textsc{ReachablePredictions}$(G,x)$}
\label{alg:reachable-predictions}
\begin{algorithmic}[1]
\INPUT OR node $G$ and sample $x$
\STATE $\mathcal{Y}_{\mathrm{reachable}} \gets \emptyset$
\FOR{\textbf{each leaf choice} $\ell$ stored in $G$}
    \STATE $\mathcal{Y}_{\mathrm{reachable}}
        \gets
        \mathcal{Y}_{\mathrm{reachable}}
        \cup
        \{\textsc{Prediction}(\ell)\}$
\ENDFOR
\FOR{\textbf{each split choice} $s$ stored in $G$}
    \IF{$x$ satisfies the split condition of $s$}
        \STATE $G' \gets \textsc{LeftChild}(s)$
    \ELSE
        \STATE $G' \gets \textsc{RightChild}(s)$
    \ENDIF
    \STATE $\mathcal{Y}_{\mathrm{reachable}}
        \gets
        \mathcal{Y}_{\mathrm{reachable}}
        \cup
        \textsc{ReachablePredictions}(G',x)$
    \IF{$|\mathcal{Y}_{\mathrm{reachable}}| \ge 2$}
        \STATE \textbf{return} $\mathcal{Y}_{\mathrm{reachable}}$
        \COMMENT{The sample already has conflicting predictions}
    \ENDIF
\ENDFOR
\STATE \textbf{return} $\mathcal{Y}_{\mathrm{reachable}}$
\end{algorithmic}
\end{algorithm}

\begin{algorithm}[!t]
\caption{\textsc{CountGraphFeatures}$(G)$}
\label{alg:count-graph-features}
\begin{algorithmic}[1]
\INPUT Root OR node $G$ of the Rashomon graph
\STATE $\mathcal{F} \gets \emptyset$
\STATE $\mathcal{V} \gets \emptyset$
\STATE $\textsc{CollectGraphFeatures}(G,\mathcal{F},\mathcal{V})$
\STATE \textbf{return} $|\mathcal{F}|$
\end{algorithmic}
\end{algorithm}

\begin{algorithm}[!t]
\caption{\textsc{CollectGraphFeatures}$(G,\mathcal{F},\mathcal{V})$}
\label{alg:collect-graph-features}
\begin{algorithmic}[1]
\INPUT OR node $G$, set of encountered features $\mathcal{F}$, and set of visited OR nodes $\mathcal{V}$
\IF{$G=\emptyset$ \textbf{ or } $G\in\mathcal{V}$}
    \STATE \textbf{return}
\ENDIF
\STATE $\mathcal{V} \gets \mathcal{V}\cup\{G\}$
\FOR{\textbf{each split choice} $s$ stored in $G$}
    \STATE $j \gets \textsc{OriginalFeature}(\textsc{SplitFeature}(s))$
    \COMMENT{Map all thresholds of one continuous feature to the same feature}
    \STATE $\mathcal{F} \gets \mathcal{F}\cup\{j\}$
    \STATE $\textsc{CollectGraphFeatures}(\textsc{LeftChild}(s),\mathcal{F},\mathcal{V})$
    \STATE $\textsc{CollectGraphFeatures}(\textsc{RightChild}(s),\mathcal{F},\mathcal{V})$
\ENDFOR
\end{algorithmic}
\end{algorithm}

We also provide a more efficient implementation of \textsc{GetKthTreeWithObjective} than was present in \citet{heile2026}. For full context, see Algorithm 12-14 in \citet{heile2026}. 

\begin{algorithm}[!t]
\caption{\textsc{GetKthTreeWithObjective}$(G,z,k)$}
\label{alg:get-kth-tree-with-objective}
\begin{algorithmic}[1]
\INPUT OR node $G$, target objective $z$, and zero-indexed position $k$ among trees rooted at $G$ with objective $z$
\STATE $\textsc{BuildHistogramsPost}(G)$
\COMMENT{Ensure all objective histograms are available}

\STATE \COMMENT{Enumerate feasible leaves before splits}
\FOR{\textbf{each} leaf $(b,\ell)\in\textsc{Leaves}(G)$}
    \IF{$\ell=z$}
        \IF{$k=0$}
            \STATE \textbf{return} $\textsc{MakeLeaf}(b)$
        \ENDIF
        \STATE $k\gets k-1$
    \ENDIF
\ENDFOR

\STATE \COMMENT{Enumerate split trees in stored order}
\FOR{\textbf{each} split $s\in\textsc{Splits}(G)$}
    \STATE $G_L\gets\textsc{Left}(s)$;\quad
           $G_R\gets\textsc{Right}(s)$
    \STATE $c\gets 0$
    \COMMENT{Number of objective-$z$ trees encountered under this split}

    \FOR{\textbf{each} $(z_L,n_L)\in\textsc{Histogram}(G_L)$}
        \STATE $z_R\gets z-z_L$
        \STATE $n_R\gets
        \textsc{HistogramCount}(G_R,z_R)$
        \COMMENT{Use binary search in the sorted histogram to find the number of trees with objective $z_R$}

        \IF{$n_R>0$}
            \STATE $n\gets n_Ln_R$
            \COMMENT{Number of trees in this Cartesian-product block}

            \IF{$k<c+n$}
                \STATE $q\gets k-c$
                \STATE $k_L\gets\left\lfloor q/n_R\right\rfloor$
                \STATE $k_R\gets q\bmod n_R$
                \STATE $T_L\gets
                \textsc{GetKthTreeWithObjective}(G_L,z_L,k_L)$
                \STATE $T_R\gets
                \textsc{GetKthTreeWithObjective}(G_R,z_R,k_R)$
                \STATE \textbf{return}
                $\textsc{MakeSplit}(\textsc{Feature}(s),T_L,T_R)$
            \ENDIF

            \STATE $c\gets c+n$
        \ENDIF
    \ENDFOR

    \STATE $k\gets k-c$
    \COMMENT{Skip all objective-$z$ trees under this split}
\ENDFOR

\end{algorithmic}
\end{algorithm}

\FloatBarrier

\subsection{AND/OR Graph Caching}
\label{app:andor-caching}

\begin{figure*}[!h]
    \centering
    \includegraphics[width=1.0\linewidth]{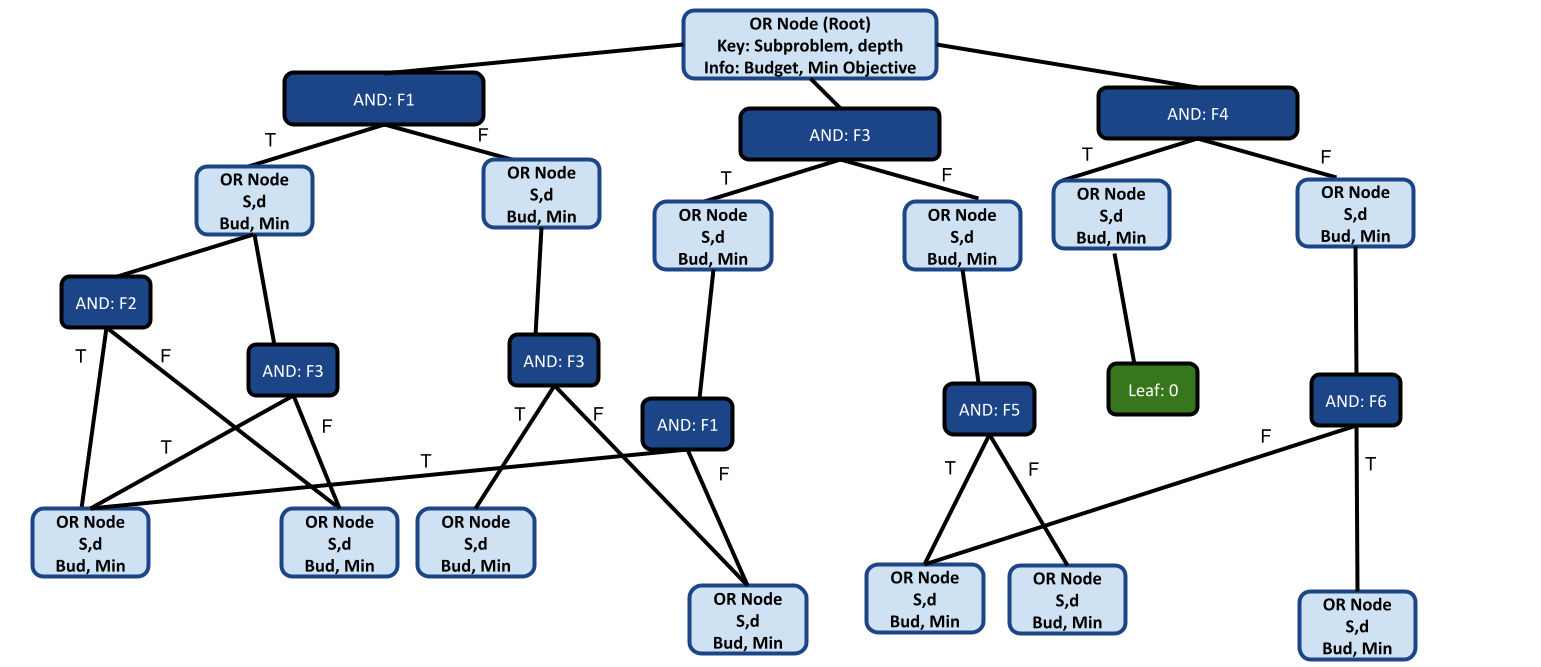}
    \caption{An example graph structure for encoding a Rashomon set. The OR nodes at the bottom of this figure have more split/leaf choices connected to them. We now state what must be true about these features (F1, F2, F3, F4, F5, and F6) for this graph to be built the way it is. It does not matter whether these are truly binary features or a threshold of a continuous feature. 
    \textbf{Assumption 1: } F2 and F3 are the same if F1 is True. Thus, the nodes on the far left connect to the same children (and they need the same budget) 
    \textbf{Assumption 2: }  !F3 \& F5 gives the same bitvector as !F4 \& !F6. They may need the same budget; they also may not need the same budget. The OR node encodes solutions for the larger of the two budgets, so that it works for both. \textbf{Remaining Remark: } The remaining shared children are due to conjunctions of literals being invariant to permutation.}
    \label{fig:tiedgraph}
\end{figure*}

Without caching Rashomon subgraphs, the enumeration
builds a trie over sequences of split decisions. That is, each root-to-node path
corresponds to an ordered sequence of splits encountered by the recursion. This
representation is order-dependent: two different split sequences may reach the
same subproblem, but they are still represented by different nodes because they
arise from different paths in the recursion.

A first improvement is to cache Rashomon subgraphs by subproblem, remaining
depth, and budget. This changes the data structure from an acyclic
AND/OR graph that is a trie into an acyclic
AND/OR graph that isn't a trie. The graph is still acyclic because every split decreases the
remaining depth and restricts the active sample set, but an OR node may now have
multiple parents. Conceptually, this is simply duplicate elimination: whenever
two recursive paths reach the same subproblem with the same remaining depth and
budget, we store one canonical OR node and point both parents to it. We note that this caching covers more than just permutations of split choices because we use cache based on what samples are in the subproblem. That set of samples can be reached in many different ways. 

We push this idea further by making the cache independent of the budget. Instead
of storing separate subgraphs for the same subproblem and remaining depth under
different budgets, we store a single canonical OR node for each pair
$(D,d)$. This node represents the largest budget under which that subproblem has
so far been expanded. If the same subproblem is later reached with a budget that
is no larger than the stored budget, we return the existing node. If it is
reached with a larger budget, we extend the existing node in place. Thus, at no point do
we maintain more than one subgraph for the same subproblem-depth pair (i.e., this is not a deduplication step at the end of the algorithm, but the algorithm never maintains more than one at a time).

This does not lose information because any smaller-budget Rashomon set can be
carved out of the larger-budget subgraph during extraction. In fact, in the approximation regime, this can lead to more trees recovered, because approximation algorithms return a subset of trees with budget $\varepsilon_{\textrm{abs}}$ at each subproblem. The larger graph may
contain split or leaf choices whose objectives exceed a smaller requested budget,
we claim that these choices are can be ignored when extracting trees under that smaller budget.

It remains to explain why this is compatible with the objective histograms used
for indexing trees. For each OR node, we store a histogram
\[
    H_G = \{(o,c)\},
\]
where $o$ is an achievable objective value for a subtree rooted at $G$, and $c$
is the number of such subtrees with objective $o$. Leaf choices contribute one
unit of mass to the bucket corresponding to their leaf objective. A split
contributes by combining the histograms of its two children. If the left and
right child histograms contain entries $(o_L,c_L)$ and $(o_R,c_R)$, respectively,
then the split contributes
\[
    (o_L + o_R,\; c_L c_R)
\]
to the parent histogram, provided that
\[
    o_L + o_R \leq \varepsilon_G,
\]
where $\varepsilon_G$ is the budget of the parent OR node. That is, the propagation step is a budget-filtered cross product, of the
two child histograms.

Because histograms are built after the graph has been expanded, there are no issues with them becoming stale. We also remark that while we are storing one OR node where other representations store multiple, this is not true at the root node. The root node has a single budget and its histogram  is filtered exactly to the Rashomon budget of interest. 

In summary, internal OR nodes may store more information than their parent needs. However, when that parent combines the child histograms, it applies its own budget filter to the child objective pairs. Therefore, even if a child
histogram contains entries enabled by a larger budget, only child subtrees that
fit within the parent's remaining budget contribute to the parent's histogram. This means that tasks like "finding the 50th tree in the Rashomon set", which map to finding the 7th tree with objective 112 are unchanged, if we get to a node that has more information than we need, that doesn't effect the indexing operations. 

Thus, indexing operations are unchanged. For example, finding the 50th tree in the Rashomon set may reduce to finding the 7th tree with objective 112 in a child subgraph; this is the type of recursive indexing used by \citet{heile2026} to enumerate trees in sorted order. The procedure identifies the split containing the desired tree, then recurses on the left and right child subgraphs, asking for the $L_1$th left subtree with objective $L_2$ and the $R_1$th right subtree with objective $R_2$.

The presence of additional trees in the cached child subgraph does not affect this query, because those extra trees have objectives above the budget relevant to the parent query. Indexing only scans histogram entries with the target objective or smaller, and these entries are guaranteed to be fully represented in the cached graph: the cached graph has been expanded to at least the budget needed for the current query. Therefore, the larger cached subgraph behaves exactly like the smaller-budget subgraph for sorted-order indexing.

\FloatBarrier

\section{Theoretical Results}
\label{app:theorems}

\begin{figure}[!b]
    \centering
    \includegraphics[width=1.0\linewidth]{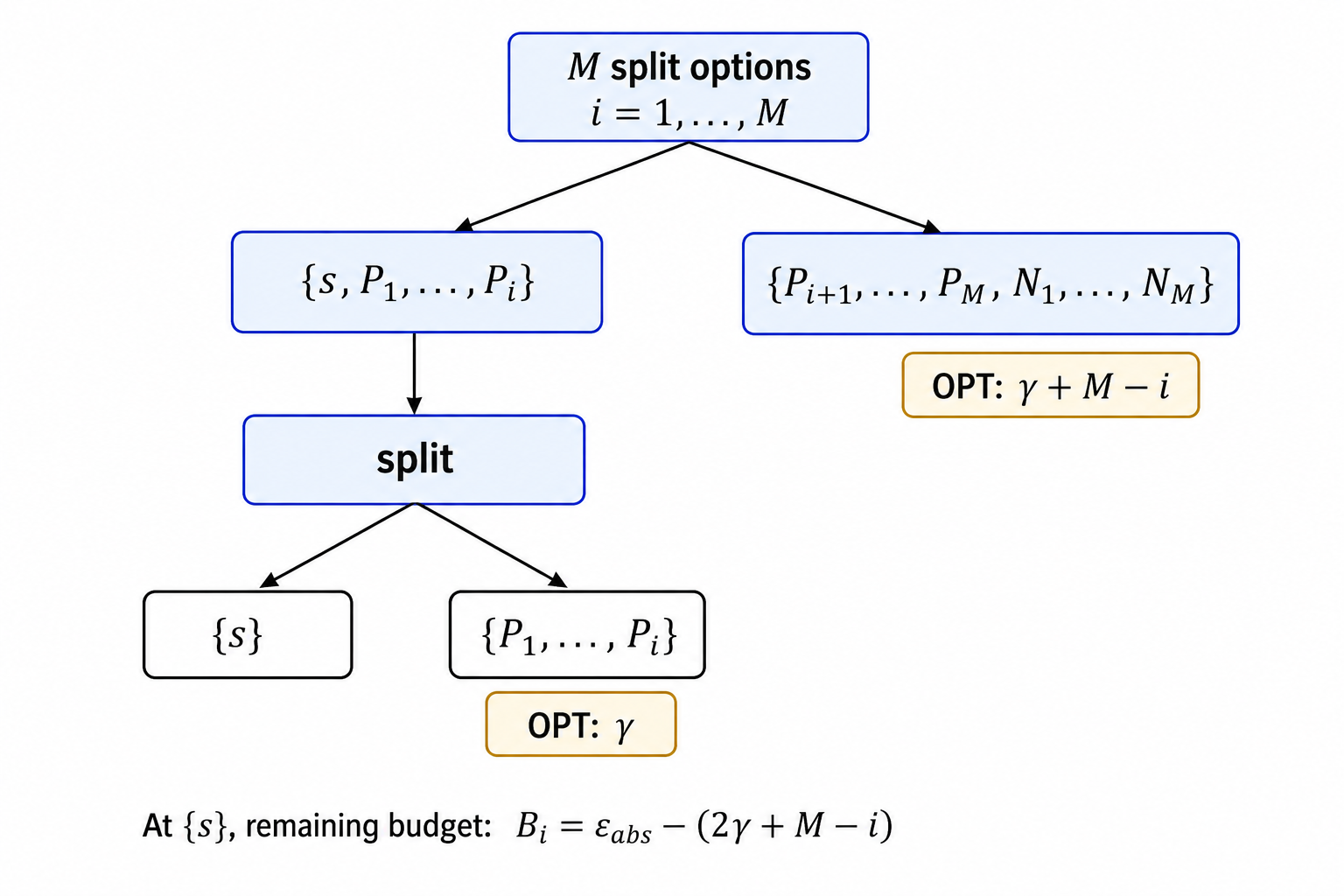}
    \caption{An example showing a dataset where the subproblem $\{ \text{s} \}$ is reached with $M$ different budgets. This is accomplished by using $M$ different splits at the root (they can be viewed as being part of the same continuous feature). There are $2M+1$ samples. $P_i$ are  points with positive labels, $N_i$ are points with negative labels, $s$ is a special point. $\varepsilon_{\text{abs}}$  is the budget at the root. We assume an optimal proxy. }
    \label{fig:morebudgets}
\end{figure}

\begin{theorem}[Distinct Budgets for One Subproblem]
\label{thm:budget-independent}
Fix depth budget \(d=2\). For every \(M\), there exists a binary classification dataset with \(n=2M+1\) samples and \(M\) binary features such that the same subproblem is reached with \(M=\Omega(n)\) distinct remaining budgets. Thus, a budget-dependent AND/OR graph that keys nodes by \((D,d,B)\) may store \(\Omega(n)\) OR nodes for a single canonical subproblem \((D,d)\), while a budget-independent representation stores only one.
\end{theorem}

\begin{proof}
Let \(\gamma\) denote the leaf penalty, and choose \(\gamma > M\). We consider an optimal Rashomon set algorithm. It may or may not use bounds for continuous features. 
Construct a dataset consisting of one special point \(s\), positive points \(P_1,\ldots,P_M\), and negative points \(N_1,\ldots,N_M\). Thus \(n=2M+1\). 
Assign labels \[ y(s)=0, \qquad y(P_i)=1, \qquad y(N_i)=0. \] Use a continuous feature \(x\) whose ordering is \[ s, P_1, P_2,\ldots,P_M, N_1,\ldots,N_M. \] 
For each \(i=1,\ldots,M\), let the threshold \(x \le \tau_i\) split the data into \[ L_i = \{s,P_1,\ldots,P_i\}, \qquad R_i = \{P_{i+1},\ldots,P_M,N_1,\ldots,N_M\}. \] 
Also include a binary feature \(z\) satisfying \[ z(s)=1, \qquad z(P_i)=z(N_i)=0. \]

Therefore, after taking threshold \(x \le \tau_i\), a split on \(z\) inside \(L_i\) produces the two children 
\[ \{s\} \qquad\text{and}\qquad \{P_1,\ldots,P_i\}. \]
Hence, for every \(i\), the same subproblem \(S=\{s\}\) is reached after two splits. 
It remains to show that the remaining budget at this same subproblem depends on \(i\). 
Let the root budget \(\varepsilon_{\mathrm{abs}}\) be large enough that all these trees are feasible; for concreteness, take \[ \varepsilon_{\mathrm{abs}} = 4\gamma . \] 
Since \(\gamma > M\), no depth-one split inside \(R_i\) is cheaper than predicting a single leaf. 
Thus \[ \operatorname{Opt}(R_i,1) = \gamma + (M-i), \]
because \(R_i\) contains \(M-i\) positive points and \(M\) negative points, so the minority class has size \(M-i\). 
The other sibling, \(\{P_1,\ldots,P_i\}\), is pure, so \[ \operatorname{Opt}(\{P_1,\ldots,P_i\},0) = \gamma . \]

Therefore, the budget remaining for the shared subproblem \(S=\{s\}\) along the path corresponding to threshold \(i\) is \[ B_i = \varepsilon_{\mathrm{abs}} - \operatorname{Opt}(R_i,1) - \operatorname{Opt}(\{P_1,\ldots,P_i\},0). \]

Substituting the two expressions above gives \[ B_i = 4\gamma - (\gamma + M-i) - \gamma = 2\gamma - M + i. \] 

It remains to verify that the
leaf \(S=\{s\}\) fits within each remaining budget. Since \(S\) is pure and has
remaining depth zero,
\[
    \operatorname{Opt}(S,0)=\gamma.
\]
Using that \(\gamma > M\), then for every \(i=1,\ldots,M\),
\[
    B_i
    =
    2\gamma-M+i
    \ge
    \gamma.
\]

Thus \(B_1,B_2,\ldots,B_M\) are all distinct and achievable within the fixed root budget.
A budget-dependent graph may therefore create the \(M\) distinct OR nodes \[ (S,0,B_1), (S,0,B_2), \ldots, (S,0,B_M), \]

even though the active sample set and remaining depth are identical in all cases. 
A budget-independent graph instead creates one canonical OR node \((S,0)\) and extends it as larger budgets are encountered. 
Since \(n=2M+1\), this is a \(\Omega(n)\)-factor reduction in OR-node count for this subproblem. On \(n\) samples, the mistakes term takes at most \(n+1\) values, and the number of leaves is at most \(n\) (we do not consider leaves with zero-support in our algorithms, consistent with \citet{arslan2025sorted, heile2026, xin2022exploring}). 
\end{proof}

The previous theorem is related (but not exactly equivalent) to how \citet{arslan2025sorted} stores Rashomon sets. Their implementation uses a BranchTracker together with a UB parameter, which is analogous to our objective budget. From inspecting SORTeD, we find that the representation is not fully budget-independent. It avoids duplicate cached entries when an existing cached UB dominates the requested UB; however, if the same subproblem is later requested with a larger UB, SORTeD constructs a new tracker state rather than extending or merging the old one in place. This theorem does not directly apply to SORTeD, but, something else does. Thus, SORTeD supports restricting a previously computed solution set, but not expanding it. Consequently, repeatedly loosening the root budget forces previously solved subproblems to be recomputed under larger descendant budgets (without a natural "extend" operation), without the opportunity for incremental subgraph reuse.

We now connect to TreeFARMS of \citet{xin2022exploring}, showing the analogous bound except for their model sets. TreeFARMS allows $\gamma$ to be non-integer, so beyond this lower bound (which holds ever for integer $\gamma$), there may be even more duplication.

\begin{theorem}[Distinct Objectives for One Subproblem]
\label{thm:model-set-duplication}

Fix depth budget $d=1$. For every $M$, there exists a binary classification dataset with $n=2M$ samples and $M$ binary features such that a single subproblem-depth pair $S$ has $M=\Omega(n)$ distinct achievable objective values among depth-one trees. Therefore, a model-set representation that identifies instances by $(S,\mathrm{objective})$, as in TreeFARMS, will store $\Omega(n)$ distinct model-set instances for the same subproblem-depth pair $S$. A budget-independent representation that keys the subproblem only by $S$ and stores all objective realizations beneath that node stores only one canonical subproblem node.
\end{theorem}

\begin{proof}
Let $\gamma$ denote the leaf penalty. Construct a binary classification dataset with positive points
\[
    P_1,\ldots,P_M
\]
and negative points
\[
    N_1,\ldots,N_M.
\]
Thus $n=2M$. Let the subproblem that we refer to in the theorem statement be the full support
\[
    S=\{P_1,\ldots,P_M,N_1,\ldots,N_M\}.
\]

We do not continue to remind the reader that $S$ is paired with a depth budget of 1; we instead take it to be implied.

For each $i=1,\ldots,M$, define a binary feature $z_i$ by
\[
    z_i(P_j)=1 \quad\text{if and only if } j\le i,
    \qquad
    z_i(N_j)=0 \quad\text{for all } j.
\]
Splitting $S$ on $z_i$ gives the two children
\[
    L_i=\{P_1,\ldots,P_i\},
    \qquad
    R_i=\{P_{i+1},\ldots,P_M,N_1,\ldots,N_M\}.
\]
The left child $L_i$ is pure positive, so its leaf misclassification cost is zero. The right child $R_i$ contains $M-i$ positives and $M$ negatives, so its optimal leaf prediction is negative and its misclassification cost is $M-i$. Therefore, the objective of the depth-one tree that splits on $z_i$ is
\[
    \mathrm{Obj}_i = (M-i) + 2\gamma,
\]
because the tree has two leaves. As $i$ ranges from $1$ to $M$, the values
\[
    2\gamma+M-1,\quad 2\gamma+M-2,\quad \ldots,\quad 2\gamma
\]
are all distinct. Hence the same canonical subproblem $S$ has $M$ distinct achievable objective values within depth budget $d=1$.

Now choose the Rashomon threshold large enough to include all these trees, for example
\[
     \ge 2\gamma+M-1.
\]
Then all $M$ depth-one trees belong to the Rashomon set. We can also include single-leaf trees with either a positive or negative predicting leaf (depending on the budget).
\end{proof}

We now connect to PRAXIS of \citet{heile2026}, showing a much stronger bound because there is no subgraph caching, only caching of proxy solutions.

\begin{theorem}[Factorial Path Duplication]
\label{thm:path-duplication}
For every depth \(d\), there exists a binary classification dataset with \(d\)
binary features such that, if the Rashomon budget contains all depth-\(d\) trees,
then every depth zero subproblem is represented at least \(d!\) times in the representation from \citet{heile2026}, but only once in a canonical AND/OR graph keyed by active sample
set and remaining depth.
\end{theorem}

\begin{proof}
Let the dataset contain one sample for each binary vector in \(\{0,1\}^d\), with
features \(x_1,\ldots,x_d\). Take the Rashomon budget large enough that every
depth-\(d\) tree is feasible.

Fix any sample \(v=(v_1,\ldots,v_d)\in\{0,1\}^d\), and let
\[
    S_v = \{x\in\{0,1\}^d : x_j=v_j \text{ for all } j=1,\ldots,d\}.
\]
This identifies a subregion containing only \(v\). For any permutation
\(\pi\) of \(\{1,\ldots,d\}\), consider the path that splits first on
\(x_{\pi(1)}\), then on \(x_{\pi(2)}\), and so on, always following the branch
\[
    x_{\pi(k)} = v_{\pi(k)}.
\]
After all \(d\) splits, the active sample set is
\[
    \{x : x_{\pi(1)}=v_{\pi(1)},\ldots,x_{\pi(d)}=v_{\pi(d)}\}
    =
    S_v.
\]
Thus the same depth-zero subproblem \((S_v,0)\) is reached by all \(d!\) feature
orderings.

The representation in \citet{heile2026} caches proxy solutions based on caching a bitvector representing what samples are in the subproblem, but does not cache AND/OR subgraphs. Thus, the only way that AND/OR subgraphs are attached is by directly following the recursion path, which is ordered. That means that the representation in \citet{heile2026} is stored in memory as a path-based trie and identifies subproblems by the ordered sequence of splits used
to reach them, even if the bitvector-based caching of subproblems allows for quick construction of a subgraph for the same problem. Thus, under the representation from \citet{heile2026}, these \(d!\) orderings create \(d!\) distinct nodes for
the same subproblem \((S_v,0)\). In contrast, a canonical AND/OR graph identifies
subproblems by their active sample set and remaining depth. Therefore, all of
these paths point to a single OR node for \((S_v,0)\). Since \(v\) was arbitrary,
the claim holds for every depth-zero subproblem.
\end{proof}

The following results characterize what happens as a continuous feature is binarized using increasingly many thresholds: as the selected thresholds better approximate the full set of continuous splits, the worst-case optimality gap decreases according to Theorem~\ref{thm:snapping}. Theorems~\ref{thm:rank-net}--\ref{thm:quantile-snapping} give concrete ways to select thresholds within each feature while controlling this approximation error and, in some cases, the number of thresholds required. We also study how to refine an existing threshold set by adding new thresholds, directly connecting these guarantees to our anytime algorithm. In particular, Theorem~\ref{thm:rank-refinement-halves} shows that one midpoint-refinement round reduces the covering radius from $\delta$ to at most $\lceil \delta/2 \rceil$.

\begin{theorem}[Binarization Optimality Gap]
\label{thm:snapping}
Let $D=\{(x_i,y_i)\}_{i=1}^n$ be a binary classification dataset with continuous
features. Consider decision trees of depth at most $d$ with objective
\[
    \mathrm{Obj}(T,D,\gamma)
    =
    \mathrm{misclassifications}(T;D)
    +
    \gamma |\mathrm{Leaves}(T)|.
\]
Let $\mathcal{H}_{\mathrm{cont}}$ denote the set of all continuous threshold
splits and let $\mathcal{H}_{\mathrm{bin}}$ denote a restricted set of binary
threshold columns.

Assume that $\mathcal{H}_{\mathrm{bin}}$ satisfies the following Hamming
snapping condition: for every continuous threshold split
$h\in\mathcal{H}_{\mathrm{cont}}$, there exists a binary threshold split
$\hat h\in\mathcal{H}_{\mathrm{bin}}$ such that
\[
    d_H(h,\hat h)
    :=
    \left|\{i : h(x_i)\neq \hat h(x_i)\}\right|
    \le \delta .
\]
Then
\[
    \mathrm{OPT}_{\mathrm{bin}}(D,d)
    -
    \mathrm{OPT}_{\mathrm{cont}}(D,d)
    \le
    (2^d-1)\delta,
\]
where $\mathrm{OPT}_{\mathrm{cont}}(D,d)$ is the optimal objective over
depth-$d$ trees using arbitrary continuous thresholds, and
$\mathrm{OPT}_{\mathrm{bin}}(D,d)$ is the optimal objective over depth-$d$
trees using only splits from $\mathcal{H}_{\mathrm{bin}}$.
\end{theorem}

\begin{proof}
Let $T^\star$ be an optimal depth-$d$ continuous-threshold tree. We construct a
tree $\widehat T$ with the same topology and the same leaf predictions as
$T^\star$, but with each internal split replaced by its snapped binary split.

A depth-$d$ binary tree has at most $2^d-1$ internal nodes. By the Hamming
snapping assumption, each internal split of $T^\star$ can be replaced by a
binary split that disagrees with it on at most $\delta$ training samples. When
one split is replaced, only samples whose branch assignment changes at that
split can possibly change their final leaf assignment. Therefore, replacing one
split can change the prediction of at most $\delta$ training samples, and hence
can increase the number of misclassifications by at most $\delta$.

Applying this argument to every internal node and taking a union bound, replacing
all internal splits increases the number of misclassifications by at most
\[
    (2^d-1)\delta.
\]
The snapped tree $\widehat T$ has the same topology as $T^\star$, so it has the
same number of leaves. Hence the leaf penalty is unchanged:
\[
    \gamma |\mathrm{Leaves}(\widehat T)|
    =
    \gamma |\mathrm{Leaves}(T^\star)|.
\]
Thus
\[
    \mathrm{Obj}(\widehat T;D)
    \le
    \mathrm{Obj}(T^\star;D) + (2^d-1)\delta.
\]
Since $\widehat T$ is feasible for the binarized problem,
\[
    \mathrm{OPT}_{\mathrm{bin}}(D,d)
    \le
    \mathrm{Obj}(\widehat T;D).
\]
Combining this with
$\mathrm{Obj}(T^\star;D)=\mathrm{OPT}_{\mathrm{cont}}(D,d)$ gives
\[
    \mathrm{OPT}_{\mathrm{bin}}(D,d)
    -
    \mathrm{OPT}_{\mathrm{cont}}(D,d)
    \le
    (2^d-1)\delta.
\]
\end{proof}

\begin{algorithm}[t]
  \caption{Greedy threshold selection algorithm}
  \label{alg:rank-net}
  \begin{algorithmic}[1]
    \REQUIRE Attainable split ranks $R_j = \{r_1 < \cdots < r_m\}$, radius $\delta$
    \STATE $B_j \gets \emptyset$, \quad $i \gets 1$
    \WHILE{$i \le m$}
      \STATE $a \gets r_i$
      \STATE Let $q$ be the largest index such that $r_q \le a + \delta$
      \STATE $B_j \gets B_j \cup \{r_q\}$
      \STATE Set $i$ to the smallest index such that $r_i > r_q + \delta$
    \ENDWHILE
    \RETURN $B_j$
  \end{algorithmic}
\end{algorithm}

 Algorithm \ref{alg:rank-net} provides one way to choose different thresholds for a continuous feature such that every split on a midpoint of two unique values is within $\delta$ of a chosen threshold. This is a standard greedy algorithm that appears in scheduling or covering problems. As we show in Theorem \ref{thm:rank-net}, it is optimal because we are working in 1D (handling one continuous feature at a time).
 
\begin{theorem}[Selecting Thresholds with Covering Radius Guarantee]
\label{thm:rank-net}
  Fix a continuous feature $j$ and let
  \[
    R_j = \{r_1 < \cdots < r_m\}
  \]
  be its attainable split ranks, i.e., the cumulative sample counts after each
  distinct feature value.  For any integer $\delta \ge 0$, Algorithm~\ref{alg:rank-net}
  returns a set $B_j \subseteq R_j$ such that
  \[
    \max_{r \in R_j} \min_{b \in B_j} |r - b| \le \delta.
  \]
  Moreover, among subsets of $R_j$ with the above covering radius guarantee (at most $\delta$),
  the algorithm uses the minimum possible number of selected thresholds.
\end{theorem}
 
\begin{proof}
  Let $a$ be the leftmost uncovered attainable rank.  Any selected rank that
  covers $a$ must lie in $R_j \cap (-\infty, a + \delta]$.  The algorithm
  chooses the largest such rank, say $b$.  This choice covers all attainable
  ranks up to $b + \delta$.  No feasible solution can cover $a$ using a
  selected rank to the right of $b$, because $b$ is the largest attainable
  rank within distance $\delta$ of $a$.  Therefore, replacing the first
  selected rank of any optimal solution by $b$ cannot decrease the set of
  ranks covered to the right.  After removing the ranks covered by $b$, the
  same argument applies recursively to the remaining suffix.  Thus the greedy
  algorithm is optimal.  The covering property follows directly from the
  construction.
\end{proof}

\begin{theorem}[Optimality Gap induced by Threshold-Selection]
\label{cor:rank-net-gap}
  For each continuous feature $j$, construct $B_j$ using
  Algorithm~\ref{alg:rank-net} with radius $\delta$, and let
  $\mathcal{H}_{\mathrm{bin}}$ be the corresponding selected threshold
  columns.  Then
  \[
    \mathrm{OPT}_{\mathrm{bin}}(D, d)
    -
    \mathrm{OPT}_{\mathrm{cont}}(D, d)
    \le (2^d - 1)\,\delta.
  \]
  Consequently, choosing
  \[
    \delta \le \frac{\epsilon\,}{2^d - 1}
  \]
  gives a objective gap of at most $\epsilon$.
\end{theorem}
 
\begin{proof}
  Every continuous threshold on feature $j$ is equivalent on the training
  data to one attainable rank $r \in R_j$, since tied feature values cannot
  be separated.  By Theorem~\ref{thm:rank-net}, there exists a selected rank
  $b \in B_j$ with $|r - b| \le \delta$.  The corresponding threshold columns
  differ on exactly the samples whose ranks lie between $r$ and $b$, so their
  Hamming distance is at most $\delta$.  Thus $\mathcal{H}_{\mathrm{bin}}$
  satisfies the Hamming snapping condition of Theorem~\ref{thm:snapping},
  and the bound $(2^d - 1)\,\delta$ follows immediately. 
\end{proof}

\begin{theorem}[Number of Thresholds for some Covering Radius]
\label{thm:rank-net-size}
Let \(R_j=\{r_1<\cdots<r_m\}\) be the attainable split ranks for feature \(j\).
Algorithm~\ref{alg:rank-net} selects at most
\[
    |B_j|
    \le
    \left\lceil \frac{r_m-r_1+1}{2\delta+1}\right\rceil
    \le
    \left\lceil \frac{n}{2\delta+1}\right\rceil
\]
thresholds.
\end{theorem}

\begin{proof}
Each selected rank covers all attainable ranks within distance \(\delta\), i.e.,
an interval of at most \(2\delta+1\) integer ranks. Algorithm~\ref{alg:rank-net}
is optimal by Theorem~\ref{thm:rank-net}, so it uses no more thresholds than
are needed to cover the full integer interval \([r_1,r_m]\) by intervals of
length \(2\delta+1\). This requires at most
\[
    \left\lceil \frac{r_m-r_1+1}{2\delta+1}\right\rceil
\]
intervals. Since \(r_m-r_1+1\le n\), the second bound follows.
\end{proof}

\begin{theorem}[Quantile Snapping Bound]
\label{thm:quantile-snapping}
  Fix a continuous feature $j$ and a deterministic ordering of training samples
  by this feature, breaking ties arbitrarily.  Let
  $\mathcal{H}_{\mathrm{rank}}$ be the set of all prefix threshold columns in
  this ordering (cuts between distinct values and within tied values).  For an integer $K \ge 1$, select the $K$ quantile ranks
  \[
    b_q = \left\lfloor \frac{q\,n}{K+1} \right\rceil, \qquad q = 1, \ldots, K,
  \]
  and let $\mathcal{H}_{\mathrm{bin}}$ be the corresponding threshold columns.
  Then every $h \in \mathcal{H}_{\mathrm{rank}}$ has a selected
  $\hat{h} \in \mathcal{H}_{\mathrm{bin}}$ with
  \[
    d_H(h, \hat{h}) \le \left\lceil \frac{n}{K+1} \right\rceil.
  \]
  Consequently, Theorem~\ref{thm:snapping} gives
  \[
    \mathrm{OPT}_{\mathrm{bin}}(D, d)
    - \mathrm{OPT}_{\mathrm{rank}}(D, d)
    \le (2^d - 1) \left\lceil \frac{n}{K+1} \right\rceil.
  \]
\end{theorem}
 
\begin{proof}
  A prefix-threshold column is determined by a prefix size $r \in \{1, \ldots, n-1\}$.
  The Hamming distance between prefix columns of ranks $r$ and $s$ is exactly
  $|r - s|$, since they differ precisely on the samples between the two
  prefix endpoints.  The $K$ quantile ranks divide $\{1, \ldots, n-1\}$ into
  $K+1$ intervals of length at most $\lceil n/(K+1) \rceil$, so every $r$ has
  a quantile rank $b_q$ with $|r - b_q| \le \lceil n/(K+1) \rceil$.
  Applying Theorem~\ref{thm:snapping} with $\delta = \lceil n/(K+1) \rceil$
  gives the claimed bound.
\end{proof}
 
Theorem \ref{thm:quantile-snapping} assumes a fixed tie-breaking order.
In practice we use value thresholds $x_j \le t$ at empirical quantile values,
deduplicate, and discard constant columns. In the absence of ties this coincides with the rank
construction above. For Theorem \ref{thm:rank-refinement-halves}, we make the same assumption.

\begin{theorem}[Midpoint Refinement Halves Covering Radius]
\label{thm:rank-refinement-halves}
  Fix a continuous feature \(j\) and fix a deterministic ordering of the
  training samples by this feature, breaking ties arbitrarily.  Let
  \[
    R_j^{\mathrm{rank}} = \{1,\ldots,n-1\}
  \]
  denote the set of nonconstant prefix ranks in this ordering (i.e., each
  \(r\in R_j^{\mathrm{rank}}\) determines the prefix-threshold column that
  assigns the first \(r\) samples to one branch and the remaining samples to
  the other; ranks that split tied values are allowed).  For
  \(B \subseteq R_j^{\mathrm{rank}}\), define its covering radius by
  \[
    \mathrm{rad}(B)
    =
    \max_{r\in R_j^{\mathrm{rank}}}
    \min_{b\in B} |r-b|.
  \]
  Suppose \(\mathrm{rad}(B) \le \delta\).  Form \(B^+\) by adding to \(B\):
  for every adjacent pair \(b<b'\) in \(B\), the nearest integer rank to
  \((b+b')/2\); for the left boundary interval \([1,b_{\min}]\), the nearest
  integer rank to \((1+b_{\min})/2\); and for the right boundary interval
  \([b_{\max},n-1]\), the nearest integer rank to \((b_{\max}+n-1)/2\).
  Then
  \[
    \mathrm{rad}(B^+) \le \left\lceil \frac{\delta}{2}\right\rceil.
  \]
\end{theorem}

\begin{proof}
  We prove that every rank \(r\in R_j^{\mathrm{rank}}\) is within distance
  at most \(\lceil \delta/2\rceil\) of some rank in \(B^+\).

  \textbf{Interior intervals.}
  Let \(b<b'\) be adjacent ranks in \(B\), and consider the integer interval
  \([b,b']\).  Since \(B\) has covering radius at most \(\delta\), every
  integer rank in \([b,b']\) is within distance \(\delta\) of either \(b\) or
  \(b'\).  Therefore
  \[
    \left\lfloor \frac{b'-b}{2} \right\rfloor \le \delta.
  \]
 Let
\[
  \mu = \operatorname{round}\!\left(\frac{b+b'}{2}\right),
\]
where \(\operatorname{round}\) returns the nearest integer (breaking
ties arbitrarily). The largest gap between consecutive ranks in
  \(\{b,\mu,b'\}\) is at most
  \[
    \left\lceil \frac{b'-b}{2} \right\rceil .
  \]
  Hence every integer rank in \([b,b']\) is within distance at most
  \[
    \left\lfloor
      \frac{1}{2}\left\lceil \frac{b'-b}{2}\right\rceil
    \right\rfloor
  \]
  of one of \(b,\mu,b'\).  Since
  \[
    \left\lfloor \frac{b'-b}{2}\right\rfloor \le \delta,
  \]
  the overall term is at most \(\lceil \delta/2\rceil\).  Thus, every rank in
  the interior interval \([b,b']\) is covered by \(B^+\) with radius at most
  \(\lceil \delta/2\rceil\).

  \textbf{Boundary intervals.}
  For the left boundary, the assumed covering guarantee implies
  \[
    b_{\min}-1 \le \delta,
  \]
  since \(b_{\min}\) is the closest selected rank in \(B\) to the leftmost
  rank \(1\).  Let
  \[
  \mu_L = \operatorname{round}\!\left(\frac{1+b_{\min}}{2}\right)
\]
  be the integer rank added in the left boundary interval.  Then every rank in
  \([1,b_{\min}]\) is within distance at most
  \[
    \left\lceil \frac{b_{\min}-1}{2}\right\rceil
    \le
    \left\lceil \frac{\delta}{2}\right\rceil
  \]
  of either \(\mu_L\) or \(b_{\min}\).

  The right boundary is identical. 
  Finally, every rank in \(R_j^{\mathrm{rank}}\) lies either on a boundary or between two adjacent ranks.  Each such
  interval is covered by \(B^+\) with radius at most
  \(\lceil \delta/2\rceil\).  Therefore
  \[
    \mathrm{rad}(B^+) \le \left\lceil \frac{\delta}{2}\right\rceil.
  \]
\end{proof}

\begin{theorem}[Fixed-Binarization Superset Guarantee]
\label{thm:fixed-binarization-superset}
Using a proxy restricted to a fixed binarization, if the proxy is optimal over
that binarization, then our continuous-threshold algorithm recovers a superset
of the trees returned by exact Rashomon set enumeration over that binarization.
\end{theorem}

\begin{proof}
    Because the proxy is optimal over some binarization, it satisfies the robustness conditions we assume in our pruning. See \citet{brița2025optimal} for details of this fact. Therefore, we will never prune a split that could actually be completed by the proxy to be within budget. This means that we consider optimal completions over the binarization on every split in the binarization. That is, we fully enumerate the Rashomon set over binary features. We also evaluate binarized optimal completions on other thresholds. These thresholds imply we could return more trees than we would by finding a Rashomon set over the binarization.
\end{proof}

There are many interesting theoretical guarantees in \citet{heile2026}, including Theorem~3.5, Corollary~3.6, Theorem~A.3, and Corollary~A.6. These all extend to continuous features whenever the proxy does not violate the continuous-feature pruning conditions. This holds, for example, for a majority-leaf proxy or an optimal proxy restricted to any subset of split options. In particular, we would like to highlight that if the proxy's optimality gap is maximized at the root, we recovers the full Rashomon set without additional slack. In the worst case, the proxy may be optimal at the root and attain its full gap only on a descendant subproblem; multiplying the root budget by $c$ still guarantees full recovery (where $c$ is the worst-case optimality gap). Although we are not aware of formal approximation algorithms for decision tree optimization when considering the same split choices, our preceding continuous-feature results bound the gap between optimal trees with and without discretization. Thus, an optimal proxy over a carefully chosen subset of thresholds yields both computational savings and an explicit recovery guarantee. In practice, we do find using a near-optimal proxy over an extended feature set is even better than this option.

\FloatBarrier

\FloatBarrier

\section{Proxy Algorithms}
\label{app:proxy-algorithms}

\subsection{LicketySNIP and Proxy Strength}
\label{app:licketysnip}

In this section, we provide pseudocode for the LicketySNIP proxy algorithm (Algorithm~\ref{alg:licketysnip}). The parameter $\ell$ controls the proxy strength. When $\ell=0$, LicketySNIP reduces to a fast greedy completion. When $\ell=1$, it roughly chooses the best split using greedy completions, fixes that split, and then recurses on the two child subproblems. In other words, it recursively chooses splits whose quality is estimated by greedy completions. This $\ell=1$ setting is the default algorithm we refer to as LicketySNIP. When $\ell=2$, the algorithm instead chooses the best split using LicketySNIP($\ell=1$) completions, and so on.

The reason we say ``roughly'' is that we apply the same neighborhood-pruning idea when evaluating these cheap split completions. Thus, if one threshold has a poor proxy completion, we choose to prune nearby thresholds, assuming the greedy completions are robust. 

This process creates a hierarchy of proxy algorithms; more details and provable cache reuse are given in \citet{heile2026}. For example, if we first use greedy completions as our proxy and then upgrade to LicketySNIP($\ell=1$), the greedy-completion caches can be reused when solving the $\ell=1$ subproblems. The same reuse continues as we increase $\ell$, until the proxy eventually becomes exact. This lookahead parameter $\ell$ is exactly what we mean by proxy strength in the anytime algorithm. As a result, LicketySNIP is naturally amenable to staged refinement: we can begin with cheap proxy evaluations, cache their results, and then increase $\ell$ to obtain stronger certificates while reusing the work already performed.

Unlike the threshold-to-proxy completion maps used by the Rashomon set algorithm, the map \(E\) in \textsc{BestContinuousSplitSNIP} is local to a single continuous-feature search and is not cached across calls. Persisting \(E\) is valuable for Rashomon set enumeration because the same subproblem--depth pair may be revisited many times as the anytime algorithm activates additional thresholds or as iterative budget refinement increases the budget given to it. In the proxy optimization, by contrast, once a subproblem and depth have been solved, they need not be revisited. This is why we do not persist $E$. However, maintaining a temporary \(E\) during the current queue search remains useful because the proxy incumbent \(B^\star\) can improve as better thresholds are found. A threshold evaluated earlier may therefore become a stronger pruning certificate later: as \(B^\star\) decreases, the gap \(P_L(t)+P_R(t)-B^\star\) increases, allowing a larger neighborhood around \(t\) to be pruned without recomputing its proxy completions. Accordingly, whenever an interval is popped from \(Q\), \textsc{ShrinkFromEvaluated} reapplies the evaluations accumulated in the local map \(E\) using the current incumbent, after which \(E\) is discarded when the feature search terminates.

\begin{algorithm}[!t]
\caption{\textsc{LicketySNIP}$(D,d,\ell,\gamma,\mathcal{S})$}
\label{alg:licketysnip}
\begin{algorithmic}[1]
\REQUIRE Subproblem $D$, remaining depth $d$, lookahead $\ell$, per-leaf penalty $\gamma$,
and threshold registry $\mathcal{S}$

\IF{$d=0$}
    \STATE \textbf{return} $\textsc{LeafObj}(D,\gamma)$
\ENDIF

\IF{$\ell=0$}
    \STATE \textbf{return} $\textsc{GreedyContinuous}(D,d,\gamma,\mathcal{S})$
\ENDIF

\IF{$d=1$}
    \STATE \textbf{return} $\textsc{ExactStumpContinuous}(D,\gamma,\mathcal{S})$
\ENDIF

\STATE $\ell \gets \min\{\ell,d-1\}$

\IF{$(D,d,\ell)$ is in the LicketySNIP cache}
    \STATE \textbf{return} cached value
\ENDIF

\STATE $L_{\mathrm{leaf}} \gets \textsc{LeafObj}(D,\gamma)$
\IF{$L_{\mathrm{leaf}} \le 2\gamma$}
    \STATE Cache and \textbf{return} $L_{\mathrm{leaf}}$
\ENDIF

\STATE $B \gets L_{\mathrm{leaf}}$
\COMMENT{Current best proxy objective}
\STATE $s^\star \gets \bot$
\COMMENT{Best split found so far}

\FOR{\textbf{each ordinary binary feature} $j$}
    \STATE $(D_L,D_R) \gets \textsc{Partition}(D,j)$
    \IF{$D_L=\emptyset$ \textbf{or} $D_R=\emptyset$}
        \STATE \textbf{continue}
    \ENDIF

    \STATE $P_L \gets \textsc{LicketySNIP}(D_L,d-1,\ell-1,\gamma,\mathcal{S})$
    \STATE $P_R \gets \textsc{LicketySNIP}(D_R,d-1,\ell-1,\gamma,\mathcal{S})$

    \IF{$P_L+P_R < B$}
        \STATE $B \gets P_L+P_R$
        \STATE $s^\star \gets j$
    \ENDIF
\ENDFOR

\FOR{\textbf{each continuous feature group} $J=[a,b)$}
    \STATE $[a',b') \gets \textsc{RestrictRange}(\mathcal{S},D,J)$
    \IF{$a' \ge b'$}
        \STATE \textbf{continue}
    \ENDIF

    \STATE $(B_J,s_J) \gets
    \textsc{BestContinuousSplitSNIP}(D,d,\ell-1,\gamma,\mathcal{S},[a',b'),B)$

    \IF{$B_J < B$}
        \STATE $B \gets B_J$
        \STATE $s^\star \gets s_J$
    \ENDIF
\ENDFOR

\STATE $A \gets L_{\mathrm{leaf}}$

\IF{$s^\star \neq \bot$}
    \STATE $(D_L,D_R) \gets \textsc{Partition}(D,s^\star)$
    \STATE $\mathcal{S}_L \gets \mathcal{S}$; \quad
    $\mathcal{S}_R \gets \mathcal{S}$
    \IF{$s^\star$ is a threshold of continuous feature $c$}
        \STATE $\textsc{LowerUpperBound}(\mathcal{S}_L,c,s^\star-1)$
        \STATE $\textsc{RaiseLowerBound}(\mathcal{S}_R,c,s^\star+1)$
    \ENDIF

    \STATE $Q_L \gets \textsc{LicketySNIP}(D_L,d-1,\ell,\gamma,\mathcal{S}_L)$
    \STATE $Q_R \gets \textsc{LicketySNIP}(D_R,d-1,\ell,\gamma,\mathcal{S}_R)$

    \STATE $A \gets \min\{A,Q_L+Q_R,B\}$
\ENDIF

\STATE Cache $A$ for $(D,d,\ell)$
\STATE \textbf{return} $A$
\end{algorithmic}
\end{algorithm}

\begin{algorithm}[!t]
\caption{\textsc{BestContinuousSplitSNIP}$(D,d,\ell,\gamma,\mathcal{S},[a,b),B)$}
\label{alg:best-continuous-split-snip}
\begin{algorithmic}[1]
\REQUIRE Subproblem $D$, remaining depth $d$, child lookahead $\ell$, penalty $\gamma$,
threshold registry $\mathcal{S}$, threshold interval $[a,b)$, and incumbent objective $B$

\STATE $B^\star \gets B$; \quad $t^\star \gets \bot$
\STATE $E \gets \emptyset$
\COMMENT{$E[t]=(P_L(t),P_R(t))$ stores evaluated proxy completions}
\STATE $Q \gets \{[a,b-1]\}$
\STATE $L \gets a$; \quad $U \gets b-1$

\WHILE{$Q \neq \emptyset$}
    \STATE Remove an interval $[i,j]$ from $Q$

    \STATE $(i,j) \gets
    \textsc{ShrinkFromEvaluated}(E,D,i,j,L,U,B^\star)$

    \STATE $i \gets \max\{i,L\}$; \quad $j \gets \min\{j,U\}$

    \IF{$i>j$}
        \STATE \textbf{continue}
    \ENDIF

    \STATE $m \gets \lfloor(i+j)/2\rfloor$
    \STATE $(D_L,D_R) \gets \textsc{Partition}(D,m)$

    \IF{$D_L=\emptyset$}
        \STATE $L \gets \max\{L,m+1\}$
        \STATE Add $[m+1,j]$ to $Q$
        \STATE \textbf{continue}
    \ENDIF

    \IF{$D_R=\emptyset$}
        \STATE $U \gets \min\{U,m-1\}$
        \STATE Add $[i,m-1]$ to $Q$
        \STATE \textbf{continue}
    \ENDIF

    \STATE $\mathcal{S}_L \gets \mathcal{S}$; \quad
    $\mathcal{S}_R \gets \mathcal{S}$
    \STATE $\textsc{LowerUpperBound}(\mathcal{S}_L,c,m-1)$
    \STATE $\textsc{RaiseLowerBound}(\mathcal{S}_R,c,m+1)$

    \STATE $P_L \gets \textsc{LicketySNIP}(D_L,d-1,\ell,\gamma,\mathcal{S}_L)$
    \STATE $P_R \gets \textsc{LicketySNIP}(D_R,d-1,\ell,\gamma,\mathcal{S}_R)$
    \STATE $E[m] \gets (P_L,P_R)$
    \STATE $P \gets P_L+P_R$

    \IF{$P < B^\star$}
        \STATE $B^\star \gets P$
        \STATE $t^\star \gets m$
    \ENDIF

    \IF{$P_L=\gamma$}
        \STATE $L \gets \max\{L,m+1\}$
        \COMMENT{Farther-left thresholds cannot improve the left child}
    \ENDIF

    \IF{$P_R=\gamma$}
        \STATE $U \gets \min\{U,m-1\}$
        \COMMENT{Farther-right thresholds cannot improve the right child}
    \ENDIF

    \IF{$P \ge B^\star$}
        \STATE $\Delta \gets \max\{1,P-B^\star\}$
        \STATE $u \gets \textsc{TightenUpper}(D,m,i,m-1,\Delta)$
        \STATE $v \gets \textsc{TightenLower}(D,m,m+1,j,\Delta)$

        \IF{$i \le u$}
            \STATE Add $[i,u]$ to $Q$
        \ENDIF
        \IF{$v \le j$}
            \STATE Add $[v,j]$ to $Q$
        \ENDIF
    \ELSE
        \STATE Add $[i,m-1]$ and $[m+1,j]$ to $Q$
    \ENDIF
\ENDWHILE

\STATE \textbf{return} $(B^\star,t^\star)$
\end{algorithmic}
\end{algorithm}

Lines 4-8 of Algorithm \ref{alg:shrink-from-evaluated} uses a pruning technique that has not been previously used in this work. Given a subtree that handles splits on the left, a subtree that handles splits on the right, and haven't handled the middle, and this is already over the incumbent solution, you can prune the entirity of the middle. While we do have both left and right proxy completions if we have one of them, this is particularly effective because we can avoid needing to compute active sample distances. Thus, even if this does not help most of the time, it is a cheap check that can be very rewarding. We do not apply this pruning in our Rashomon set enumeration because that algorithm already processes every entry in its completion map $E$ under a fixed objective bound. Best-split selection requires a different design: its incumbent can improve throughout the search, so previously evaluated completions must be reconsidered adaptively as intervals are popped from the queue. Thus, the fixed-bound structure of Rashomon enumeration supports largely upfront pruning, whereas proxy split selection require something adaptive. Needing to be adaptive motivates the inclusion of this lightweight check.

\begin{algorithm}[!t]
\caption{\textsc{ShrinkFromEvaluated}%
\((E,D,i,j,L,U,B^\star)\)}
\label{alg:shrink-from-evaluated}
\begin{algorithmic}[1]
\REQUIRE Local map \(E\) from evaluated threshold indices to proxy completions,
subproblem \(D\), popped interval \([i,j]\), current global bounds \([L,U]\),
and incumbent proxy objective \(B^\star\)

\STATE \(u\gets\textsc{PredecessorKey}(E,i)\)
\COMMENT{Largest evaluated threshold strictly below \(i\), or \(\bot\)}
\STATE \(v\gets\textsc{SuccessorKey}(E,j)\)
\COMMENT{Smallest evaluated threshold strictly above \(j\), or \(\bot\)}

\STATE \COMMENT{This is a pruning technique from \citet{brița2025optimal}.}
\IF{\(u\neq\bot\) \textbf{ and } \(v\neq\bot\)}
    \STATE \((P_L(u),P_R(u))\gets E(u)\)
    \STATE \((P_L(v),P_R(v))\gets E(v)\)
    \IF{\(P_L(u)+P_R(v)\geq B^\star\)}
        \STATE \textbf{return} \((1,0)\)
        \COMMENT{No threshold in \([i,j]\) can strictly improve the incumbent}
    \ENDIF
\ENDIF

\IF{\(u\neq\bot\)}
    \STATE \((P_L(u),P_R(u))\gets E(u)\)
    \STATE \(P(u)\gets P_L(u)+P_R(u)\)
    \IF{\(P(u)\geq B^\star\)}
        \STATE \(\Delta\gets\max\{1,P(u)-B^\star\}\)
        \STATE \(i\gets\max\!\left\{
        i,\textsc{RightBoundary}(D,u,i,U,\Delta)
        \right\}\)
        \COMMENT{Move right until strict improvement is possible}
    \ENDIF
\ENDIF

\IF{\(v\neq\bot\)}
    \STATE \((P_L(v),P_R(v))\gets E(v)\)
    \STATE \(P(v)\gets P_L(v)+P_R(v)\)
    \IF{\(P(v)\geq B^\star\)}
        \STATE \(\Delta\gets\max\{1,P(v)-B^\star\}\)
        \STATE \(j\gets\min\!\left\{
        j,\textsc{LeftBoundary}(D,v,j,L,\Delta)
        \right\}\)
        \COMMENT{Move left until strict improvement is possible}
    \ENDIF
\ENDIF

\STATE \(i\gets\max\{i,L\}\); \quad \(j\gets\min\{j,U\}\)
\STATE \textbf{return} \((i,j)\)
\end{algorithmic}
\end{algorithm}

For the greedy subroutine, there are two natural ways to handle continuous features. One option is to avoid bitvectors and instead create sorted lists of sample indices for each feature, maintaining these lists across recursive calls. With this representation, we can scan each feature in sorted order, updating the class counts on the left and right sides incrementally to efficiently identify the best split.

The empirical question is whether the cost of constructing and maintaining these sorted index lists is justified relative to running the greedy solver directly over binarized threshold features. Because our main algorithm does not maintain sorted index lists throughout the search, this representation must be constructed specifically for each greedy proxy call. This setup cost can outweigh the savings, particularly when the proxy is invoked on small subproblems deep in the search space. Empirically, the sorted-index representation is not beneficial on small- to medium-sized problems and can incur a ($2\text{--}3\times$) overhead. On larger problems, however, it can provide substantial speedups.

This comparison should not be interpreted as evidence that binarized threshold representations are generally slower in our implementation. In the greedy split-selection loop, we do not use bounds to skip groups of similar thresholds, whereas the main algorithm applies additional pruning that can eliminate large collections of thresholds at once. The greedy solver is also particularly well suited to sorted index lists: evaluating a candidate split requires only the feature values and the corresponding class counts; the latter can be updated efficiently when traversing left to right. More general decision-tree optimization routines must preserve joint information about samples across features and subproblems.

Appendix \ref{app:timing-memory-recall} compares the sorted-index implementation with a greedy subroutine operating over binarized threshold features. The latter is unchanged from \citet{heile2026}, except that it skips consecutive thresholds that induce the same partition of the current subproblem.

In the main-paper experiments, we use fast bitvector operations for the greedy subroutine, either over all thresholds or over a restricted binarization. While it is true that a better default for large-scale problems is to create sorted-lists of indices, we find that an even better default is to consider a smaller binarization with fast bitvector operations. 

When a greedy proxy subroutine is called, and if we do not want to scan the pre-binarized threshold columns with bitvectors, we instead obtain a sorted list of active sample indices for each continuous feature. This can be done in two ways. 
First, we may sort the samples once at the root for each continuous feature and, at a later subproblem, filter this global ordering to the samples active in the current subproblem. 
Second, we may collect the active samples in the current subproblem and sort them locally by the feature value. Our implementation uses the former approach. Given a subproblem $D$ and a continuous feature $j$, let \[ i_1,\ldots,i_m \] be the active samples in $D$, sorted so that 
\[ x_{i_1j} \le x_{i_2j} \le \cdots \le x_{i_mj}. \]

We evaluate all candidate thresholds for feature $j$ in one left-to-right sweep \citep{quinlan2014c45}. Initially, all samples lie on the right. Let $n = |D|$ and $p$ be the number of positive samples in $D$. We maintain left-side counts $n_L$ and $p_L$, initialized to zero. At each distinct value $v$, we move the entire tie block $B_v = \{i \in D : x_{ij} = v\}$ from right to left and update
\[
n_L \gets n_L + |B_v|,
\qquad
p_L \gets p_L + \sum_{i \in B_v} \mathbf{1}\{y_i = 1\}.
\]
The right-side counts follow by subtraction:
\[
n_R = n - n_L,
\qquad
p_R = p - p_L.
\]
We then score the threshold after $v$ from these counts. For information gain,
\[
\mathrm{Score}(v) = H(p/n)
-\frac{n_L}{n}H(p_L/n_L)
-\frac{n_R}{n}H(p_R/n_R),
\]
where $H(q) = -q\log q - (1-q)\log(1-q)$. At depth $1$, we instead use the negative sum of the two child leaf objectives. After selecting the best threshold, we partition the sorted sample indices into the two child subproblems and recurse.

\FloatBarrier

\subsection{Improvements without Proxy Caching}
\label{app:walk}

\begin{algorithm}[!b]
\caption{\textsc{WalkFromSolvedThreshold}%
$(G,D,d,\gamma,\varepsilon_{\mathrm{abs}},\mathcal{S},c,q,P_L,P_R)$}
\label{alg:walk-from-solved-threshold}
\begin{algorithmic}[1]
\REQUIRE Current OR node $G$, subproblem $D$, depth $d$,
budget $\varepsilon_{\mathrm{abs}}$, threshold registry $\mathcal{S}$,
continuous feature $c$, active position $q$, and proxy completions
$(P_L,P_R)$ for the threshold

\STATE $\mathcal{A} \gets
\textsc{ActiveThresholds}(\mathcal{S},D,c)$
\STATE $t \gets \mathcal{A}[q]$

\STATE $(G_L,G_R) \gets
\begin{aligned}[t]
\textsc{AddOrExtendSplit}(&G,D,t,d,\gamma,\varepsilon_{\mathrm{abs}},\\
                         &\mathcal{S},P_L,P_R)
\end{aligned}$

\IF{$G_L=\emptyset$ \textbf{ or } $G_R=\emptyset$}
    \STATE \textbf{return}
\ENDIF

\STATE $L \gets \textsc{MinObjective}(G_L)$
\STATE $U \gets \textsc{MinObjective}(G_R)$
\STATE $\Delta \gets \varepsilon_{\mathrm{abs}}-(L+U)$

\IF{$\Delta<0$}
    \STATE \textbf{return}
\ENDIF

\STATE $
\begin{aligned}[t]
\textsc{WalkDirection}(&G,D,d,\gamma,\varepsilon_{\mathrm{abs}},
\mathcal{S},c,\\
& q,t,L,U,\textsc{Left})
\end{aligned}$

\STATE $
\begin{aligned}[t]
\textsc{WalkDirection}(&G,D,d,\gamma,\varepsilon_{\mathrm{abs}},
\mathcal{S},c,\\
& q,t,L,U,\textsc{Right})
\end{aligned}$

\end{algorithmic}
\end{algorithm}

Algorithm \ref{alg:walk-from-solved-threshold} gives an optional subroutine for \textsc{EnumContFeature} that uses solved thresholds to seed nearby ones. When a threshold is within budget, the minimum objectives of its left and right child subgraphs give upper bounds on the optimal child solutions for neighboring thresholds. If these induced bounds are within budget, we skip the proxy-pruning test for the neighboring threshold and call \textsc{AddOrExtendSplit} directly. This procedure can be chained across consecutive thresholds, using each newly solved threshold to seed the next, while provably returning a superset of the trees that would otherwise be returned (Theorem \ref{alg:walk-from-solved-threshold}). In practice, this is unnecessary when we can cache every intermediate solution computed during each proxy query, but it can be very beneficial in memory-constrained settings.

\begin{algorithm}[!t]
\caption{\textsc{WalkDirection}%
$(G,D,d,\gamma,\varepsilon_{\mathrm{abs}},\mathcal{S},
c,q,t,L,U,\mathrm{dir})$}
\label{alg:walk-direction}
\begin{algorithmic}[1]
\REQUIRE Current OR node $G$, subproblem $D$, depth $d$,
budget $\varepsilon_{\mathrm{abs}}$, threshold registry $\mathcal{S}$,
continuous feature $c$, solved position $q$, solved threshold $t$,
child minima $L,U$, and direction $\mathrm{dir}$

\STATE $\mathcal{A} \gets
\textsc{ActiveThresholds}(\mathcal{S},D,c)$
\STATE $t_{\mathrm{anchor}} \gets t$
\STATE $L^{\mathrm{anchor}} \gets L$
\STATE $U^{\mathrm{anchor}} \gets U$
\STATE $\Delta \gets
\varepsilon_{\mathrm{abs}}-
(L^{\mathrm{anchor}}+U^{\mathrm{anchor}})$

\IF{$\mathrm{dir}=\textsc{Left}$}
    \STATE $\mathcal{R} \gets (q-1,q-2,\ldots,0)$
\ELSE
    \STATE $\mathcal{R} \gets
    (q+1,q+2,\ldots,|\mathcal{A}|-1)$
\ENDIF

\FOR{\textbf{each} $r\in\mathcal{R}$}
    \STATE $s \gets \mathcal{A}[r]$

    \IF{$\mathrm{dir}=\textsc{Left}$}
        \STATE $x \gets
        \mathrm{dist}_D(s,t_{\mathrm{anchor}})$
    \ELSE
        \STATE $x \gets
        \mathrm{dist}_D(t_{\mathrm{anchor}},s)$
    \ENDIF
    \COMMENT{Number of active samples whose side changes}

    \IF{$x>\Delta$}
        \STATE \textbf{break}
    \ENDIF

    \STATE $(D_L,D_R) \gets \textsc{Partition}(D,s)$

    \IF{$D_L=\emptyset$}
        \STATE $\textsc{RaiseLowerBound}(\mathcal{S},c,s+1)$
        \STATE \textbf{break}
    \ENDIF

    \IF{$D_R=\emptyset$}
        \STATE $\textsc{LowerUpperBound}(\mathcal{S},c,s-1)$
        \STATE \textbf{break}
    \ENDIF

    \IF{$\mathrm{dir}=\textsc{Left}$}
        \STATE $\widehat P_L \gets L^{\mathrm{anchor}}$
        \STATE $\widehat P_R \gets U^{\mathrm{anchor}}+x$
        \COMMENT{Right child gained at most $x$ samples}
    \ELSE
        \STATE $\widehat P_L \gets L^{\mathrm{anchor}}+x$
        \STATE $\widehat P_R \gets U^{\mathrm{anchor}}$
        \COMMENT{Left child gained at most $x$ samples}
    \ENDIF

    \STATE $(G_L,G_R) \gets
    \begin{aligned}[t]
    \textsc{AddOrExtendSplit}(&G,D,s,d,\gamma,
    \varepsilon_{\mathrm{abs}},\\
                             &\mathcal{S},
    \widehat P_L,\widehat P_R)
    \end{aligned}$

    \IF{$G_L=\emptyset$ \textbf{ or } $G_R=\emptyset$}
        \STATE \textbf{break}
    \ENDIF

    \STATE $t_{\mathrm{anchor}} \gets s$
    \STATE $L^{\mathrm{anchor}} \gets
    \textsc{MinObjective}(G_L)$
    \STATE $U^{\mathrm{anchor}} \gets
    \textsc{MinObjective}(G_R)$
    \STATE $\Delta \gets
    \varepsilon_{\mathrm{abs}}-
    (L^{\mathrm{anchor}}+U^{\mathrm{anchor}})$
\ENDFOR

\end{algorithmic}
\end{algorithm}

To be precise, we assume that whenever
a subproblem is solved with budget at least its proxy objective, the resulting
subgraph has minimum objective no worse than the proxy objective. This is proven in \citet{heile2026}, though there are some subtleties introduced with our neighborhood pruning. The condition holds immediately when the proxy is optimal, either over all continuous thresholds or over a fixed
binarization. It also holds for any proxy satisfying the robustness conditions that the pruning assumes, so that the proxy split is not falsely pruned by neighboring-threshold bounds. In the anytime algorithm, we also track the proxy's selected split separately as a special threshold that is evaluated before similarity-based pruning; this also ensures that the proxy tree is recovered at a subproblem. When when not doing this, empirically, our near-perfect recall results suggest that this condition is satisfied.

The intuition is as follows: whenever the initial solves construct feasible solutions within the available budget, iterative budget refinement can propagate their minimum objectives and recover the same solutions that would have been obtained from the proxy. Conversely, if a proxy solution cannot be recovered, then its objective already exceeds the available budget, so it could not have contributed any trees under the original procedure.

\begin{theorem}[Neighboring-Threshold Pruning-Test]
\label{thm:neighboring-threshold-walk}
Assume that whenever a subproblem is solved with budget at least its proxy objective,
the resulting subgraph has minimum objective no worse than the proxy objective.
Then the neighboring-threshold walk in Algorithm \ref{alg:walk-from-solved-threshold} never sacrifices approximation quality relative
to explicitly applying the proxy-pruning test at each threshold. 
\end{theorem}

\begin{proof}
Consider a solved threshold $t$ with child subgraphs of minimum objectives
$L(t)$ and $U(t)$, and suppose
\[
    L(t)+U(t) \le \varepsilon_{\mathrm{abs}}.
\]
Let
\[
    \Delta = \varepsilon_{\mathrm{abs}} - L(t)-U(t)
\]
be the remaining slack. Now move to a neighboring threshold $s$ in the same
continuous feature group, and let
\[
    x = \mathrm{dist}_D(t,s)
\]
be the number of active samples whose branch assignment changes between $t$
and $s$.

Suppose first that $s$ lies to the right of $t$. Then the left child gains at
most $x$ samples, while the right child loses samples. Reusing the solved left
subtree from $t$ can increase its loss by at most $x$, and reusing the solved
right subtree cannot increase its loss because it is evaluated on fewer samples.
Thus the split at $s$ has a certified completion with child objectives bounded by
\[
    L(t)+x
    \qquad\text{and}\qquad
    U(t).
\]
If $x \le \Delta$, then
\[
    L(t)+x+U(t) \le \varepsilon_{\mathrm{abs}},
\]
so these bounds are sufficient child budgets for calling
\textsc{AddOrExtendSplit} at $s$ without first recomputing the proxy-pruning test.
The leftward case is symmetric.

By assumption, whenever a child is given budget at least its proxy objective, the
constructed subgraph has minimum objective no worse than that proxy objective.

Now we ask whether the child budgets assigned by
\textsc{AddOrExtendSplit} are at least the corresponding proxy objectives. If
they are, then the claim follows. This may not be immediate: finding a child
subgraph whose minimum objective is no worse than the proxy objective does not,
by itself, say that we have recovered every tree that the proxy-pruning test
would have recovered. The reason it is enough is that
\textsc{AddOrExtendSplit} performs iterative budget refinement. Once one child
is solved to objective no worse than its proxy objective, the next refinement
step subtracts this no-larger objective instead of the original walk bound
$L(t)+x$ or $U(t)$. Therefore the opposite child receives a budget at least
as large as it would have received under the explicit proxy-pruning test. Thus
the refinement process reaches the same budgets as the explicit proxy procedure,
possibly one iteration later. Since iterative budget refinement is run to
convergence, this delay does not change the recovered set.

Let
\[
    P_L(s) \qquad\text{and}\qquad P_R(s)
\]
denote the proxy objectives of the left and right child subproblems induced by
threshold $s$. We only need to consider the case where the proxy completion at
$s$ is within budget, namely
\[
    P_L(s)+P_R(s) \le \varepsilon_{\mathrm{abs}}.
\]
If this inequality fails, then explicitly applying the proxy-pruning test at $s$
would prune $s$, so the neighboring-threshold walk cannot miss any threshold
that the proxy-pruning test would have retained.

Thus suppose
\[
    P_L(s)+P_R(s) \le \varepsilon_{\mathrm{abs}}.
\]

Furthermore, we are only interested in the case where $P_L(S) \leq L(t) + x$ and $P_R(s) \leq U(t)$. If either one of these is not true, then our budgets are set for the chilren by subtracting a smaller number: thus the budgets are at least as large, and because we know $P_L(s)+P_R(s) \le \varepsilon_{\mathrm{abs}}$, we also know $ P_L(s)\le \varepsilon_{\mathrm{abs}}-P_R(s) $ and $ P_R(s)\le \varepsilon_{\mathrm{abs}}-P_L(s) $. That is, we know that if the true proxy check passes that the children subproblems are set with budget at least the proxy, so setting budgets any bigger will also be at least the proxy.

Therefore, we are left with handling the cases where $P_L(S) \leq L(t) + x$ and $P_R(s) \leq U(t)$.

We know $L(t) + x\le \varepsilon_{\mathrm{abs}}-U(t)$ because these upper bounds pass the pruning test. We assumed that $P_L(S) \leq L(t) + x$, so we have ensured it on this side. 

It remains to check the other side. We know $U(t) \le \varepsilon_{\mathrm{abs}}-(L(t) + x)$. By analogous reasoning (knowing $P_R(s) \leq U(t)$), we have the guarantee here.

Thus, we were fine if the true proxy test failed, fine if either one of the proxy objectives were bigger than our guesses, and if either one of our proxy objectives are smaller than our guesses. Therefore, no approximation quality is lost (and maybe some is gained). To state clearly, we can solve either side first, and whether we subtract a smaller or bigger number than the proxy, it is okay.

\end{proof}

\FloatBarrier

\section{Datasets}\label{app:datasets}

We summarize each dataset considered in our experiments, including the corresponding binary prediction problem.

Across all datasets, we discard observations with missing entries and represent categorical features using one-hot encoding.

We use 20 datasets in this work, most of which are drawn from the benchmark suite of \citet{heile2026}. We exclude several datasets from that suite because they contain only binary features, such as Droid and Monk2, and add several additional datasets to reach our goal of 20 datasets. 
\paragraph{\textbf{Abalone \citep{openml_abalone_44956}} (\textit{4{,}177 samples})}
Predict whether an abalone is male based on its physical measurements. Exhaustive binarization of this dataset creates 6039 columns (not counting the label).

\paragraph{\textbf{Adult \citep{adult_2}} (\textit{48{,}842 samples})}
Predict whether an individual earns more than \$50{,}000 per year based on demographic and occupational attributes. Exhaustive binarization of this dataset creates 27237 columns.

\paragraph{\textbf{Aging \citep{national_poll_on_healthy_aging_(npha)_936, malani2017npha}} (\textit{714 samples})}
Predict whether an individual has visited at least two doctors. Exhaustive binarization of this dataset creates 35 columns.

\paragraph{\textbf{Bank \citep{moro2014ada, bank_marketing_222}} (\textit{45{,}211 samples})}
Predict whether a client subscribes to a term deposit following a marketing campaign. Exhaustive binarization of this dataset creates 9530 columns.

\paragraph{\textbf{Bike \citep{FanaeeT2013EventLC, bike_sharing_275}} (\textit{17{,}379 samples})}
Predict whether bike rental demand exceeds the median. Exhaustive binarization of this dataset creates 279 columns.

\paragraph{\textbf{Churn \citep{erickson2025tabarenalivingbenchmarkmachine, marcoulides2005data}} (\textit{5{,}000 samples})}
Predict whether a customer will churn. Exhaustive binarization of this dataset creates 16400 columns.

\paragraph{\textbf{Compas \citep{bao2021compaslicated}} (\textit{4{,}966 samples})}
Predict whether a defendant will recidivate within two years. Exhaustive binarization of this dataset creates 120 columns.

\paragraph{\textbf{Coupon \citep{in-vehicle_coupon_recommendation_603}} (\textit{108 samples})}
Predict whether an individual accepts a recommended coupon. Exhaustive binarization of this dataset creates 64 columns.

\paragraph{\textbf{Credit \citep{default_of_credit_card_clients_350, Yeh2009TheCO}} (\textit{30{,}000 samples})}
Predict whether a client will default on their credit card payment in the following month. Exhaustive binarization of this dataset creates 174581 columns.

\paragraph{\textbf{Diabetes \citep{burrows2017esrd}} (\textit{253{,}680 samples})}
Predict whether an individual is diabetic. Exhaustive binarization of this dataset creates 185 columns.

\paragraph{\textbf{Diamonds \citep{openml_diamonds_42225}} (\textit{53{,}940 samples})}
Predict whether a diamond belongs to cut category 2 based on its physical attributes. Exhaustive binarization of this dataset creates 2074 columns.

\paragraph{\textbf{Helena \citep{openml_helena_41169, automlchallenges}} (\textit{65{,}196 samples})}
Predict whether an instance belongs to the most frequent class. Exhaustive binarization of this dataset creates 1540012 (over 1 million) columns.

\paragraph{\textbf{Heloc \citep{fico2018heloc}} (\textit{2{,}502 samples})}
Predict whether an individual is high- or low-risk for a home equity line of credit. Exhaustive binarization of this dataset creates 1505 columns.

\paragraph{\textbf{Jasmine \citep{openml_jasmine_41143, automlchallenges}} (\textit{2{,}984 samples})}
Binary classification using the provided target column. Exhaustive binarization of this dataset creates 2457 columns.

\paragraph{\textbf{Pol \citep{openml_pol_44082}} (\textit{10{,}082 samples})}Binary classification using the provided target column. Exhaustive binarization of this dataset creates 2126 columns.

\paragraph{\textbf{Rl \citep{openml_rl_43949}} (\textit{4{,}970 samples})}Binary classification using the provided target column. Exhaustive binarization of this dataset creates 1313 columns.

\paragraph{\textbf{Shopping \citep{online_shoppers_purchasing_intention_dataset_468, Sakar2018RealtimePO}} (\textit{12{,}330 samples})}
Predict whether an online shopping session ends in a purchase. Exhaustive binarization of this dataset creates 23909 columns.

\paragraph{\textbf{Spambase \citep{spambase_94}} (\textit{4{,}601 samples})}
Predict whether an email is spam. Exhaustive binarization of this dataset creates 15037 columns.

\paragraph{\textbf{Student \citep{student_performance_320, Cortez2008UsingDM}} (\textit{649 samples})}
Predict whether a student passes a course (final grade $\geq 10$). Exhaustive binarization of this dataset creates 114 columns.

\paragraph{\textbf{Wine \citep{openml_wine_47041}} (\textit{6{,}497 samples})}
Predict whether wine quality is at least 7. Exhaustive binarization of this dataset creates 2640 columns.

\section{Experiments}
\label{app:experiments}

\subsection{Computational Resources}
\label{app:computational-resources}

All experiments were conducted on an institutional computing cluster. Each run was executed on a single compute node equipped with an AMD EPYC 9554 processor (2.75 GHz, 64 physical cores) and was restricted to a single CPU core. With one exception, each algorithm was given 128 GB of memory and a 100-hour timeout to compute a Rashomon set for one bootstrap of one dataset for one set of parameters. We instead limited our anytime algorithms to 32 GB of memory and a 24-hour timeout, both to reduce resource usage and to evaluate their performance in a more resource-constrained setting.

When restricting a proxy algorithm, or one of its subroutines, to a fixed binarization, we select candidate thresholds using the ThresholdGuessing procedure of \citet{gosdt_guesses}. We train a gradient-boosted tree ensemble with 150 depth-2 estimators and random seed 0, extract the thresholds used by the ensemble, and apply backward elimination to remove thresholds without reducing the ensemble's training accuracy.

\subsection{Timing, Memory, and Recall}
\label{app:timing-memory-recall}

\begin{table*}[!t]
\centering
\scriptsize
\setlength{\tabcolsep}{3pt}
\caption{
$\lambda=0.02, \varepsilon=0.03, d=5$. Runtime and peak memory for the exhaustive threshold (fully continuous) setting. Our methods are the four variants shown in the table, each combining our continuous-feature Rashomon set algorithm (ArborEnum) with a different proxy algorithm. 
Time is reported in seconds and peak RSS is reported in MB. Mean $\pm$ std across 3 bootstraps.
Entries marked ``--'' correspond to partial or unfinished runs with a 100hr timeout and 128GB memory limit.
}
\label{tab:runtime-memory-exhaustive}
\resizebox{\textwidth}{!}{%
\begin{tabular}{lrrrrrrrr}
\toprule
& \multicolumn{2}{c}{LicketySPLIT over a binarization}
& \multicolumn{2}{c}{LicketySNIP (fully continuous)}
& \multicolumn{2}{c}{LicketySNIP (with greedy subroutine using a binarization)}
& \multicolumn{2}{c}{Optimal} \\
\cmidrule(lr){2-3}
\cmidrule(lr){4-5}
\cmidrule(lr){6-7}
\cmidrule(lr){8-9}
Dataset
& Time & Mem.
& Time & Mem.
& Time & Mem.
& Time & Mem. \\
\midrule
Abalone    & $264.26 \pm 401.02$ & $1823.5 \pm 2412.8$ & $2810.60 \pm 4187.00$ & $1259.3 \pm 1544.1$ & $205.14 \pm 311.57$ & $1251.0 \pm 1528.6$ & $55174.03 \pm 47866.54$ & $53538.0 \pm 44185.3$ \\
Adult      & $239.57 \pm 26.37$ & $837.4 \pm 17.4$ & $12131.75 \pm 371.85$ & $891.1 \pm 12.5$ & $292.12 \pm 16.25$ & $851.0 \pm 10.9$ & $141689.07 \pm 13770.50$ & $28629.1 \pm 960.4$ \\
Aging      & $0.08 \pm 0.00$ & $249.1 \pm 0.4$ & $0.08 \pm 0.00$ & $248.0 \pm 0.8$ & $0.08 \pm 0.00$ & $250.6 \pm 0.6$ & $0.24 \pm 0.00$ & $253.9 \pm 2.4$ \\
Bank       & $14.64 \pm 0.99$ & $376.0 \pm 1.6$ & $212.52 \pm 5.41$ & $358.3 \pm 0.7$ & $6.43 \pm 0.11$ & $366.1 \pm 0.5$ & $20978.46 \pm 2133.15$ & $10577.0 \pm 826.5$ \\
Bike       & $7.58 \pm 2.25$ & $298.8 \pm 8.7$ & $27.05 \pm 4.96$ & $299.6 \pm 5.8$ & $7.95 \pm 1.91$ & $299.5 \pm 7.2$ & $1379.90 \pm 56.81$ & $3044.2 \pm 89.8$ \\
Churn      & $7.18 \pm 5.99$ & $335.2 \pm 49.2$ & $643.87 \pm 555.53$ & $337.4 \pm 58.8$ & $9.42 \pm 9.13$ & $338.4 \pm 57.6$ & $257505.10 \pm 25437.01$ & $129432.5 \pm 1527.8$ \\
Compas     & $0.45 \pm 0.15$ & $255.1 \pm 3.1$ & $0.50 \pm 0.19$ & $251.0 \pm 1.1$ & $0.35 \pm 0.12$ & $254.3 \pm 1.1$ & $4.76 \pm 0.15$ & $276.9 \pm 2.1$ \\
Coupon     & $0.25 \pm 0.24$ & $272.1 \pm 34.6$ & $0.15 \pm 0.04$ & $249.7 \pm 2.7$ & $0.12 \pm 0.04$ & $252.0 \pm 3.1$ & $0.20 \pm 0.01$ & $250.5 \pm 1.1$ \\
Credit     & $74.43 \pm 42.76$ & $1231.4 \pm 30.2$ & $24849.80 \pm 619.24$ & $1217.0 \pm 3.1$ & $65.52 \pm 22.19$ & $1218.8 \pm 2.4$ & -- & -- \\
Diabetes   & $6.86 \pm 0.66$ & $379.5 \pm 0.7$ & $8.55 \pm 0.22$ & $353.4 \pm 1.3$ & $3.77 \pm 0.17$ & $378.5 \pm 1.3$ & $1373.41 \pm 65.57$ & $808.1 \pm 54.1$ \\
Diamonds   & $46.39 \pm 2.67$ & $355.0 \pm 2.2$ & $313.35 \pm 23.79$ & $341.6 \pm 3.9$ & $39.10 \pm 3.25$ & $346.0 \pm 4.2$ & $16615.59 \pm 1023.55$ & $6282.0 \pm 131.1$ \\
Helena     & $213.06 \pm 16.55$ & $15908.0 \pm 21.5$ & $286987.90 \pm 12346.63$ & $15861.3 \pm 16.8$ & $185.50 \pm 11.11$ & $15879.9 \pm 17.1$ & -- & -- \\
Heloc      & $31.57 \pm 28.38$ & $537.2 \pm 227.4$ & $273.53 \pm 133.41$ & $548.9 \pm 162.1$ & $34.89 \pm 26.12$ & $565.2 \pm 184.1$ & -- & -- \\
Jasmine    & $44.03 \pm 48.44$ & $646.8 \pm 424.1$ & $405.78 \pm 427.03$ & $704.8 \pm 483.7$ & $61.10 \pm 70.61$ & $751.2 \pm 550.0$ & -- & -- \\
Pol        & $283.44 \pm 153.15$ & $1945.8 \pm 882.3$ & $2155.51 \pm 817.12$ & $2094.3 \pm 727.8$ & $383.98 \pm 190.49$ & $2014.6 \pm 768.6$ & $33700.35 \pm 2870.31$ & $35074.4 \pm 2411.6$ \\
RL         & $123.59 \pm 69.11$ & $1363.4 \pm 453.8$ & $666.06 \pm 74.75$ & $1198.2 \pm 166.1$ & $117.10 \pm 33.85$ & $1248.4 \pm 214.3$ & $32255.91 \pm 6094.71$ & $45360.1 \pm 6446.0$ \\
Shopping   & $901.11 \pm 114.47$ & $3037.6 \pm 210.3$ & $25547.42 \pm 2593.81$ & $1563.7 \pm 140.9$ & $489.25 \pm 57.80$ & $1596.9 \pm 97.5$ & -- & -- \\
Spambase   & $9897.26 \pm 2575.68$ & $64879.6 \pm 12411.8$ & -- & -- & $25248.74 \pm 8522.68$ & $93511.3 \pm 31495.6$ & -- & -- \\
Student    & $0.16 \pm 0.02$ & $253.3 \pm 0.6$ & $0.26 \pm 0.05$ & $251.5 \pm 1.5$ & $0.17 \pm 0.02$ & $253.1 \pm 2.0$ & $43.22 \pm 2.13$ & $871.6 \pm 23.4$ \\
Wine       & $17.86 \pm 8.30$ & $348.9 \pm 44.0$ & $113.01 \pm 79.44$ & $304.2 \pm 34.0$ & $10.64 \pm 7.08$ & $306.0 \pm 31.9$ & $81042.78 \pm 8688.47$ & $78053.3 \pm 10711.4$ \\
\bottomrule
\end{tabular}%
}
\end{table*}

\begin{table*}[!t]
\centering
\scriptsize
\setlength{\tabcolsep}{5pt}
\caption{
$\lambda=0.02, \varepsilon=0.03, d=5$. Recall for the exhaustive threshold (fully continuous) setting. mean $\pm$ std is reported.
Recall is reported relative to the best method that finished. Formally, we calculate the guessed Rashomon bound by taking the minimum objective any method found, counting the number of trees each method found within that bound, and scoring a method by the number of trees as a proportion of the best.
Entries marked ``--'' correspond to partial or unfinished runs. Our methods are the four variants shown in the table, each combining our continuous-feature Rashomon set algorithm with a different proxy algorithm.
}
\label{tab:recall-exhaustive}
\resizebox{\textwidth}{!}{%
\begin{tabular}{lrrrr}
\toprule
Dataset
& LicketySPLIT over a binarization
& LicketySNIP (fully continuous)
& LicketySNIP with a binarized greedy subroutine)
& Optimal \\
\midrule
Abalone    & $1.000 \pm 0.000$ & $1.000 \pm 0.000$ & $1.000 \pm 0.000$ & $1.000 \pm 0.000$ \\
Adult      & $1.000 \pm 0.000$ & $1.000 \pm 0.000$ & $1.000 \pm 0.000$ & $1.000 \pm 0.000$ \\
Aging      & $1.000 \pm 0.000$ & $1.000 \pm 0.000$ & $1.000 \pm 0.000$ & $1.000 \pm 0.000$ \\
Bank       & $1.000 \pm 0.000$ & $1.000 \pm 0.000$ & $1.000 \pm 0.000$ & $1.000 \pm 0.000$ \\
Bike       & $1.000 \pm 0.000$ & $1.000 \pm 0.000$ & $1.000 \pm 0.000$ & $1.000 \pm 0.000$ \\
Churn      & $1.000 \pm 0.000$ & $1.000 \pm 0.000$ & $1.000 \pm 0.000$ & $1.000 \pm 0.000$ \\
Compas     & $1.000 \pm 0.000$ & $1.000 \pm 0.000$ & $1.000 \pm 0.000$ & $1.000 \pm 0.000$ \\
Coupon     & $0.667 \pm 0.577$ & $1.000 \pm 0.000$ & $1.000 \pm 0.000$ & $1.000 \pm 0.000$ \\
Credit     & $1.000 \pm 0.000$ & $1.000 \pm 0.000$ & $1.000 \pm 0.000$ & -- \\
Diabetes   & $1.000 \pm 0.000$ & $1.000 \pm 0.000$ & $1.000 \pm 0.000$ & $1.000 \pm 0.000$ \\
Diamonds   & $1.000 \pm 0.000$ & $1.000 \pm 0.000$ & $1.000 \pm 0.000$ & $1.000 \pm 0.000$ \\
Helena     & $1.000 \pm 0.000$ & $1.000 \pm 0.000$ & $1.000 \pm 0.000$ & -- \\
Heloc      & $1.000 \pm 0.000$ & $1.000 \pm 0.000$ & $1.000 \pm 0.000$ & -- \\
Jasmine    & $0.931 \pm 0.120$ & $1.000 \pm 0.000$ & $1.000 \pm 0.000$ & -- \\
Pol        & $1.000 \pm 0.001$ & $1.000 \pm 0.000$ & $1.000 \pm 0.000$ & $1.000 \pm 0.000$ \\
RL         & $0.993 \pm 0.012$ & $1.000 \pm 0.000$ & $1.000 \pm 0.000$ & $1.000 \pm 0.000$ \\
Shopping   & $1.000 \pm 0.000$ & $1.000 \pm 0.000$ & $1.000 \pm 0.000$ & -- \\
Spambase   & $0.999 \pm 0.001$ & -- & $1.000 \pm 0.000$ & -- \\
Student    & $1.000 \pm 0.000$ & $1.000 \pm 0.000$ & $1.000 \pm 0.000$ & $1.000 \pm 0.000$ \\
Wine       & $1.000 \pm 0.000$ & $1.000 \pm 0.000$ & $1.000 \pm 0.000$ & $1.000 \pm 0.000$ \\
\bottomrule
\end{tabular}%
}
\end{table*}

We use one-sided Wilcoxon signed-rank tests on the results in Table~\ref{tab:runtime-exhaustive-8}, treating the mean runtime across the three bootstraps as one paired observation per dataset. Each comparison includes only datasets for which both methods completed. Because we perform two runtime comparisons, we control the family-wise error rate at $\alpha=0.05$ using Holm's correction. On the six datasets where both our optimal method and SORTD completed, our optimal method was significantly faster ($W=0$, Holm-adjusted $p=0.0156$). On the 14 datasets where both SNIP+GR and our optimal method completed, SNIP+GR was also significantly faster ($W=0$, Holm-adjusted $p=1.22\times10^{-4}$). These complete-case comparisons are conservative with respect to the faster methods: for instance, SORTD is not penalized for running out of memory or time.
\begin{table*}[!t]
\centering
\scriptsize
\setlength{\tabcolsep}{3pt}
\caption{
$\lambda=0.02, \varepsilon=0.03, d=5$. Runtime for the exhaustive-threshold (fully continuous) setting.
Time is reported in seconds as mean $\pm$ standard deviation across 3 bootstraps.
Entries marked ``--'' correspond to partial or unfinished runs under a 100-hour timeout and 128GB memory limit.
The best completed runtime in each row is bolded and the second-best is underlined. Our methods are the first four shown in the table, each combining our continuous-feature Rashomon set algorithm with a different proxy algorithm.
}
\label{tab:runtime-exhaustive-8}
\resizebox{\textwidth}{!}{%
\begin{tabular}{lrrrrrrrr}
\toprule
Dataset
& LSR
& SNIP
& SNIP+GR
& Optimal
& PRAXIS
& SORTD
& TreeFARMS
& RESPLIT \\
\midrule
Abalone & \underline{$264.26 \pm 401.02$} & $2810.60 \pm 4187.00$ & $\boldsymbol{205.14 \pm 311.57}$ & $55174.03 \pm 47866.54$ & -- & -- & -- & -- \\
Adult & $\boldsymbol{239.57 \pm 26.37}$ & $12131.75 \pm 371.85$ & \underline{$292.12 \pm 16.25$} & $141689.07 \pm 13770.50$ & -- & -- & -- & -- \\
Aging & \underline{$0.08 \pm 0.00$} & \underline{$0.08 \pm 0.00$} & \underline{$0.08 \pm 0.00$} & $0.24 \pm 0.00$ & $\boldsymbol{0.06 \pm 0.00}$ & $1.60 \pm 0.19$ & $1116.43 \pm 449.69$ & $0.52 \pm 0.02$ \\
Bank & \underline{$14.64 \pm 0.99$} & $212.52 \pm 5.41$ & $\boldsymbol{6.43 \pm 0.11}$ & $20978.46 \pm 2133.15$ & -- & -- & -- & -- \\
Bike & $\boldsymbol{7.58 \pm 2.25}$ & $27.05 \pm 4.96$ & \underline{$7.95 \pm 1.91$} & $1379.90 \pm 56.81$ & $116.84 \pm 24.95$ & $86351.11 \pm 15522.58$ & -- & $3813.08 \pm 305.80$ \\
Churn & $\boldsymbol{7.18 \pm 5.99}$ & $643.87 \pm 555.53$ & \underline{$9.42 \pm 9.13$} & $257505.10 \pm 25437.01$ & -- & -- & -- & -- \\
Compas & \underline{$0.45 \pm 0.15$} & $0.50 \pm 0.19$ & $\boldsymbol{0.35 \pm 0.12}$ & $4.76 \pm 0.15$ & $1.35 \pm 0.76$ & $239.54 \pm 10.20$ & -- & $88.90 \pm 48.53$ \\
Coupon & $0.25 \pm 0.24$ & \underline{$0.15 \pm 0.04$} & $\boldsymbol{0.12 \pm 0.04}$ & $0.20 \pm 0.01$ & $0.23 \pm 0.22$ & $1.16 \pm 0.09$ & $1006.25 \pm 323.06$ & -- \\
Credit & \underline{$74.43 \pm 42.76$} & $24849.80 \pm 619.24$ & $\boldsymbol{65.52 \pm 22.19}$ & -- & -- & -- & -- & -- \\
Diabetes & \underline{$6.86 \pm 0.66$} & $8.55 \pm 0.22$ & $\boldsymbol{3.77 \pm 0.17}$ & $1373.41 \pm 65.57$ & $64.16 \pm 0.66$ & $232015.98 \pm 24149.42$ & -- & $7507.98 \pm 232.33$ \\
Diamonds & \underline{$46.39 \pm 2.67$} & $313.35 \pm 23.79$ & $\boldsymbol{39.10 \pm 3.25}$ & $16615.59 \pm 1023.55$ & $77628.35 \pm 6710.46$ & -- & -- & -- \\
Helena & \underline{$213.06 \pm 16.55$} & $286987.90 \pm 12346.63$ & $\boldsymbol{185.50 \pm 11.11}$ & -- & -- & -- & -- & -- \\
Heloc & $\boldsymbol{31.57 \pm 28.38}$ & $273.53 \pm 133.41$ & \underline{$34.89 \pm 26.12$} & -- & $2628.72 \pm 1410.49$ & -- & -- & $87593.52 \pm 30122.99$ \\
Jasmine & $\boldsymbol{44.03 \pm 48.44}$ & $405.78 \pm 427.03$ & \underline{$61.10 \pm 70.61$} & -- & $5118.09 \pm 4269.98$ & -- & -- & $222714.18 \pm 111739.02$ \\
Pol & $\boldsymbol{283.44 \pm 153.15}$ & $2155.51 \pm 817.12$ & \underline{$383.98 \pm 190.49$} & $33700.35 \pm 2870.31$ & -- & -- & -- & -- \\
Rl & \underline{$123.59 \pm 69.11$} & $666.06 \pm 74.75$ & $\boldsymbol{117.10 \pm 33.85}$ & $32255.91 \pm 6094.71$ & $27240.65 \pm 6758.17$ & -- & -- & -- \\
Shopping & \underline{$901.11 \pm 114.47$} & $25547.42 \pm 2593.81$ & $\boldsymbol{489.25 \pm 57.80}$ & -- & -- & -- & -- & -- \\
Spambase & $\boldsymbol{9897.26 \pm 2575.68}$ & -- & \underline{$25248.74 \pm 8522.68$} & -- & -- & -- & -- & -- \\
Student & $\boldsymbol{0.16 \pm 0.02}$ & $0.26 \pm 0.05$ & \underline{$0.17 \pm 0.02$} & $43.22 \pm 2.13$ & $0.22 \pm 0.03$ & $145.98 \pm 9.97$ & -- & $16.92 \pm 6.44$ \\
Wine & \underline{$17.86 \pm 8.30$} & $113.01 \pm 79.44$ & $\boldsymbol{10.64 \pm 7.08}$ & $81042.78 \pm 8688.47$ & $7189.26 \pm 4812.77$ & -- & -- & -- \\
\bottomrule
\end{tabular}%
}
\footnotesize{
LSR = LicketySPLIT over a binarization; SNIP = LicketySNIP (fully continuous); SNIP+GR = LicketySNIP with its greedy subroutine restricted to a binarization. 
}
\end{table*}

\begin{table*}[!t]
\centering
\scriptsize
\setlength{\tabcolsep}{3pt}
\caption{
$\lambda=0.02, \varepsilon=0.03, d=5$. Peak memory for the exhaustive-threshold (fully continuous) setting.
Peak RSS is reported in MB as mean $\pm$ standard deviation across 3 bootstraps.
Entries marked ``--'' correspond to partial or unfinished runs under a 100-hour timeout and 128GB memory limit.
The lowest completed memory usage in each row is bolded and the second-lowest is underlined. Our methods are the first four shown in the table, each combining our continuous-feature Rashomon set algorithm with a different proxy algorithm.
}
\label{tab:memory-exhaustive-8}
\resizebox{\textwidth}{!}{%
\begin{tabular}{lrrrrrrrr}
\toprule
Dataset
& LSR
& SNIP
& SNIP+GR
& Optimal
& PRAXIS
& SORTD
& TreeFARMS
& RESPLIT \\
\midrule
Abalone & $1823.5 \pm 2412.8$ & \underline{$1259.3 \pm 1544.1$} & $\boldsymbol{1251.0 \pm 1528.6}$ & $53538.0 \pm 44185.3$ & -- & -- & -- & -- \\
Adult & $\boldsymbol{837.4 \pm 17.4}$ & $891.1 \pm 12.5$ & \underline{$851.0 \pm 10.9$} & $28629.1 \pm 960.4$ & -- & -- & -- & -- \\
Aging & $249.1 \pm 0.4$ & $248.0 \pm 0.8$ & $250.6 \pm 0.6$ & $253.9 \pm 2.4$ & \underline{$223.1 \pm 1.5$} & $\boldsymbol{196.4 \pm 1.3}$ & $5350.2 \pm 1187.7$ & $233.9 \pm 1.8$ \\
Bank & $376.0 \pm 1.6$ & $\boldsymbol{358.3 \pm 0.7}$ & \underline{$366.1 \pm 0.5$} & $10577.0 \pm 826.5$ & -- & -- & -- & -- \\
Bike & $\boldsymbol{298.8 \pm 8.7}$ & $299.6 \pm 5.8$ & \underline{$299.5 \pm 7.2$} & $3044.2 \pm 89.8$ & $356.5 \pm 22.4$ & $12139.7 \pm 213.2$ & -- & $2824.3 \pm 36.7$ \\
Churn & $\boldsymbol{335.2 \pm 49.2}$ & \underline{$337.4 \pm 58.8$} & $338.4 \pm 57.6$ & $129432.5 \pm 1527.8$ & -- & -- & -- & -- \\
Compas & $255.1 \pm 3.1$ & \underline{$251.0 \pm 1.1$} & $254.3 \pm 1.1$ & $276.9 \pm 2.1$ & $\boldsymbol{233.4 \pm 3.6}$ & $434.0 \pm 11.8$ & -- & $365.6 \pm 0.8$ \\
Coupon & $272.1 \pm 34.6$ & \underline{$249.7 \pm 2.7$} & $252.0 \pm 3.1$ & $250.5 \pm 1.1$ & $284.2 \pm 96.3$ & $\boldsymbol{197.4 \pm 0.9}$ & $3084.8 \pm 318.2$ & -- \\
Credit & $1231.4 \pm 30.2$ & $\boldsymbol{1217.0 \pm 3.1}$ & \underline{$1218.8 \pm 2.4$} & -- & -- & -- & -- & -- \\
Diabetes & $379.5 \pm 0.7$ & $\boldsymbol{353.4 \pm 1.3}$ & $378.5 \pm 1.3$ & $808.1 \pm 54.1$ & \underline{$372.5 \pm 8.7$} & $24278.8 \pm 434.4$ & -- & $2748.5 \pm 9.7$ \\
Diamonds & $355.0 \pm 2.2$ & $\boldsymbol{341.6 \pm 3.9}$ & \underline{$346.0 \pm 4.2$} & $6282.0 \pm 131.1$ & $4424.3 \pm 595.3$ & -- & -- & -- \\
Helena & $15908.0 \pm 21.5$ & $\boldsymbol{15861.3 \pm 16.8}$ & \underline{$15879.9 \pm 17.1$} & -- & -- & -- & -- & -- \\
Heloc & $\boldsymbol{537.2 \pm 227.4}$ & \underline{$548.9 \pm 162.1$} & $565.2 \pm 184.1$ & -- & $3213.1 \pm 1695.3$ & -- & -- & $20853.2 \pm 2931.5$ \\
Jasmine & $\boldsymbol{646.8 \pm 424.1}$ & \underline{$704.8 \pm 483.7$} & $751.2 \pm 550.0$ & -- & $4275.8 \pm 3840.1$ & -- & -- & $32739.5 \pm 1899.4$ \\
Pol & $\boldsymbol{1945.8 \pm 882.3}$ & $2094.3 \pm 727.8$ & \underline{$2014.6 \pm 768.6$} & $35074.4 \pm 2411.6$ & -- & -- & -- & -- \\
Rl & $1363.4 \pm 453.8$ & $\boldsymbol{1198.2 \pm 166.1}$ & \underline{$1248.4 \pm 214.3$} & $45360.1 \pm 6446.0$ & $21879.3 \pm 4285.1$ & -- & -- & -- \\
Shopping & $3037.6 \pm 210.3$ & $\boldsymbol{1563.7 \pm 140.9}$ & \underline{$1596.9 \pm 97.5$} & -- & -- & -- & -- & -- \\
Spambase & $\boldsymbol{64879.6 \pm 12411.8}$ & -- & \underline{$93511.3 \pm 31495.6$} & -- & -- & -- & -- & -- \\
Student & $253.3 \pm 0.6$ & \underline{$251.5 \pm 1.5$} & $253.1 \pm 2.0$ & $871.6 \pm 23.4$ & $\boldsymbol{226.1 \pm 1.6}$ & $526.4 \pm 12.3$ & -- & $287.1 \pm 5.8$ \\
Wine & $348.9 \pm 44.0$ & $\boldsymbol{304.2 \pm 34.0}$ & \underline{$306.0 \pm 31.9$} & $78053.3 \pm 10711.4$ & $2361.7 \pm 1519.2$ & -- & -- & -- \\
\bottomrule
\end{tabular}%
}
\footnotesize{
LSR = LicketySPLIT over a binarization; SNIP = LicketySNIP (fully continuous); SNIP+GR = LicketySNIP with its greedy subroutine restricted to a binarization. 
}
\end{table*}

\begin{table*}[!t]
\centering
\scriptsize
\setlength{\tabcolsep}{3pt}
\caption{
$\lambda=0.02, \varepsilon=0.03, d=5$. Runtime with at most 100 thresholds per continuous feature.
Time is reported in seconds as mean $\pm$ standard deviation across 3 bootstraps.
Entries marked ``--'' correspond to partial or unfinished runs under a 100-hour timeout and 128GB memory limit.
The best completed runtime in each row is bolded and the second-best is underlined. Our methods are the first four shown in the table, each combining our continuous-feature Rashomon set algorithm with a different proxy algorithm.
}
\label{tab:runtime-100-8}
\resizebox{\textwidth}{!}{%
\begin{tabular}{lrrrrrrrr}
\toprule
Dataset
& LSR
& SNIP
& SNIP+GR
& Optimal
& PRAXIS
& SORTD
& TreeFARMS
& RESPLIT \\
\midrule
Abalone & \underline{$24.12 \pm 32.55$} & $66.27 \pm 90.30$ & $\boldsymbol{17.93 \pm 24.48}$ & $7837.86 \pm 2316.53$ & $468.03 \pm 675.36$ & -- & -- & $23645.29 \pm 31027.96$ \\
Adult & $\boldsymbol{40.30 \pm 8.73}$ & $200.64 \pm 33.19$ & \underline{$51.53 \pm 8.39$} & $28071.59 \pm 1892.54$ & $388.16 \pm 60.47$ & $308020.36 \pm 14767.12$ & -- & $10647.57 \pm 390.47$ \\
Aging & $0.09 \pm 0.01$ & \underline{$0.08 \pm 0.00$} & \underline{$0.08 \pm 0.00$} & $0.25 \pm 0.01$ & $\boldsymbol{0.06 \pm 0.00}$ & $1.19 \pm 0.10$ & $856.66 \pm 220.11$ & $0.53 \pm 0.08$ \\
Bank & \underline{$6.65 \pm 0.63$} & $12.98 \pm 0.40$ & $\boldsymbol{3.64 \pm 0.12}$ & $5825.58 \pm 209.41$ & $77.35 \pm 1.61$ & -- & -- & $7755.72 \pm 86.31$ \\
Bike & $\boldsymbol{7.59 \pm 2.25}$ & $27.08 \pm 4.68$ & \underline{$8.08 \pm 1.95$} & $1382.71 \pm 60.02$ & $63.03 \pm 11.20$ & $25437.95 \pm 1987.45$ & -- & $2119.83 \pm 135.74$ \\
Churn & $\boldsymbol{2.31 \pm 0.37}$ & $35.93 \pm 4.98$ & \underline{$2.63 \pm 0.47$} & $54132.99 \pm 5265.63$ & $508.46 \pm 39.74$ & -- & -- & $54643.30 \pm 5983.82$ \\
Compas & \underline{$0.45 \pm 0.14$} & $0.50 \pm 0.20$ & $\boldsymbol{0.35 \pm 0.12}$ & $4.70 \pm 0.15$ & $0.63 \pm 0.31$ & $45.01 \pm 1.63$ & -- & $58.15 \pm 47.29$ \\
Coupon & $0.28 \pm 0.24$ & \underline{$0.15 \pm 0.03$} & $\boldsymbol{0.12 \pm 0.04}$ & $0.19 \pm 0.01$ & $3.10 \pm 4.74$ & $1.38 \pm 0.09$ & -- & -- \\
Credit & \underline{$33.92 \pm 15.53$} & $246.79 \pm 3.23$ & $\boldsymbol{32.31 \pm 8.89}$ & -- & $2814.93 \pm 100.25$ & -- & -- & $268946.24 \pm 4316.33$ \\
Diabetes & \underline{$6.89 \pm 0.63$} & $8.53 \pm 0.22$ & $\boldsymbol{3.79 \pm 0.15}$ & $1368.23 \pm 61.30$ & $10.94 \pm 0.47$ & $23158.60 \pm 1157.70$ & -- & $1152.02 \pm 43.55$ \\
Diamonds & $\boldsymbol{8.11 \pm 0.31}$ & $34.78 \pm 2.90$ & \underline{$9.05 \pm 0.28$} & $4232.98 \pm 182.45$ & $320.47 \pm 31.37$ & -- & -- & $30681.51 \pm 1680.20$ \\
Helena & $\boldsymbol{58.98 \pm 4.05}$ & $952.19 \pm 50.87$ & \underline{$64.08 \pm 4.62$} & -- & $35213.33 \pm 1016.12$ & -- & -- & -- \\
Heloc & $\boldsymbol{28.62 \pm 25.81}$ & $221.98 \pm 109.88$ & \underline{$34.07 \pm 25.68$} & -- & $794.29 \pm 429.05$ & -- & -- & $20140.03 \pm 7215.81$ \\
Jasmine & $\boldsymbol{20.31 \pm 11.73}$ & $149.04 \pm 95.28$ & \underline{$27.55 \pm 15.79$} & $77700.00 \pm 3568.15$ & $311.73 \pm 166.92$ & -- & -- & $13199.46 \pm 8250.98$ \\
Pol & $\boldsymbol{169.89 \pm 90.07}$ & $1243.43 \pm 478.57$ & \underline{$301.10 \pm 221.54$} & $22884.28 \pm 3738.72$ & $3032.13 \pm 1062.79$ & -- & -- & $20700.59 \pm 2163.83$ \\
Rl & $\boldsymbol{33.83 \pm 18.94}$ & $122.82 \pm 28.90$ & \underline{$34.20 \pm 10.26$} & $9259.82 \pm 1195.01$ & $619.29 \pm 189.76$ & $160810.85 \pm 6355.71$ & -- & $25441.88 \pm 10029.45$ \\
Shopping & \underline{$13.42 \pm 0.96$} & $30.13 \pm 1.55$ & $\boldsymbol{8.42 \pm 0.42}$ & $26436.73 \pm 5736.45$ & $271.09 \pm 19.24$ & -- & -- & $11775.36 \pm 323.10$ \\
Spambase & $\boldsymbol{195.95 \pm 64.78}$ & $3962.34 \pm 852.79$ & \underline{$389.91 \pm 93.21$} & -- & $40131.08 \pm 9168.36$ & -- & -- & -- \\
Student & $\boldsymbol{0.16 \pm 0.01}$ & $0.27 \pm 0.05$ & \underline{$0.17 \pm 0.02$} & $41.09 \pm 3.43$ & $0.23 \pm 0.03$ & $113.73 \pm 15.50$ & -- & $19.52 \pm 9.48$ \\
Wine & \underline{$15.06 \pm 7.76$} & $48.69 \pm 36.00$ & $\boldsymbol{9.36 \pm 6.36}$ & $49434.96 \pm 4044.88$ & $335.38 \pm 197.14$ & -- & -- & $17320.45 \pm 5949.11$ \\
\bottomrule
\end{tabular}%
}
\footnotesize{
LSR = LicketySPLIT over a binarization; SNIP = LicketySNIP (fully continuous); SNIP+GR = LicketySNIP with its greedy subroutine restricted to a binarization. 
}
\end{table*}

\begin{table*}[!t]
\centering
\scriptsize
\setlength{\tabcolsep}{3pt}
\caption{
$\lambda=0.02, \varepsilon=0.03, d=5$. Peak memory with at most 100 thresholds per continuous feature.
Peak RSS is reported in MB as mean $\pm$ standard deviation across 3 bootstraps.
Entries marked ``--'' correspond to partial or unfinished runs under a 100-hour timeout and 128GB memory limit.
The lowest completed memory usage in each row is bolded and the second-lowest is underlined. Our methods are the first four shown in the table, each combining our continuous-feature Rashomon set algorithm with a different proxy algorithm.
}
\label{tab:memory-100-8}
\resizebox{\textwidth}{!}{%
\begin{tabular}{lrrrrrrrr}
\toprule
Dataset
& LSR
& SNIP
& SNIP+GR
& Optimal
& PRAXIS
& SORTD
& TreeFARMS
& RESPLIT \\
\midrule
Abalone & $431.0 \pm 249.3$ & $\boldsymbol{363.8 \pm 155.5}$ & \underline{$370.5 \pm 164.6$} & $12823.7 \pm 2559.4$ & $1106.4 \pm 1276.0$ & -- & -- & $8218.2 \pm 3190.3$ \\
Adult & $\boldsymbol{344.3 \pm 11.5}$ & \underline{$350.8 \pm 12.6$} & $359.5 \pm 12.7$ & $12485.5 \pm 478.1$ & $415.7 \pm 18.3$ & $36588.5 \pm 524.1$ & -- & $11378.9 \pm 151.7$ \\
Aging & $250.8 \pm 1.3$ & $249.3 \pm 0.4$ & $250.1 \pm 0.6$ & $254.5 \pm 0.4$ & \underline{$222.8 \pm 0.1$} & $\boldsymbol{195.6 \pm 0.8}$ & $5273.8 \pm 1116.2$ & $234.8 \pm 0.3$ \\
Bank & \underline{$287.3 \pm 1.0$} & $\boldsymbol{272.4 \pm 1.6}$ & $288.0 \pm 1.1$ & $5910.5 \pm 99.7$ & $288.1 \pm 0.3$ & -- & -- & $13410.7 \pm 77.4$ \\
Bike & $\boldsymbol{297.3 \pm 8.6}$ & \underline{$299.9 \pm 6.1$} & $300.5 \pm 7.2$ & $3045.2 \pm 89.9$ & $317.4 \pm 11.9$ & $7148.5 \pm 220.0$ & -- & $2078.5 \pm 66.3$ \\
Churn & $\boldsymbol{282.5 \pm 3.8}$ & \underline{$284.2 \pm 5.1$} & $285.5 \pm 4.7$ & $73047.4 \pm 10055.8$ & $571.3 \pm 45.5$ & -- & -- & $16742.8 \pm 477.8$ \\
Compas & $256.6 \pm 1.8$ & \underline{$253.3 \pm 2.1$} & $256.4 \pm 0.8$ & $277.4 \pm 1.3$ & $\boldsymbol{228.6 \pm 1.9}$ & $298.1 \pm 6.6$ & -- & $317.0 \pm 2.5$ \\
Coupon & $273.4 \pm 34.5$ & $252.3 \pm 2.3$ & $254.3 \pm 1.5$ & \underline{$251.7 \pm 0.7$} & $1746.7 \pm 2508.9$ & $\boldsymbol{199.5 \pm 2.0}$ & -- & -- \\
Credit & $341.4 \pm 17.5$ & $\boldsymbol{335.2 \pm 1.1}$ & \underline{$339.6 \pm 1.1$} & -- & $765.6 \pm 13.5$ & -- & -- & $89769.3 \pm 1303.4$ \\
Diabetes & $380.6 \pm 0.3$ & \underline{$353.8 \pm 0.3$} & $381.1 \pm 1.5$ & $807.9 \pm 54.1$ & $\boldsymbol{323.7 \pm 0.6}$ & $5819.4 \pm 157.4$ & -- & $1432.0 \pm 29.1$ \\
Diamonds & \underline{$293.5 \pm 1.6$} & $\boldsymbol{287.7 \pm 1.6}$ & $296.8 \pm 1.6$ & $2611.4 \pm 215.3$ & $357.9 \pm 5.6$ & -- & -- & $10629.9 \pm 379.5$ \\
Helena & $425.0 \pm 3.3$ & $\boldsymbol{404.1 \pm 2.5}$ & \underline{$422.7 \pm 3.8$} & -- & $1491.2 \pm 30.8$ & -- & -- & -- \\
Heloc & $\boldsymbol{525.0 \pm 216.7}$ & \underline{$541.7 \pm 157.6$} & $558.8 \pm 179.3$ & -- & $1751.9 \pm 868.2$ & -- & -- & $9313.2 \pm 1657.7$ \\
Jasmine & $\boldsymbol{466.5 \pm 136.2}$ & \underline{$500.2 \pm 159.6$} & $513.0 \pm 159.4$ & $127143.2 \pm 3196.2$ & $864.6 \pm 367.2$ & -- & -- & $5204.7 \pm 213.2$ \\
Pol & $\boldsymbol{1308.5 \pm 507.5}$ & $1420.4 \pm 480.4$ & \underline{$1341.1 \pm 410.0$} & $27732.1 \pm 2957.8$ & $3135.5 \pm 1062.4$ & -- & -- & $8678.5 \pm 166.9$ \\
Rl & $596.2 \pm 152.8$ & $\boldsymbol{557.2 \pm 88.0}$ & \underline{$568.9 \pm 69.2$} & $17706.2 \pm 2953.7$ & $1812.7 \pm 520.3$ & $26898.2 \pm 916.4$ & -- & $5012.2 \pm 506.3$ \\
Shopping & $327.9 \pm 4.6$ & $\boldsymbol{297.6 \pm 1.3}$ & \underline{$300.6 \pm 1.7$} & $41825.9 \pm 9248.8$ & $466.3 \pm 11.4$ & -- & -- & $7043.9 \pm 162.2$ \\
Spambase & $\boldsymbol{2165.7 \pm 660.2}$ & \underline{$3169.9 \pm 692.4$} & $3181.9 \pm 546.0$ & -- & $33904.9 \pm 7972.6$ & -- & -- & -- \\
Student & $252.9 \pm 0.7$ & \underline{$252.5 \pm 1.1$} & $254.3 \pm 1.0$ & $871.2 \pm 22.9$ & $\boldsymbol{225.8 \pm 0.8}$ & $488.7 \pm 20.7$ & -- & $286.7 \pm 4.9$ \\
Wine & $335.7 \pm 41.2$ & $\boldsymbol{299.6 \pm 32.7}$ & \underline{$301.6 \pm 32.3$} & $60847.9 \pm 3302.1$ & $558.9 \pm 185.0$ & -- & -- & $12178.3 \pm 1361.2$ \\
\bottomrule
\end{tabular}%
}
\footnotesize{
LSR = LicketySPLIT over a binarization; SNIP = LicketySNIP (fully continuous); SNIP+GR = LicketySNIP with its greedy subroutine restricted to a binarization. 
}
\end{table*}

\begin{table*}[!t]
\centering
\scriptsize
\setlength{\tabcolsep}{3pt}
\caption{
$\lambda=0.02, \varepsilon=0.03, d=5$. Recall with at most 100 thresholds per continuous feature for our four methods.
Mean $\pm$ standard deviation is reported across 3 bootstraps.
Recall is measured relative to the best method that finished: we use the minimum objective found by any method as the guessed Rashomon bound, count the trees each method found within that bound, and divide by the largest such count.
Entries marked ``--'' correspond to partial or unfinished runs.
The best recall in each row is bolded and the second-best distinct recall is underlined. Our methods are the four shown in the table, each combining our continuous-feature Rashomon set algorithm with a different proxy algorithm.
}
\label{tab:recall-100-ours}
\resizebox{\textwidth}{!}{%
\begin{tabular}{lrrrr}
\toprule
Dataset
& LSR
& SNIP
& SNIP+GR
& Optimal \\
\midrule
Abalone & $\boldsymbol{1.000 \pm 0.000}$ & $\boldsymbol{1.000 \pm 0.000}$ & $\boldsymbol{1.000 \pm 0.000}$ & $\boldsymbol{1.000 \pm 0.000}$ \\
Adult & $\boldsymbol{1.000 \pm 0.000}$ & $\boldsymbol{1.000 \pm 0.000}$ & $\boldsymbol{1.000 \pm 0.000}$ & $\boldsymbol{1.000 \pm 0.000}$ \\
Aging & $\boldsymbol{1.000 \pm 0.000}$ & $\boldsymbol{1.000 \pm 0.000}$ & $\boldsymbol{1.000 \pm 0.000}$ & $\boldsymbol{1.000 \pm 0.000}$ \\
Bank & $\boldsymbol{1.000 \pm 0.000}$ & $\boldsymbol{1.000 \pm 0.000}$ & $\boldsymbol{1.000 \pm 0.000}$ & $\boldsymbol{1.000 \pm 0.000}$ \\
Bike & $\boldsymbol{1.000 \pm 0.000}$ & $\boldsymbol{1.000 \pm 0.000}$ & $\boldsymbol{1.000 \pm 0.000}$ & $\boldsymbol{1.000 \pm 0.000}$ \\
Churn & $\boldsymbol{1.000 \pm 0.000}$ & $\boldsymbol{1.000 \pm 0.000}$ & $\boldsymbol{1.000 \pm 0.000}$ & $\boldsymbol{1.000 \pm 0.000}$ \\
Compas & $\boldsymbol{1.000 \pm 0.000}$ & $\boldsymbol{1.000 \pm 0.000}$ & $\boldsymbol{1.000 \pm 0.000}$ & $\boldsymbol{1.000 \pm 0.000}$ \\
Coupon & \underline{$0.667 \pm 0.577$} & $\boldsymbol{1.000 \pm 0.000}$ & $\boldsymbol{1.000 \pm 0.000}$ & $\boldsymbol{1.000 \pm 0.000}$ \\
Credit & $\boldsymbol{1.000 \pm 0.000}$ & $\boldsymbol{1.000 \pm 0.000}$ & $\boldsymbol{1.000 \pm 0.000}$ & -- \\
Diabetes & $\boldsymbol{1.000 \pm 0.000}$ & $\boldsymbol{1.000 \pm 0.000}$ & $\boldsymbol{1.000 \pm 0.000}$ & $\boldsymbol{1.000 \pm 0.000}$ \\
Diamonds & $\boldsymbol{1.000 \pm 0.000}$ & $\boldsymbol{1.000 \pm 0.000}$ & $\boldsymbol{1.000 \pm 0.000}$ & $\boldsymbol{1.000 \pm 0.000}$ \\
Helena & $\boldsymbol{1.000 \pm 0.000}$ & $\boldsymbol{1.000 \pm 0.000}$ & $\boldsymbol{1.000 \pm 0.000}$ & -- \\
Heloc & $\boldsymbol{1.000 \pm 0.000}$ & $\boldsymbol{1.000 \pm 0.000}$ & $\boldsymbol{1.000 \pm 0.000}$ & -- \\
Jasmine & \underline{$0.931 \pm 0.120$} & $\boldsymbol{1.000 \pm 0.000}$ & $\boldsymbol{1.000 \pm 0.000}$ & $\boldsymbol{1.000 \pm 0.000}$ \\
Pol & \underline{$0.999 \pm 0.001$} & $\boldsymbol{1.000 \pm 0.000}$ & $\boldsymbol{1.000 \pm 0.000}$ & $\boldsymbol{1.000 \pm 0.000}$ \\
Rl & \underline{$0.995 \pm 0.009$} & $\boldsymbol{1.000 \pm 0.000}$ & $\boldsymbol{1.000 \pm 0.000}$ & $\boldsymbol{1.000 \pm 0.000}$ \\
Shopping & $\boldsymbol{1.000 \pm 0.000}$ & $\boldsymbol{1.000 \pm 0.000}$ & $\boldsymbol{1.000 \pm 0.000}$ & $\boldsymbol{1.000 \pm 0.000}$ \\
Spambase & $0.989 \pm 0.017$ & $\boldsymbol{0.996 \pm 0.006}$ & \underline{$0.990 \pm 0.009$} & -- \\
Student & $\boldsymbol{1.000 \pm 0.000}$ & $\boldsymbol{1.000 \pm 0.000}$ & $\boldsymbol{1.000 \pm 0.000}$ & $\boldsymbol{1.000 \pm 0.000}$ \\
Wine & $\boldsymbol{1.000 \pm 0.000}$ & $\boldsymbol{1.000 \pm 0.000}$ & $\boldsymbol{1.000 \pm 0.000}$ & $\boldsymbol{1.000 \pm 0.000}$ \\
\bottomrule
\end{tabular}%
}
\end{table*}

\begin{table*}[!t]
\centering
\scriptsize
\setlength{\tabcolsep}{2pt}
\caption{
$\lambda=0.01, \varepsilon=0.03, d=5$. Runtime for the exhaustive-threshold (fully continuous) setting. Time is reported in seconds as mean $\pm$ std across 3 bootstraps.
Best values are bolded and second-best values are underlined. Entries marked ``--'' correspond to partial or unfinished runs with a 100hr timeout and 128GB memory limit. Our methods are the four shown in the table, each combining our continuous-feature Rashomon set algorithm with a different proxy algorithm.
}
\label{tab:runtime-exhaustive-lam001}
\resizebox{\textwidth}{!}{%
\begin{tabular}{lrrrr}
\toprule
Dataset & LicketySPLIT over a binarization & LicketySNIP (fully continuous) & LicketySNIP (with greedy subroutine using a binarization) & Optimal \\
\midrule
Abalone & $\boldsymbol{14796.07 \pm 5622.74}$ & -- & \underline{$23979.29 \pm 9221.78$} & -- \\
Adult & $\boldsymbol{303.37 \pm 38.40}$ & $19017.89 \pm 587.66$ & \underline{$535.41 \pm 47.74$} & -- \\
Aging & $\boldsymbol{0.08 \pm 0.00}$ & $\boldsymbol{0.08 \pm 0.00}$ & $\boldsymbol{0.08 \pm 0.00}$ & \underline{$0.27 \pm 0.02$} \\
Bank & \underline{$1087.63 \pm 678.55$} & $23380.61 \pm 15301.02$ & $\boldsymbol{1024.59 \pm 701.07}$ & -- \\
Bike & $\boldsymbol{18.41 \pm 10.36}$ & $69.35 \pm 18.80$ & \underline{$20.24 \pm 6.84$} & $3151.88 \pm 141.49$ \\
Churn & $\boldsymbol{2162.29 \pm 2435.05}$ & $148043.59 \pm 125342.15$ & \underline{$6403.88 \pm 6989.95$} & -- \\
Compas & $\boldsymbol{2.95 \pm 0.76}$ & $4.91 \pm 0.68$ & \underline{$3.07 \pm 0.74$} & $12.14 \pm 0.67$ \\
Coupon & $1.08 \pm 1.08$ & $0.60 \pm 0.08$ & \underline{$0.34 \pm 0.10$} & $\boldsymbol{0.20 \pm 0.01}$ \\
Credit & $\boldsymbol{222.26 \pm 132.96}$ & $121677.23 \pm 42357.76$ & \underline{$291.35 \pm 109.58$} & -- \\
Diabetes & \underline{$13.67 \pm 0.76$} & $18.95 \pm 0.71$ & $\boldsymbol{8.06 \pm 0.21}$ & $3815.47 \pm 37.82$ \\
Diamonds & \underline{$126.94 \pm 3.90$} & $950.32 \pm 90.93$ & $\boldsymbol{119.30 \pm 2.11}$ & $49773.95 \pm 4976.55$ \\
Heloc & $\boldsymbol{4973.98 \pm 7001.58}$ & -- & -- & -- \\
Jasmine & $\boldsymbol{2791.13 \pm 2212.80}$ & $24647.36 \pm 24775.85$ & \underline{$6131.83 \pm 7256.69$} & -- \\
Pol & $\boldsymbol{493.97 \pm 120.44}$ & $4160.87 \pm 553.92$ & \underline{$720.60 \pm 160.16$} & $75883.34 \pm 5102.78$ \\
RL & $\boldsymbol{3373.67 \pm 1552.18}$ & $32741.28 \pm 7847.16$ & \underline{$6501.85 \pm 2522.64$} & -- \\
Shopping & \underline{$2088.32 \pm 245.03$} & $67693.59 \pm 9408.62$ & $\boldsymbol{1405.68 \pm 132.88}$ & -- \\
Student & \underline{$0.65 \pm 0.37$} & $2.40 \pm 2.13$ & $\boldsymbol{0.59 \pm 0.19}$ & $65.68 \pm 5.63$ \\
Wine & $\boldsymbol{344.17 \pm 376.88}$ & $3571.75 \pm 4037.83$ & \underline{$371.28 \pm 419.06$} & -- \\
\bottomrule
\end{tabular}%
}
\end{table*}

\begin{table*}[!t]
\centering
\scriptsize
\setlength{\tabcolsep}{2pt}
\caption{
$\lambda=0.01, \varepsilon=0.03, d=5$. Peak memory for the exhaustive-threshold (fully continuous) setting. Peak RSS is reported in MB as mean $\pm$ std across 3 bootstraps.
Best values are bolded and second-best values are underlined. Entries marked ``--'' correspond to partial or unfinished runs with a 100hr timeout and 128GB memory limit. Our methods are the four shown in the table, each combining our continuous-feature Rashomon set algorithm with a different proxy algorithm.
}
\label{tab:memory-exhaustive-lam001}
\resizebox{\textwidth}{!}{%
\begin{tabular}{lrrrr}
\toprule
Dataset & LicketySPLIT over a binarization & LicketySNIP (fully continuous) & LicketySNIP (with greedy subroutine using a binarization) & Optimal \\
\midrule
Abalone & $\boldsymbol{84966.9 \pm 30923.3}$ & -- & \underline{$97009.6 \pm 36423.7$} & -- \\
Adult & $\boldsymbol{886.8 \pm 30.2}$ & $1171.7 \pm 60.6$ & \underline{$1127.4 \pm 24.3$} & -- \\
Aging & \underline{$250.0 \pm 0.9$} & $\boldsymbol{247.6 \pm 1.0}$ & $250.6 \pm 0.2$ & $255.6 \pm 0.3$ \\
Bank & $1449.3 \pm 752.0$ & \underline{$1425.5 \pm 764.1$} & $\boldsymbol{1291.5 \pm 669.2}$ & -- \\
Bike & \underline{$353.3 \pm 53.8$} & $356.2 \pm 27.6$ & $\boldsymbol{351.1 \pm 23.7}$ & $5698.7 \pm 191.7$ \\
Churn & $\boldsymbol{16918.3 \pm 18404.5}$ & \underline{$34152.9 \pm 33922.0$} & $37057.9 \pm 38343.5$ & -- \\
Compas & $288.1 \pm 10.0$ & $\boldsymbol{280.2 \pm 6.1}$ & \underline{$285.5 \pm 9.1$} & $310.5 \pm 4.3$ \\
Coupon & $354.5 \pm 104.2$ & $303.3 \pm 12.6$ & \underline{$285.1 \pm 12.2$} & $\boldsymbol{252.4 \pm 0.5}$ \\
Credit & $\boldsymbol{1383.6 \pm 90.9}$ & $1475.5 \pm 10.3$ & \underline{$1470.6 \pm 11.8$} & -- \\
Diabetes & $380.8 \pm 1.0$ & $\boldsymbol{353.6 \pm 0.3}$ & \underline{$380.6 \pm 1.4$} & $1268.5 \pm 6.4$ \\
Diamonds & $455.3 \pm 16.3$ & \underline{$425.3 \pm 7.9$} & $\boldsymbol{423.1 \pm 7.5}$ & $15013.6 \pm 274.2$ \\
Heloc & $\boldsymbol{44727.9 \pm 60429.1}$ & -- & -- & -- \\
Jasmine & $\boldsymbol{29019.5 \pm 24304.9}$ & \underline{$30157.7 \pm 34382.8$} & $32932.1 \pm 35373.9$ & -- \\
Pol & $\boldsymbol{2848.3 \pm 579.4}$ & $3594.2 \pm 627.6$ & \underline{$3279.9 \pm 432.7$} & $71209.5 \pm 3690.3$ \\
RL & $\boldsymbol{34691.5 \pm 14586.2}$ & \underline{$44668.7 \pm 16325.0$} & $45394.2 \pm 16727.7$ & -- \\
Shopping & $5700.2 \pm 473.7$ & \underline{$3668.6 \pm 478.1$} & $\boldsymbol{3470.2 \pm 311.8}$ & -- \\
Student & \underline{$272.9 \pm 14.9$} & $279.4 \pm 26.8$ & $\boldsymbol{269.4 \pm 6.9}$ & $947.2 \pm 4.0$ \\
Wine & \underline{$1996.1 \pm 1981.0$} & $\boldsymbol{1940.5 \pm 2026.2}$ & $2030.6 \pm 2157.7$ & -- \\
\bottomrule
\end{tabular}%
}
\end{table*}

\begin{table*}[!t]
\centering
\scriptsize
\setlength{\tabcolsep}{2pt}
\caption{
$\lambda=0.01, \varepsilon=0.03, d=5$. Recall for the exhaustive-threshold (fully continuous) setting, reported as mean $\pm$ std across 3 bootstraps. Recall is measured relative to the best method that finished: we take the minimum objective found by any completed method, count each method's trees within the resulting Rashomon bound, and divide by the largest such count.
Best values are bolded and second-best values are underlined. Entries marked ``--'' correspond to partial or unfinished runs with a 100hr timeout and 128GB memory limit. Our methods are the four shown in the table, each combining our continuous-feature Rashomon set algorithm with a different proxy algorithm.
}
\label{tab:recall-exhaustive-lam001}
\resizebox{\textwidth}{!}{%
\begin{tabular}{lrrrr}
\toprule
Dataset & LicketySPLIT over a binarization & LicketySNIP (fully continuous) & LicketySNIP (with greedy subroutine using a binarization) & Optimal \\
\midrule
Abalone & \underline{$0.999 \pm 0.002$} & -- & $\boldsymbol{1.000 \pm 0.000}$ & -- \\
Adult & $\boldsymbol{1.000 \pm 0.000}$ & $\boldsymbol{1.000 \pm 0.000}$ & $\boldsymbol{1.000 \pm 0.000}$ & -- \\
Aging & $\boldsymbol{1.000 \pm 0.000}$ & $\boldsymbol{1.000 \pm 0.000}$ & $\boldsymbol{1.000 \pm 0.000}$ & $\boldsymbol{1.000 \pm 0.000}$ \\
Bank & $\boldsymbol{1.000 \pm 0.000}$ & $\boldsymbol{1.000 \pm 0.000}$ & $\boldsymbol{1.000 \pm 0.000}$ & -- \\
Bike & $\boldsymbol{1.000 \pm 0.000}$ & \underline{$0.994 \pm 0.010$} & \underline{$0.994 \pm 0.010$} & $\boldsymbol{1.000 \pm 0.000}$ \\
Churn & \underline{$0.967 \pm 0.055$} & $\boldsymbol{1.000 \pm 0.000}$ & $\boldsymbol{1.000 \pm 0.000}$ & -- \\
Compas & $\boldsymbol{1.000 \pm 0.000}$ & $\boldsymbol{1.000 \pm 0.000}$ & $\boldsymbol{1.000 \pm 0.000}$ & $\boldsymbol{1.000 \pm 0.000}$ \\
Coupon & \underline{$0.999 \pm 0.001$} & $\boldsymbol{1.000 \pm 0.000}$ & $\boldsymbol{1.000 \pm 0.000}$ & $\boldsymbol{1.000 \pm 0.000}$ \\
Credit & $\boldsymbol{1.000 \pm 0.000}$ & $\boldsymbol{1.000 \pm 0.000}$ & $\boldsymbol{1.000 \pm 0.000}$ & -- \\
Diabetes & $\boldsymbol{1.000 \pm 0.000}$ & $\boldsymbol{1.000 \pm 0.000}$ & $\boldsymbol{1.000 \pm 0.000}$ & $\boldsymbol{1.000 \pm 0.000}$ \\
Diamonds & $\boldsymbol{1.000 \pm 0.000}$ & $\boldsymbol{1.000 \pm 0.000}$ & $\boldsymbol{1.000 \pm 0.000}$ & $\boldsymbol{1.000 \pm 0.000}$ \\
Heloc & $\boldsymbol{1.000 \pm 0.000}$ & -- & -- & -- \\
Jasmine & $0.898 \pm 0.176$ & $\boldsymbol{1.000 \pm 0.000}$ & \underline{$0.981 \pm 0.032$} & -- \\
Pol & $\boldsymbol{1.000 \pm 0.000}$ & $\boldsymbol{1.000 \pm 0.000}$ & $\boldsymbol{1.000 \pm 0.000}$ & $\boldsymbol{1.000 \pm 0.000}$ \\
RL & $0.911 \pm 0.114$ & $\boldsymbol{1.000 \pm 0.000}$ & \underline{$0.988 \pm 0.019$} & -- \\
Shopping & $\boldsymbol{1.000 \pm 0.000}$ & $\boldsymbol{1.000 \pm 0.000}$ & $\boldsymbol{1.000 \pm 0.000}$ & -- \\
Student & $0.928 \pm 0.112$ & $\boldsymbol{1.000 \pm 0.000}$ & \underline{$0.989 \pm 0.018$} & $\boldsymbol{1.000 \pm 0.000}$ \\
Wine & \underline{$0.998 \pm 0.004$} & $\boldsymbol{1.000 \pm 0.000}$ & $\boldsymbol{1.000 \pm 0.000}$ & -- \\
\bottomrule
\end{tabular}%
}
\end{table*}

\begin{table*}[!t]
\centering
\scriptsize
\setlength{\tabcolsep}{5pt}
\caption{
$\lambda=0.005, \varepsilon=0.03, d=5$. Runtime for the exhaustive threshold (fully continuous) setting. Time is reported in seconds as mean $\pm$ std across 3 bootstraps. Entries marked ``--'' correspond to partial or unfinished runs with a 100hr timeout and 128GB memory limit. The fastest completed method is bolded and the second-fastest distinct method is underlined. Our methods are the four shown in the table, each combining our continuous-feature Rashomon set algorithm with a different proxy algorithm.
}
\label{tab:runtime-exhaustive-lam0005-four-methods}
\resizebox{\textwidth}{!}{%
\begin{tabular}{lrrrr}
\toprule
Dataset
& LicketySPLIT over a binarization
& LicketySNIP (fully continuous)
& LicketySNIP (with greedy subroutine using a binarization)
& Optimal \\
\midrule
Adult & $\boldsymbol{7051.04 \pm 2359.88}$ & -- & \underline{$28856.32 \pm 10451.01$} & -- \\
Aging & \underline{$0.14 \pm 0.01$} & \underline{$0.14 \pm 0.01$} & $\boldsymbol{0.12 \pm 0.00}$ & $0.31 \pm 0.02$ \\
Bank & $\boldsymbol{31700.27 \pm 17379.86}$ & -- & \underline{$65726.91 \pm 1897.62$} & -- \\
Bike & $\boldsymbol{147.49 \pm 82.47}$ & $569.73 \pm 224.77$ & \underline{$191.26 \pm 77.17$} & $5157.56 \pm 397.29$ \\
Compas & $\boldsymbol{9.32 \pm 3.01}$ & $17.19 \pm 4.54$ & \underline{$13.72 \pm 4.76$} & $19.57 \pm 2.33$ \\
Coupon & $1.13 \pm 1.20$ & $0.59 \pm 0.09$ & \underline{$0.38 \pm 0.11$} & $\boldsymbol{0.21 \pm 0.01}$ \\
Diabetes & \underline{$24.49 \pm 1.07$} & $40.79 \pm 1.69$ & $\boldsymbol{16.71 \pm 0.57}$ & $11244.89 \pm 160.80$ \\
Diamonds & $\boldsymbol{969.94 \pm 37.14}$ & $7918.95 \pm 385.47$ & \underline{$1331.94 \pm 94.96$} & $184770.98 \pm 8532.51$ \\
Pol & $\boldsymbol{152.20 \pm 29.16}$ & $1598.81 \pm 253.20$ & \underline{$218.53 \pm 60.63$} & $98966.66 \pm 2495.21$ \\
Student & $\boldsymbol{5.57 \pm 4.24}$ & \underline{$7.69 \pm 5.23$} & $15.77 \pm 18.00$ & $78.95 \pm 3.47$ \\
\bottomrule
\end{tabular}%
}
\end{table*}

\begin{table*}[!t]
\centering
\scriptsize
\setlength{\tabcolsep}{5pt}
\caption{
$\lambda=0.005, \varepsilon=0.03, d=5$. Peak memory for the exhaustive threshold (fully continuous) setting. Peak RSS is reported in MB as mean $\pm$ std across 3 bootstraps. Entries marked ``--'' correspond to partial or unfinished runs with a 100hr timeout and 128GB memory limit. The lowest completed memory usage is bolded and the second-lowest distinct value is underlined. Our methods are the four shown in the table, each combining our continuous-feature Rashomon set algorithm with a different proxy algorithm.
}
\label{tab:memory-exhaustive-lam0005-four-methods}
\resizebox{\textwidth}{!}{%
\begin{tabular}{lrrrr}
\toprule
Dataset
& LicketySPLIT over a binarization
& LicketySNIP (fully continuous)
& LicketySNIP (with greedy subroutine using a binarization)
& Optimal \\
\midrule
Adult & $\boldsymbol{10311.2 \pm 3203.0}$ & -- & \underline{$20715.1 \pm 7059.0$} & -- \\
Aging & $253.2 \pm 2.1$ & $\boldsymbol{248.1 \pm 0.2}$ & \underline{$251.9 \pm 1.2$} & $254.3 \pm 1.3$ \\
Bank & $\boldsymbol{29022.6 \pm 15684.3}$ & -- & \underline{$53701.6 \pm 1801.3$} & -- \\
Bike & $\boldsymbol{1167.8 \pm 471.9}$ & $1377.4 \pm 611.2$ & \underline{$1322.2 \pm 585.2$} & $7719.4 \pm 471.1$ \\
Compas & $474.5 \pm 106.4$ & $\boldsymbol{466.4 \pm 107.5}$ & $486.8 \pm 113.4$ & \underline{$469.8 \pm 96.9$} \\
Coupon & $355.3 \pm 103.1$ & $304.3 \pm 10.2$ & \underline{$285.6 \pm 12.0$} & $\boldsymbol{251.7 \pm 0.5}$ \\
Diabetes & \underline{$379.7 \pm 1.0$} & $\boldsymbol{352.0 \pm 2.2}$ & $381.3 \pm 0.6$ & $2678.1 \pm 24.8$ \\
Diamonds & $\boldsymbol{1314.9 \pm 56.3}$ & $1738.6 \pm 189.6$ & \underline{$1644.0 \pm 104.9$} & $50519.3 \pm 1195.6$ \\
Pol & $\boldsymbol{1047.3 \pm 78.6}$ & $1273.7 \pm 166.8$ & \underline{$1168.6 \pm 207.5$} & $93571.6 \pm 1205.8$ \\
Student & \underline{$678.6 \pm 484.0$} & $\boldsymbol{424.7 \pm 198.8}$ & $1112.4 \pm 1167.2$ & $963.6 \pm 13.4$ \\
\bottomrule
\end{tabular}%
}
\end{table*}

\begin{table*}[!t]
\centering
\scriptsize
\setlength{\tabcolsep}{5pt}
\caption{
$\lambda=0.005, \varepsilon=0.03, d=5$. Recall for the exhaustive threshold (fully continuous) setting. Mean $\pm$ std is reported across 3 bootstraps. Recall is measured relative to the best method that finished. Entries marked ``--'' correspond to partial or unfinished runs. The highest completed recall is bolded and the second-highest distinct value is underlined. Our methods are the four shown in the table, each combining our continuous-feature Rashomon set algorithm with a different proxy algorithm.
}
\label{tab:recall-exhaustive-lam0005-four-methods}
\resizebox{\textwidth}{!}{%
\begin{tabular}{lrrrr}
\toprule
Dataset
& LicketySPLIT over a binarization
& LicketySNIP (fully continuous)
& LicketySNIP (with greedy subroutine using a binarization)
& Optimal \\
\midrule
Adult & $\boldsymbol{1.000 \pm 0.000}$ & -- & $\boldsymbol{1.000 \pm 0.000}$ & -- \\
Aging & $\boldsymbol{1.000 \pm 0.000}$ & $\boldsymbol{1.000 \pm 0.000}$ & $\boldsymbol{1.000 \pm 0.000}$ & $\boldsymbol{1.000 \pm 0.000}$ \\
Bank & $\boldsymbol{1.000 \pm 0.000}$ & -- & $\boldsymbol{1.000 \pm 0.000}$ & -- \\
Bike & \underline{$0.999 \pm 0.002$} & \underline{$0.999 \pm 0.002$} & \underline{$0.999 \pm 0.002$} & $\boldsymbol{1.000 \pm 0.000}$ \\
Compas & $\boldsymbol{1.000 \pm 0.001}$ & $\boldsymbol{1.000 \pm 0.000}$ & $\boldsymbol{1.000 \pm 0.000}$ & $\boldsymbol{1.000 \pm 0.000}$ \\
Coupon & \underline{$0.999 \pm 0.001$} & $\boldsymbol{1.000 \pm 0.000}$ & $\boldsymbol{1.000 \pm 0.000}$ & $\boldsymbol{1.000 \pm 0.000}$ \\
Diabetes & $\boldsymbol{1.000 \pm 0.000}$ & $\boldsymbol{1.000 \pm 0.000}$ & $\boldsymbol{1.000 \pm 0.000}$ & $\boldsymbol{1.000 \pm 0.000}$ \\
Diamonds & $\boldsymbol{1.000 \pm 0.000}$ & $\boldsymbol{1.000 \pm 0.000}$ & $\boldsymbol{1.000 \pm 0.000}$ & $\boldsymbol{1.000 \pm 0.000}$ \\
Pol & \underline{$0.949 \pm 0.088$} & $0.945 \pm 0.096$ & $0.945 \pm 0.096$ & $\boldsymbol{1.000 \pm 0.000}$ \\
Student & $0.880 \pm 0.206$ & \underline{$0.995 \pm 0.009$} & \underline{$0.995 \pm 0.009$} & $\boldsymbol{1.000 \pm 0.000}$ \\
\bottomrule
\end{tabular}%
}
\end{table*}

\begin{table*}[!t]
\centering
\small
\setlength{\tabcolsep}{6pt}
\caption{
Runtime comparison of LicketySNIP using sorted lists of active sample indices
versus binary threshold columns for its fully continuous greedy subroutine.
The italicized ArborEnum+SNIP+GR column is shown only as a reference and is not part
of the comparison; accordingly, it does not affect the bolding.
$\lambda=0.02$, $\varepsilon=0.03$, $d=5$, and exhaustive thresholds.
Time is reported in seconds as mean $\pm$ standard deviation across 3 bootstraps.
}
\label{tab:sorted-vs-bitvector-greedy}
\begin{tabular}{lrr@{\hspace{12pt}}|@{\hspace{12pt}}r}
\toprule
Dataset
& Sorted lists
& Binary columns
& \textit{ArborEnum+SNIP+GR (reference)} \\
\midrule
Abalone
& $\mathbf{2772.79 \pm 4095.38}$
& $2810.60 \pm 4187.00$
& \textit{$205.14 \pm 311.57$} \\

Adult
& $\mathbf{5246.38 \pm 159.04}$
& $12131.75 \pm 371.85$
& \textit{$292.12 \pm 16.25$} \\

Aging
& $\mathbf{0.08 \pm 0.00}$
& $\mathbf{0.08 \pm 0.00}$
& \textit{$0.08 \pm 0.00$} \\

Bank
& $\mathbf{91.67 \pm 1.07}$
& $212.52 \pm 5.41$
& \textit{$6.43 \pm 0.11$} \\

Bike
& $104.38 \pm 12.09$
& $\mathbf{27.05 \pm 4.96}$
& \textit{$7.95 \pm 1.91$} \\

Churn
& $815.17 \pm 832.06$
& $\mathbf{643.87 \pm 555.53}$
& \textit{$9.42 \pm 9.13$} \\

Compas
& $2.24 \pm 0.10$
& $\mathbf{0.50 \pm 0.19}$
& \textit{$0.35 \pm 0.12$} \\

Coupon
& $0.22 \pm 0.09$
& $\mathbf{0.15 \pm 0.04}$
& \textit{$0.12 \pm 0.04$} \\

Credit
& $\mathbf{7204.55 \pm 1665.56}$
& $24849.80 \pm 619.24$
& \textit{$65.52 \pm 22.19$} \\

Diabetes
& $101.47 \pm 26.68$
& $\mathbf{8.55 \pm 0.22}$
& \textit{$3.77 \pm 0.17$} \\

Diamonds
& $614.77 \pm 181.34$
& $\mathbf{313.35 \pm 23.79}$
& \textit{$39.10 \pm 3.25$} \\

Helena
& $\mathbf{19566.26 \pm 566.81}$
& $286987.90 \pm 12346.63$
& \textit{$185.50 \pm 11.11$} \\

Heloc
& $720.83 \pm 339.09$
& $\mathbf{273.53 \pm 133.41}$
& \textit{$34.89 \pm 26.12$} \\

Jasmine
& $501.59 \pm 419.02$
& $\mathbf{405.78 \pm 427.03}$
& \textit{$61.10 \pm 70.61$} \\

Pol
& $7438.25 \pm 3733.37$
& $\mathbf{2155.51 \pm 817.12}$
& \textit{$383.98 \pm 190.49$} \\

Rl
& $883.34 \pm 134.21$
& $\mathbf{666.06 \pm 74.75}$
& \textit{$117.10 \pm 33.85$} \\

Shopping
& $\mathbf{18270.56 \pm 2017.08}$
& $25547.42 \pm 2593.81$
& \textit{$489.25 \pm 57.80$} \\

Student
& $0.40 \pm 0.08$
& $\mathbf{0.26 \pm 0.05}$
& \textit{$0.17 \pm 0.02$} \\

Wine
& $199.71 \pm 167.05$
& $\mathbf{113.01 \pm 79.44}$
& \textit{$10.64 \pm 7.08$} \\
\bottomrule
\end{tabular}
\end{table*}

\begin{table*}[!t]
\centering
\small
\setlength{\tabcolsep}{6pt}
\caption{
Runtime comparison of LicketySNIP using sorted lists of active sample indices
versus binary threshold columns for its fully continuous greedy subroutine.
The italicized ArborEnum+SNIP+GR column is shown only as a reference and is not part
of the comparison; accordingly, it does not affect the bolding.
$\lambda=0.01$, $\varepsilon=0.03$, $d=5$, and exhaustive thresholds.
Time is reported in seconds as mean $\pm$ standard deviation across 3 bootstraps.
}
\label{tab:sorted-vs-bitvector-greedy-lam001}
\begin{tabular}{lrr@{\hspace{12pt}}|@{\hspace{12pt}}r}
\toprule
Dataset
& Sorted lists
& Binary columns
& \textit{ArborEnum+SNIP+GR (reference)} \\
\midrule
Adult
& $\mathbf{8757.94 \pm 580.08}$
& $19017.89 \pm 587.66$
& \textit{$535.41 \pm 47.74$} \\

Aging
& $\mathbf{0.08 \pm 0.00}$
& $\mathbf{0.08 \pm 0.00}$
& \textit{$0.08 \pm 0.00$} \\

Bank
& $\mathbf{15233.48 \pm 10927.13}$
& $23380.61 \pm 15301.02$
& \textit{$1024.59 \pm 701.07$} \\

Bike
& $198.65 \pm 39.26$
& $\mathbf{69.35 \pm 18.80}$
& \textit{$20.24 \pm 6.84$} \\

Churn
& $174079.20 \pm 119293.59$
& $\mathbf{148043.59 \pm 125342.15}$
& \textit{$6403.88 \pm 6989.95$} \\

Compas
& $10.33 \pm 0.86$
& $\mathbf{4.91 \pm 0.68}$
& \textit{$3.07 \pm 0.74$} \\

Coupon
& $\mathbf{0.51 \pm 0.08}$
& $0.60 \pm 0.08$
& \textit{$0.34 \pm 0.10$} \\

Credit
& $\mathbf{20343.32 \pm 790.15}$
& $121677.23 \pm 42357.76$
& \textit{$291.35 \pm 109.58$} \\

Diabetes
& $103.66 \pm 6.74$
& $\mathbf{18.95 \pm 0.71}$
& \textit{$8.06 \pm 0.21$} \\

Diamonds
& $1119.35 \pm 75.59$
& $\mathbf{950.32 \pm 90.93}$
& \textit{$119.30 \pm 2.11$} \\

Jasmine
& $50979.34 \pm 68968.40$
& $\mathbf{24647.36 \pm 24775.85}$
& \textit{$6131.83 \pm 7256.69$} \\

Pol
& $8162.68 \pm 818.02$
& $\mathbf{4160.87 \pm 553.92}$
& \textit{$720.60 \pm 160.16$} \\

Rl
& $34333.97 \pm 8102.61$
& $\mathbf{32741.28 \pm 7847.16}$
& \textit{$6501.85 \pm 2522.64$} \\

Shopping
& $\mathbf{41722.32 \pm 4729.10}$
& $67693.59 \pm 9408.62$
& \textit{$1405.68 \pm 132.88$} \\

Student
& $3.09 \pm 2.70$
& $\mathbf{2.40 \pm 2.13}$
& \textit{$0.59 \pm 0.19$} \\

Wine
& $5589.56 \pm 6239.08$
& $\mathbf{3571.75 \pm 4037.83}$
& \textit{$371.28 \pm 419.06$} \\
\bottomrule
\end{tabular}
\end{table*}

\begin{table*}[!t]
\centering
\scriptsize
\setlength{\tabcolsep}{4pt}
\caption{
$\lambda=0.005$, $\varepsilon=0.03$, and $d=5$.
Runtime of ArborEnum+LicketySNIP depending on when the greedy subroutine in LicketySNIP uses binary threshold columns or sorted lists.
The italicized ArborEnum+SNIP+GR column is shown only as a reference and is not part
of the comparison; accordingly, it does not affect the bolding or underlining.
Time is reported in seconds as mean $\pm$ standard deviation across
3 bootstraps. The best value among the first two columns in each row is bolded,
and the second-best is underlined.
}
\label{tab:licketysnip-sorted-lists-lam005}
\resizebox{\columnwidth}{!}{%
\begin{tabular}{lrr@{\hspace{12pt}}|@{\hspace{12pt}}r}
\toprule
Dataset
& Binary threshold columns
& Sorted lists
& \textit{ArborEnum+SNIP+GR (reference)} \\
\midrule
Aging
& \textbf{0.140 $\pm$ 0.010}
& \underline{0.153 $\pm$ 0.018}
& \textit{$0.120 \pm 0.000$} \\
Bike
& \textbf{569.730 $\pm$ 224.770}
& \underline{1086.120 $\pm$ 393.106}
& \textit{$191.260 \pm 77.170$} \\
Compas
& \textbf{17.190 $\pm$ 4.540}
& \underline{23.993 $\pm$ 4.518}
& \textit{$13.720 \pm 4.760$} \\
Coupon
& \underline{0.590 $\pm$ 0.090}
& \textbf{0.524 $\pm$ 0.090}
& \textit{$0.380 \pm 0.110$} \\
Diabetes
& \textbf{40.790 $\pm$ 1.690}
& \underline{238.117 $\pm$ 4.698}
& \textit{$16.710 \pm 0.570$} \\
Diamonds
& \textbf{7918.950 $\pm$ 385.470}
& \underline{8271.592 $\pm$ 493.634}
& \textit{$1331.940 \pm 94.960$} \\
Pol
& \textbf{1598.810 $\pm$ 253.200}
& \underline{3987.730 $\pm$ 392.975}
& \textit{$218.530 \pm 60.630$} \\
Student
& \textbf{7.690 $\pm$ 5.230}
& \underline{8.332 $\pm$ 5.647}
& \textit{$15.770 \pm 18.000$} \\
\bottomrule
\end{tabular}%
}
\end{table*}

\FloatBarrier
\subsection{Anytime Overhead}
\label{app:anytime-overhead}

Table \ref{tab:anytime-overhead} reports the runtime overhead of the anytime algorithm when proxy strength is held fixed. In some cases, the anytime variant is actually faster because it explores subproblems in a different order. Across datasets, the median overhead is only (2.7\%).

In all experiments, we evaluate the anytime algorithm without progressively strengthening the proxy because our approximate proxies already recover essentially all trees while running orders of magnitude faster than an optimal proxy. For example, when the approximation is (100$\times$) faster, approximately (99\%) of the full optimal runtime would be spent obtaining a certificate of optimality and recovering the small number of remaining trees. We therefore focus on the more practically consequential regime of aiming to converge to our approximations. Runtime overhead is also less informative in the optimal-proxy setting, since one could simply discard the approximate computation and rerun the optimal method from scratch, incurring only about a (1\%) overhead when the approximation is (100$\times$) faster.

\begin{table}[!t]
\centering
\scriptsize
\setlength{\tabcolsep}{4pt}
\caption{Runtime overhead of the anytime algorithm (ArborEnum+LicketySNIP) relative to non-anytime execution using the LicketySNIP proxy (ArborEnum+LicketySNIP) over continuous features. Each value is the mean $\pm$ standard deviation across three bootstraps of the anytime runtime divided by the non-anytime runtime. $\lambda=0.02$ and $\varepsilon_{\mathrm{mult}}=0.03$.}
\label{tab:anytime-overhead}
\begin{tabular}{lr}
\toprule
Dataset & Overhead Ratio \\
\midrule
Abalone  & $1.087 \pm 0.011$ \\
Adult    & $1.043 \pm 0.006$ \\
Aging    & $1.009 \pm 0.011$ \\
Bank     & $1.024 \pm 0.012$ \\
Bike     & $1.033 \pm 0.008$ \\
Churn    & $1.029 \pm 0.022$ \\
COMPAS   & $1.056 \pm 0.009$ \\
Coupon   & $0.988 \pm 0.011$ \\
Credit   & $0.860 \pm 0.005$ \\
Diabetes & $1.076 \pm 0.019$ \\
Diamonds & $0.998 \pm 0.008$ \\
HELOC    & $1.006 \pm 0.008$ \\
Jasmine  & $1.018 \pm 0.037$ \\
Pol      & $1.085 \pm 0.011$ \\
RL       & $1.069 \pm 0.014$ \\
Shopping & $1.099 \pm 0.004$ \\
Student  & $1.026 \pm 0.006$ \\
Wine     & $1.021 \pm 0.009$ \\
\midrule
Median   & $1.027$ \\
\bottomrule
\end{tabular}
\end{table}

\FloatBarrier

\subsection{Improving Recall with a Bigger Budget}
\label{app:bigger-budget}

Table \ref{tab:budget-extension} presents the same result as the main paper but for $\lambda=0.01$. We do not show results for $\lambda=0.02$ because this proxy choice always yielded perfect approximation quality (there is nothing to show).

\begin{table}[!t]
\centering
\scriptsize
\setlength{\tabcolsep}{2.5pt}
\caption{For every dataset on which ArborEnum+SNIP+GR does not achieve perfect recall (averaged across bootstraps), extending the root budget recovers the remaining trees with little additional runtime while remaining substantially faster than our optimal method when it completes. The timings are obtained by first solving with $\varepsilon_{\textrm{mult}}=0.03$ and then extending the subgraph. $\lambda=0.01$.}
\label{tab:budget-extension}
\begin{tabular}{lrrrr}
\toprule
Dataset
& Recall
& ArborEnum+SNIP+GR
& Expanded
& ArborEnum+OPT \\
\midrule
Bike
& $.994 \!\to\! 1.000$
& $20\,\mathrm{s}$
& $22\,\mathrm{s}\ (\varepsilon_{\textrm{mult}}=.0325)$
& $3152\,\mathrm{s}$ \\
Jasmine
& $.980 \!\to\! 1.000$
& $6132\,\mathrm{s}$
& $10998\,\mathrm{s}\ (\varepsilon_{\textrm{mult}}=.0375)$
& -- \\
RL
& $.985 \!\to\! 1.000$
& $6502\,\mathrm{s}$
& $16715\,\mathrm{s}\ (\varepsilon_{\textrm{mult}}=.0375)$
& -- \\
\bottomrule
\end{tabular}
\end{table}

\FloatBarrier

\subsection{Budget-Independent Subgraphs}
\label{app:budget-independent-results}

\begin{table*}[!t]
\centering
\scriptsize
\setlength{\tabcolsep}{5pt}
\caption{Effect of budget-independent subgraph caching for
$\lambda=0.005$ and $\varepsilon_{\mathrm{mult}}=0.015$.
Speedup is the runtime without storing pointers to previously constructed
subgraphs, as in PRAXIS \citep{heile2026}, divided by the runtime with
budget-independent subgraph caching. The increase in trees is the number of
trees recovered with budget-independent subgraph caching divided by the number
recovered without storing subgraph pointers.}
\label{tab:cache-lam0005-rm0015}
\resizebox{\textwidth}{!}{%
\begin{tabular}{lrrrr}
\toprule
Dataset
& Speedup by Caching
& $\times$ Increase in Trees by Caching
& OR Nodes with Budget-Independent Caching
& OR Nodes without Stored Subgraph Pointers, as in PRAXIS \\
\midrule
Adult    & 1.043  & 1.000 & 404       & 2,661 \\
Aging    & 0.989  & 1.000 & 1         & 1 \\
Bank     & 0.997  & 1.000 & 1,161     & 1,659 \\
Bike     & 1.032  & 1.000 & 80,817    & 424,071 \\
COMPAS   & 1.054  & 1.000 & 5,512     & 35,981 \\
Coupon   & 53.586 & 1.000 & 92,711    & 116,247,505 \\
Credit   & 0.991  & 1.000 & 3         & 3 \\
Diabetes & 1.003  & 1.000 & 1         & 1 \\
Diamonds & 0.998  & 1.000 & 173       & 349 \\
HELOC    & 1.021  & 1.001 & 1,103,134 & 8,545,465 \\
Pol      & 1.012  & 1.000 & 3,001     & 13,743 \\
Shopping & 1.000  & 1.000 & 1,037     & 1,037 \\
Spambase & 1.322  & 1.007 & 1,047,399 & 24,536,957 \\
Student  & 0.980  & 1.000 & 1,240     & 4,981 \\
Wine     & 1.006  & 1.000 & 562       & 1,177 \\
\bottomrule
\end{tabular}%
}
\end{table*}

\begin{table*}[!t]
\centering
\scriptsize
\setlength{\tabcolsep}{5pt}
\caption{Effect of budget-independent subgraph caching for
$\lambda=0.005$ and $\varepsilon_{\mathrm{mult}}=0.03$.
Speedup is the runtime without storing pointers to previously constructed
subgraphs, as in PRAXIS, divided by the runtime with
budget-independent subgraph caching. The increase in trees is the number of
trees recovered with budget-independent subgraph caching divided by the number
recovered without storing subgraph pointers.}
\label{tab:cache-lam0005-rm003}
\resizebox{\textwidth}{!}{%
\begin{tabular}{lrrrr}
\toprule
Dataset
& Speedup by Caching
& $\times$ Increase in Trees by Caching
& OR Nodes with Budget-Independent Caching
& OR Nodes without Stored Subgraph Pointers, as in PRAXIS \\
\midrule
Adult    & 1.098  & 1.000 & 278,933 & 5,144,219 \\
Aging    & 0.965  & 1.000 & 71      & 71 \\
Bank     & 1.014  & 1.000 & 41,390  & 210,399 \\
Bike     & 1.130  & 1.000 & 814,521 & 5,743,543 \\
COMPAS   & 1.923  & 1.000 & 133,985 & 3,775,339 \\
Coupon   & 74.979 & 1.000 & 150,620 & 301,679,963 \\
Diabetes & 1.004  & 1.000 & 1       & 1 \\
Diamonds & 1.007  & 1.000 & 1,190   & 4,249 \\
Pol      & 1.011  & 1.000 & 13,090  & 67,771 \\
Shopping & 1.016  & 1.000 & 22,672  & 118,759 \\
Student  & 1.044  & 1.001 & 7,592   & 41,319 \\
\bottomrule
\end{tabular}%
}
\end{table*}

\begin{table*}[!t]
\centering
\scriptsize
\setlength{\tabcolsep}{5pt}
\caption{Effect of budget-independent subgraph caching for
$\lambda=0.005$ and $\varepsilon_{\mathrm{mult}}=0.05$.
Speedup is the runtime without storing pointers to previously constructed
subgraphs, as in PRAXIS, divided by the runtime with
budget-independent subgraph caching. The increase in trees is the number of
trees recovered with budget-independent subgraph caching divided by the number
recovered without storing subgraph pointers.}
\label{tab:cache-lam0005-rm005}
\resizebox{\textwidth}{!}{%
\begin{tabular}{lrrrr}
\toprule
Dataset
& Speedup by Caching
& $\times$ Increase in Trees by Caching
& OR Nodes with Budget-Independent Caching
& OR Nodes without Stored Subgraph Pointers, as in \citet{heile2026} \\
\midrule
Aging    & 0.948  & 1.000 & 115       & 151 \\
Bike     & 1.374  & 1.000 & 5,951,647 & 68,214,333 \\
COMPAS   & 11.769 & 1.000 & 1,742,674 & 412,465,135 \\
Coupon   & 73.996 & 1.000 & 150,620   & 301,679,963 \\
Diabetes & 1.007  & 1.000 & 371       & 371 \\
Diamonds & 1.052  & 1.000 & 290,858   & 1,044,781 \\
Pol      & 1.037  & 1.000 & 66,291    & 639,571 \\
Student  & 1.603  & 1.001 & 404,680   & 6,353,339 \\
\bottomrule
\end{tabular}%
}
\end{table*}

\begin{table*}[!t]
\centering
\scriptsize
\setlength{\tabcolsep}{5pt}
\caption{Effect of budget-independent subgraph caching for
$\lambda=0.005$ and $\varepsilon_{\mathrm{mult}}=0.5$.
Speedup is the runtime without storing pointers to previously constructed
subgraphs, as in PRAXIS, divided by the runtime with
budget-independent subgraph caching. The increase in trees is the number of
trees recovered with budget-independent subgraph caching divided by the number
recovered without storing subgraph pointers.}
\label{tab:cache-lam0005-rm05}
\resizebox{\textwidth}{!}{%
\begin{tabular}{lrrrr}
\toprule
Dataset
& Speedup by Caching
& $\times$ Increase in Trees by Caching
& OR Nodes with Budget-Independent Caching
& OR Nodes without Stored Subgraph Pointers, as in \citet{heile2026} \\
\midrule
Aging & 18.510 & 1.000 & 1,167,911 & 329,792,087 \\
\bottomrule
\end{tabular}%
}
\end{table*}

\begin{table*}[!t]
\centering
\scriptsize
\setlength{\tabcolsep}{5pt}
\caption{Effect of budget-independent subgraph caching for
$\lambda=0.01$ and $\varepsilon_{\mathrm{mult}}=0.03$.
Speedup is the runtime without storing pointers to previously constructed
subgraphs, as in PRAXIS, divided by the runtime with
budget-independent subgraph caching. The increase in trees is the number of
trees recovered with budget-independent subgraph caching divided by the number
recovered without storing subgraph pointers.}
\label{tab:cache-lam001-rm003}
\resizebox{\textwidth}{!}{%
\begin{tabular}{lrrrr}
\toprule
Dataset
& Speedup by Caching
& $\times$ Increase in Trees by Caching
& OR Nodes with Budget-Independent Caching
& OR Nodes without Stored Subgraph Pointers, as in \citet{heile2026} \\
\midrule
Abalone  & 1.131  & 1.000 & 285,698   & 2,159,377 \\
Adult    & 1.020  & 1.000 & 483       & 1,069 \\
Aging    & 0.985  & 1.000 & 1         & 1 \\
Bank     & 0.999  & 1.000 & 79        & 79 \\
Bike     & 0.991  & 1.000 & 677       & 2,007 \\
Churn    & 1.060  & 1.000 & 14,623    & 625,837 \\
COMPAS   & 1.060  & 1.000 & 7,854     & 47,399 \\
Coupon   & 75.680 & 1.000 & 150,620   & 301,679,963 \\
Credit   & 1.008  & 1.000 & 3         & 3 \\
Diabetes & 1.009  & 1.000 & 1         & 1 \\
Diamonds & 0.991  & 1.000 & 115       & 147 \\
HELOC    & 1.009  & 1.000 & 88,306    & 296,815 \\
Jasmine  & 1.311  & 1.000 & 1,978,196 & 33,918,179 \\
Pol      & 1.083  & 1.000 & 63,139    & 2,823,515 \\
Shopping & 1.004  & 1.000 & 1,861     & 1,861 \\
Spambase & 1.587  & 1.020 & 6,124,151 & 297,521,989 \\
Student  & 0.975  & 1.000 & 86        & 205 \\
Wine     & 0.993  & 1.000 & 33        & 35 \\
\bottomrule
\end{tabular}%
}
\end{table*}

\begin{table*}[!t]
\centering
\scriptsize
\setlength{\tabcolsep}{5pt}
\caption{Effect of budget-independent subgraph caching for
$\lambda=0.01$ and $\varepsilon_{\mathrm{mult}}=0.05$.
Speedup is the runtime without storing pointers to previously constructed
subgraphs, as in PRAXIS, divided by the runtime with
budget-independent subgraph caching. The increase in trees is the number of
trees recovered with budget-independent subgraph caching divided by the number
recovered without storing subgraph pointers.}
\label{tab:cache-lam001-rm005}
\resizebox{\textwidth}{!}{%
\begin{tabular}{lrrrr}
\toprule
Dataset
& Speedup by Caching
& $\times$ Increase in Trees by Caching
& OR Nodes with Budget-Independent Caching
& OR Nodes without Stored Subgraph Pointers, as in \citet{heile2026} \\
\midrule
Adult    & 1.062  & 1.000 & 196,707 & 803,077 \\
Aging    & 0.963  & 1.000 & 71      & 71 \\
Bank     & 1.007  & 1.000 & 1,323   & 1,811 \\
Bike     & 1.011  & 1.000 & 6,215   & 30,489 \\
Compas   & 1.652  & 1.000 & 80,977  & 1,663,445 \\
Coupon   & 70.756 & 1.000 & 150,620 & 301,679,963 \\
Diabetes & 1.002  & 1.000 & 1       & 1 \\
Diamonds & 0.997  & 1.000 & 429     & 1,015 \\
Pol      & 1.132  & 1.000 & 260,820 & 15,415,497 \\
Shopping & 0.999  & 1.000 & 2,349   & 2,349 \\
Student  & 0.978  & 1.000 & 1,063   & 3,175 \\
Wine     & 1.024  & 1.000 & 11,411  & 18,567 \\
\bottomrule
\end{tabular}%
}
\end{table*}

\begin{table*}[!t]
\centering
\scriptsize
\setlength{\tabcolsep}{5pt}
\caption{Effect of budget-independent subgraph caching for
$\lambda=0.02$ and $\varepsilon_{\mathrm{mult}}=0.2$.
Speedup is the runtime without storing pointers to previously constructed
subgraphs, as in PRAXIS, divided by the runtime with
budget-independent subgraph caching. The increase in trees is the number of
trees recovered with budget-independent subgraph caching divided by the number
recovered without storing subgraph pointers.}
\label{tab:cache-lam002-rm02}
\resizebox{\textwidth}{!}{%
\begin{tabular}{lrrrr}
\toprule
Dataset
& Speedup by Caching
& $\times$ Increase in Trees by Caching
& OR Nodes with Budget-Independent Caching
& OR Nodes without Stored Subgraph Pointers, as in \citet{heile2026} \\
\midrule
Aging    & 0.977  & 1.000 & 1,928     & 4,223 \\
Bank     & 0.996  & 1.000 & 20,083    & 20,905 \\
Bike     & 1.625  & 1.000 & 6,969,673 & 154,718,899 \\
Coupon   & 68.379 & 1.000 & 217,516   & 412,732,685 \\
Diabetes & 1.001  & 1.000 & 371       & 371 \\
Diamonds & 1.100  & 1.000 & 570,896   & 1,785,377 \\
Student  & 1.198  & 1.000 & 72,674    & 640,247 \\
\bottomrule
\end{tabular}%
}
\end{table*}


\end{document}